\documentclass[12pt,a4paper]{article}

\usepackage{caption}

\usepackage{amsmath,amssymb,amsfonts,siunitx,commath}
\usepackage{amsthm}
\usepackage{tikz,url}
\usepackage{mathtools}
\mathtoolsset{showonlyrefs}
\usepackage{enumerate}

\usepackage[utf8]{inputenc}
\usepackage[T1]{fontenc}  

\usepackage{url} 
\usepackage{booktabs}
\usepackage{amsfonts}
\usepackage{nicefrac}
\usepackage{microtype}
\usepackage{subcaption}
\usepackage{geometry}
\DeclareMathOperator*{\argmax}{arg\,max}
\DeclareMathOperator*{\argmin}{arg\,min}
\newcommand\R{\mathbb{R}}
\newcommand\E{\mathbb{E}}
\renewcommand\P{\mathbb{P}}
\newcommand\tT{\mathrm{T}}
\newcommand\dx{\mathrm{d}}

\newtheorem{lemma}{Lemma}

\newtheorem{proposition}[lemma]{Proposition}
\newtheorem{remark}[lemma]{Remark}

\usepackage{xr}

\begin{document}
\title{PatchNR: Learning from Very Few Images\\ by Patch Normalizing Flow Regularization}
\author{
    Fabian Altekr\"uger\footnotemark[1] \footnotemark[2] \and
    Alexander Denker \footnotemark[3] \and 
    Paul Hagemann\footnotemark[2] \and
    Johannes Hertrich\footnotemark[2] \and 
    Peter Maass \footnotemark[3] \and 
    Gabriele Steidl\footnotemark[2]
}

\maketitle
\renewcommand{\thefootnote}{\fnsymbol{footnote}}

\footnotetext[1]{
Department of Mathematics, 
Humboldt-Universit\"at zu Berlin, 
Unter den Linden 6, 
D-10099 Berlin, Germany,
fabian.altekrueger@hu-berlin.de
}
\footnotetext[2]{
Institute of Mathematics,
Technische Universit\"at Berlin,
Stra{\ss}e des 17. Juni 136, 
D-10623 Berlin, Germany,
hagemann@math.tu-berlin.de, j.hertrich@math.tu-berlin.de, steidl@math.tu-berlin.de
} 
\footnotetext[3]{
 Center for Industrial Mathematics,
University of Bremen,
Bibliothekstra{\ss}e 5,
D-28359 Bremen, Germany,
adenker@uni-bremen.de, pmaass@math.uni-bremen.de
} 

\vspace{10pt}
\renewcommand{\thefootnote}{\arabic{footnote}}
\begin{abstract}
Learning neural networks using only few available information is
an important ongoing research topic with tremendous potential for applications.
In this paper, we introduce a powerful regularizer 
for the variational modeling of inverse problems in imaging.
 Our regularizer, called patch normalizing flow regularizer (patchNR), 
involves a normalizing flow learned on small patches of very few images.
In particular, the training is independent of the considered inverse problem such that
the same regularizer
can be applied for different forward operators acting on the same class of images.
By investigating the distribution of patches versus those of the whole image class,
we prove that our  model is indeed a MAP approach.
Numerical examples 
for low-dose and limited-angle computed tomography (CT) as well as  superresolution of  material images  demonstrate that our method
provides very high quality results.
The training set consists of just six images for CT and one image for superresolution.
Finally, we combine our patchNR with ideas from internal learning
for performing superresolution of natural images directly from the low-resolution observation
without  knowledge of any high-resolution image.
\end{abstract}

\section{Introduction}

Solving inverse problems with limited access to training data is an active field of research. 
In inverse problems, we aim to reconstruct an unknown ground truth image $x$ from some observation
\begin{equation}\label{eq_IP}
y=\mathrm{noisy}(f(x)),
\end{equation}
where $f$ is an ill-posed forward operator and ``$\mathrm{noisy}$'' describes some noise model. In the recent years learning-based reconstruction methods as
supervisely trained neural networks on large datasets \cite{JMFU17,RFB15a} 
and conditional generative models \cite{ALKRK2019,HHS2021,MO2014,SLY2015}
attained a lot of attention. However, for many applications like medical or material imaging, the acquisition of a large database of  ground truth images or  pairs of ground truth images
and observations is very costly or even impossible \cite{Kohli2017, xu2017enhanced}. 
Model-based approaches assume that 
the forward operator $f$ is known and aim to minimize a variational functional of the form
\begin{equation}\label{eq_variational}
\mathcal J(x;y)=\mathcal D(f(x),y)+ \lambda \mathcal R(x), ~\lambda >0,
\end{equation}
where $\mathcal D$ is a data-fidelity term which depends on the noise model and measures how well the reconstruction fits to the observation and $\mathcal R$ is a regularizer which copes with the ill-posedness and 
incorporates prior information.
Over the last years, learned regularizers like
the total deep variation \cite{KEKP2020,KEKP2021} or adversarial regularizers \cite{lunz2018adversarial,MCOS2021,prost2021learning}
as well as extensions of plug-and-play methods \cite{gilton2019,SVWB2016,VBW2013}
with learned denoisers \cite{HNS2021, HLP22, REM2017,ZLZZ2021}  showed promising results,
see \cite{arridge2019solving,OJMBDW2020} for an overview.

Furthermore, many papers leveraged the tractability of the likelihood of normalizing flows to learn a prior \cite{ADLAH2020,Helminger2021GenericIR, Wei2022DeepUW,WLD2021}. They utilize the invertiblity to optimize over the range of the flow together with the Gaussian assumption on the latent space. Also, diffusion models \cite{kawar2022denoising,kawar2021snips, songermongrad,song2022solving} have shown great generative modelling capabilities and have been used as a prior for inverse problems. Moreover, other generative models, such as GANs \cite{ATKP2018,gangoodfellow, pan2020dgp} or VAEs \cite{kingmaautoenc}, have been used as a regularizer, see the recent review \cite{Duff21} and references therein.
However, even if these methods allow an unsupervised reconstruction, their training is often computationally costly 
and a huge amount of training images is required. 

One possibility to reduce the training effort consists in using only small image patches.
Denoising methods based on the comparison of similar patches provided state-of-the-art methods
\cite{BCM2005,DFKE09,LNPS2017,LBM2013} for a long time. Recently,
the approximation of \emph{patch distributions} of images was successfully exploited in certain papers
\cite{AH2022,delon2018gaussian, Granot_2022_CVPR,Hertrich21,HNABBSS2020,HBD2018,SJ2016}. 
In particular, \cite{ZW2011} proposed the negative log likelihood of all patches of an image
as a regularizer, where the underlying patch distribution was assumed to follow a Gaussian mixture model (GMM)
which parameters were learned from few clean images.
This method is still competitive to many approaches based on deep learning 
and several extensions were suggested recently \cite{FW2021,PDDN2019,STA2021}.
However, even though GMMs can approximate any probability density function if the number of components is large enough, 
they suffer from limited flexibility in case of a fixed number of components, see \cite{GW2000} and the references therein.
Moreover the subsequent reconstruction procedure detailed in  \cite{ZW2011} is computationally expensive. 

In this paper, we propose to use a regularizer which incorporates a normalizing flow (NF) learned on small image patches, usually of size $6\times 6$.
NFs were introduced in  \cite{dinhrnvp,varinfNF}, see also \cite{RH2021} for its continuous counterpart. They build upon invertible neural networks and allow for explicit density evaluation.
Our PatchNR consists of a NF which is trained to approximate the distribution of patches. As the structure of small patches is usually much simpler than those of whole images, it appears that their approximation is more accurate.
Moreover, as a large data base of patches can be extracted from few images, we require only a very small amount of training images. Once the patchNR is learned, 
we use the negative log likelihood of all patches as regularizer in \eqref{eq_variational} for its reconstruction.
Indeed we will prove that our model can be obtained by a maximum a posteriori (MAP) approach.
Since the regularizer is only specific to the considered image class, but not to the inverse problem,
it can be applied for different forward operators as, e.g., for low-dose CT and limited-angle CT without additional training.
This is in contrast to data-based methods as FBP+UNet, where the network has to be trained for each new forward operator separately.

We demonstrate by numerical examples that our patchNR admits high quality results for low-dose and limited-angle CT as well as superresolution. 
Moreover, we combine patchNRs with ideas from internal learning to perform superresolution on natural images without any training data. 
Internal learning \cite{GBI2009,SCI2018} is based on the observation that the patches within natural images are self-similar across the scales.
In our case, this leads to the idea to train the patchNR on the low-resolution observation instead of a dataset of training images.
We demonstrate on the BSD68 dataset that zero-shot superresolution with patchNRs outperforms comparable methods including the original ZSSR paper by \cite{SCI2018}.
Finally, let us mention that recent works have explored meta-learning 
for efficient one gradient step reconstruction \cite{Soh_2020_CVPR}, 
applications to  light field superresolution \cite{Cheng_2021_CVPR} and 
CT superresolution \cite{ct_zero}.

The paper is organized as follows: We start by explaining the training of our patchNR 
and the variational reconstruction model in Section~\ref{sec_var_patchnf}.
Our approach is integrated into the MAP framework in Section~\ref{sec_analysis_patchnf}, 
i.e., the patchNR defines a probability density on the full image space. 
In Section~\ref{Sec_experiments}, we evaluate the performance of our model in CT and superresolution and use the patchNR also for zero-shot superresolution. Finally, conclusions are drawn in Section~\ref{sec:conclusions}.

\section{Patch Normalizing Flows in Variational Modeling}\label{sec_var_patchnf}

In the following, we assume that we are given a small number of high quality example images $x_1,...,x_M \in \R^{d_1 \times d_2}$ - indeed $M=1$ or $M=6$ in our numerical examples -
from a certain image class as CT, material or texture images. 
Our method consists of two steps, namely
i) learning a NF for approximating the patch distribution of
the example images, 
and 
ii) establishing a variational model which incorporates the learned patchNR in the regularizer.

\paragraph{Learning patchNRs}
Let $\text{p}_1,...,\text{p}_N\in\R^{s_1 \times s_2}$, $s_i \ll d_i$, $i=1,2$ denote all
possibly overlapping patches of the example images, 
where we assume that the patches $\text{p}_1,...,\text{p}_N$ are realizations 
of an absolute continuous probability distribution $Q$ with density $q$.
We aim to approximate $q$ by a NF.
For simplicity, we rearrange the images and the patches columnwise into vectors of size $d = d_1d_2$ and  $s \coloneqq s_1 s_2$, respectively.
Then we learn a diffeomorphism $\mathcal T = \mathcal{T}_\theta \colon\R^s\to\R^s$ such that 
$Q\approx\mathcal T_\#P_Z \coloneqq P_Z \circ \mathcal T^{-1}$, 
where $P_Z$ is a $s$-dimensional standard normal distribution.
To this end, we set our normalizing flow $\mathcal T$ to be an invertible neural network with parameters collected in $\theta$. There were several approaches in the literature to construct such invertible neural networks \cite{AKRK2019,CBDJ2019,dinhrnvp,KD2018,Lugmayr20}. 
In this paper, we adapt the architecture from \cite{AKRK2019}. In order to train our NF, we aim to minimize the (forward) Kullback-Leibler divergence
\begin{align}
\mathrm{KL}(Q,\mathcal {T}_\#P_Z)
\coloneqq \int_{\mathbb R^s} \log \Big( \frac{\dx Q}{\dx \mathcal {T}_\#P_Z}\Big) \dx Q
&=\E_{\text{p}\sim Q} \big[\log\big(q(\text{p})\big)]-\E_{\text{p}\sim Q}\big[\log\big(p_{\mathcal T_\#P_Z}(\text{p})\big) \big],
\end{align}
where we set the expression to $+\infty$ if $Q$ is not absolutely continuous with respect to $\mathcal {T}_\#P_Z$.
The first term on the right-hand side is a constant independent of the parameters $\theta$. Thus, 
using the change-of-variables formula for probability density functions of push-forward measures
$p_{\mathcal T_\#P_Z}=p_Z(\mathcal T^{-1})|\mathrm{det}(\nabla \mathcal T^{-1})|$, 
we obtain that the above formula is
up to a constant equal to
$$
-\E_{\text{p}\sim Q}\big[\log\big(p_Z(\mathcal T^{-1}(\text{p})\big)+\log\big(\big|\mathrm{det}(\nabla \mathcal T^{-1}(\text{p}))\big|\big)\big],
$$
where $\nabla \mathcal{T}^{-1}$ denotes the Jacobian matrix of $\mathcal{T}^{-1}$.
By replacing the expectation by the empirical mean of our training set, inserting the standard normal density $p_Z$ and neglecting some constants, we finally obtain the loss function
\begin{equation} \label{equ_patchNR_training}
    \mathcal L(\theta) \coloneqq \frac{1}{N}\sum_{i=1}^N \frac12 \big\|\mathcal T^{-1}(\text{p}_i) \big\|^2-\log\big(\big|\mathrm{det}(\nabla \mathcal T^{-1}(\text{p}_i))\big|\big).
\end{equation}

We minimize this loss function using the Adam optimizer \cite{KB2015}.

\paragraph{Reconstruction with PatchNRs}

Once the patchNR $\mathcal T$ is learned, 
we aim to use it within a regularizer of the variational model \eqref{eq_variational}
to solve the inverse problem \eqref{eq_IP}.
To this end, denote by $E_i\colon\R^{d}\to \R^s$, $i=1,\ldots,N_p$ the operator, 
which extracts the $i$-th (vectorized) patch from the unknown (vectorized) image $x \in \R^{d}$. 
Then, we define our regularizer by the negative log likelihood of all patches under the probability distribution learned by the patchNR.
More precisely, we define the patchNR based prior
$$
\frac1{s}\sum_{i=1}^{N_p}-\log\big(p_{\mathcal T_\#P_Z}\big(E_i(x)\big)\big),
$$
where $N_p$ is the number of patches in the image $x$ and $s=s_1 s_2$ the number of pixels in a patch.
Similar to \eqref{equ_patchNR_training}, this can be reformulated by the change-of-variables formula and by ignoring some constants as
\begin{align}
\mathrm{patchNR}(x) &\coloneqq 
\frac1{s}\sum_{i=1}^{N_p}\frac{1}{2}\big\|\mathcal T^{-1} \big(E_i(x) \big)\big\|^2
-
\log\big(\big|\mathrm{det}\big(\nabla \mathcal T^{-1} \big(E_i(x) \big) \big) \big|\big).
\end{align}

Note that if we ignore the boundary of the image, the patchNR is translation invariant. That is, a translation of the image does not change the value of the regularizer.
Now, we reconstruct our ground truth by finding a minimizer of the variational problem
\begin{align}
\label{eq:PatchNR_VariationalFormulation}
\mathcal J(x; y)=\mathcal D(f(x),y) + \lambda \, \mathrm{patchNR}(x), \quad \lambda >0.
\end{align}
For the minimization, we use the Adam optimizer \cite{KB2015}.
To speed up the numerical computations, we do not consider all overlapping patches in $x$, but 
choose randomly a subset of $N_p$ possibly overlapping patches in each iteration. 
Note that the resulting optimization problem is non-convex and therefore the choice of the initialization is important. In our experiments we initialize this optimization with a coarse reconstruction, i.e., a bicubic interpolation for superresolution or the filtered backprojection (FBP) for CT.

\begin{remark}[Relation to EPLL]
Our patchNR is closely related to the expected patch log likelihood (EPLL) prior proposed by \cite{ZW2011}.
Here, the authors use the prior defined as
\begin{align}
\mathrm{EPLL}(x)=\frac{1}{N_p}\sum_{i=1}^{N_p} -\log\big(p\big(E_i(x)\big)\big),
\end{align}
where $p$ is the probability density function of a GMM approximating the patch distribution of the image class of interest.
However, GMMs have a limited expressiveness and can only hardly approximate complicated probability distributions induced by patches \cite{GW2000}.
Further, the reconstruction process proposed in \cite{ZW2011} is computationally very costly even though a reduction of the computational effort was considered in several papers \cite{PDDN2019,STA2021}.
Indeed, we will show in our numerical examples that the patchNR clearly outperforms the reconstructions from EPLL.
\end{remark}

\section{Analysis of Patch Normalizing Flows} \label{sec_analysis_patchnf}

In this section, we investigate the patch distribution which is approximated by the patchNR. More precisely, we show that any probability density on the class of images induces a probability density on the space of patches
and vice versa. We will use this to interpret the minimization of the variational problem  \eqref{eq:PatchNR_VariationalFormulation} as maximizing the posterior distribution.

\begin{remark}
Considering the inverse problem \eqref{eq_IP} as a Bayesian inverse problem 
\begin{align*}
    Y = \mathrm{noisy}(f(X)),
\end{align*}
where $X$ and $Y$ are random variables,  Bayes' theorem implies that maximizing the log-posterior distribution $\log(p_{X|Y=y}(x))$ can be written as
\begin{align*}
\argmax_x \{\log (p_{X|Y=y}(x) )\}
&=
\argmin_x \Big\{-\log p_{Y|X=x}(y) - \log p_X(x) \Big\}.
\end{align*}
Consequently, the data-fidelity term and the regularizer in \eqref{eq_variational} correspond to the negative log likelihood $\mathcal{D}(f(x),y) = - \log p_{Y|X=x}(y)$ and to the negative log prior $\mathcal{R}(x) = - \log p_X(x)$, respectively.
\end{remark}

Let $X\colon\Omega\to\R^{d}$ with $X\sim P_X$ be a $d$-dimensional random variable on the space of images.
By $\tilde E_i\colon\R^{d}\to\R^{d-s}$ we denote the operator which extracts all pixels from a $d$-dimensional image, which do not belong to the $i$-th patch.
Let $I\colon\Omega\to\{1,...,N_p\}$ be a random variable which follows the uniform distribution on $\{1,...,N_p\}$.
Then, the random variable $\omega\mapsto E_{I(\omega)}(X(\omega))$ describes the selection of a random patch from a random image.
We call the distribution $Q$ of $E_I(X)$ the patch distribution corresponding to $P_X$. The following lemma provides an explicit formula for the density of $Q$.

\begin{lemma}\label{lem_density}
Let $P_X$ be a probability distribution on $\R^{d}$ with density $p_X$. 
Then, also the corresponding patch distribution $Q$ is absolute continuous with density
\begin{align}
q(\textnormal{p})=\frac{1}{N_p}\sum_{i=1}^{N_p}\int_{\R^{d-s}}p_X\big(E_i^\tT(\textnormal{p})+\tilde E_i^\tT(\tilde{x})\big) \dx \tilde{x}.
\end{align}
\end{lemma}

\begin{proof}
Let $A\in\mathcal B(\R^s)$ be an arbitrary Borel set. Then, we have by Bayes' theorem that
\begin{align*}
Q(A)&=\sum_{i=1}^{N_p} \P(I=i){E_i}_\#P_X(A)
=\frac{1}{N_p}\sum_{i=1}^{N_p}P_X(E_i^{-1}(A)) 
=\frac{1}{N_p}\sum_{i=1}^{N_p}\int_{\R^d}1_{A}\big(E_i(x)\big)p_X(x)\dx x. 
\end{align*}
Now, we use the decomposition 
$$
\R^{d}=\mathrm{Ker}(E_i)\oplus\mathrm{Im}(E_i^\tT)=\mathrm{Im}(\tilde E_i^\tT)\oplus\mathrm{Im}(E_i^\tT).
$$
With Fubini's theorem, $E_iE_i^\tT=I$ and $E_i\tilde E_i^\tT=0$, this is equal to
\begin{align*}
Q(A)
&=\frac{1}{N_p}\sum_{i=1}^{N_p}\int_{\R^s}\int_{\R^{d-s}}1_A(E_i(E_i^\tT(\text{p})+\tilde E_i^\tT(\tilde{x}))) p_X(E_i^\tT(\text{p})+\tilde E_i^\tT(\tilde{x}))\dx \tilde{x}\dx \text{p}\\
&=\frac{1}{N_p}\sum_{i=1}^{N_p}\int_{\R^s}\int_{\R^{d-s}}1_A(\text{p})p_X(E_i^\tT(\text{p})+\tilde E_i^\tT(\tilde{x}))\dx \tilde{x}\dx \text{p}\\
&=\int_A\frac{1}{N_p}\sum_{i=1}^{N_p}\int_{\R^{d-s}}p_X(E_i^\tT(\text{p})+\tilde E_i^\tT(\tilde{x}))\dx \tilde{x}\dx \text{p}.
\end{align*}
This proves the claim.
\end{proof}

For the proof of the reverse direction, namely that given a  probability measure $Q$ on the space of patches, 
the patchNR defines a probability density function
on the space of all images, we need 
the following lemma. 
It states that the density induced by a NF with a Gaussian latent distribution is up to a constant 
bounded from below and above by the density of certain normal distributions.
In particular, it has exponential asymptotic decay.
Note that similar questions about bi-Lipschitz continuous diffeomorphisms were investigated more detailed in \cite{HN2021,JKYB2020}. 

\begin{lemma}\label{lem_decay}
Let $\mathcal T\colon\R^s\to\R^s$ be a diffeomorphism with Lipschitz constants $\mathrm{Lip}(\mathcal T)\leq K$ and $\mathrm{Lip}(\mathcal T^{-1})\leq L$ and let $P_Z=\mathcal N(0,I)$. Then, it holds
\begin{align}
\frac{1}{L^s K^s}\mathcal N(\textnormal{p}|\mathcal T(0),\tfrac{1}{L^2} I)&\leq p_{\mathcal T_\#P_Z}(\textnormal{p}) \leq L^sK^s\, \mathcal N(\textnormal{p}|\mathcal T(0), K^2 I)
\end{align}
for any $\textnormal{p}\in\R^s$.
\end{lemma}

\begin{proof}
Using the Lipschitz continuity of $\mathcal T$, we obtain
\begin{align}
\frac1{K^2}\|\text{p}-\mathcal T(0)\|^2&=\frac1{K^2}\|\mathcal T(\mathcal T^{-1}(\text{p}))-\mathcal T(0)\|^2  \leq\|\mathcal T^{-1}(\text{p})-0\|^2=\|\mathcal T^{-1}(\text{p})\|^2.
\end{align}
Now, applying the change-of-variables formula and that $|\mathrm{det}(\nabla \mathcal T^{-1}(\text{p}))|\leq L^s$, we conclude
\begin{align*}
p_{\mathcal T_\#P_Z}(\text{p})&=p_Z(\mathcal T^{-1}(\text{p}))|\mathrm{det}(\nabla \mathcal T^{-1}(\text{p}))| \leq \frac{L^s}{(2\pi)^{s/2}}\exp(-\tfrac12\|\mathcal T^{-1}(\text{p})\|^2)\\
&\leq \frac{L^s}{(2\pi)^{s/2}}\exp(-\tfrac1{2K^2}\|\text{p}-\mathcal T(0)\|^2)
=L^sK^{s}\,\mathcal N(\text{p}|\mathcal T(0),K^2 I).
\end{align*}
This shows the second inequality. For the first inequality, note that by the inverse function theorem
$$
\nabla \mathcal T^{-1}(\text{p})=\nabla \mathcal T^{-1}(\mathcal T(\mathcal T^{-1}(\text{p})))=(\nabla \mathcal T(\mathcal T^{-1}(\text{p})))^{-1}.
$$
Using the Lipschitz continuity of $\mathcal T$, this implies 
$$
\big| \mathrm{det}\big( \nabla \mathcal T^{-1}(\text{p}) \big) \big|
=
\Big| \mathrm{det} \Big(\nabla \mathcal T\big( \mathcal T^{-1}(\text{p}) \big)\Big) \Big|^{-1}
\geq 1/K^s.
$$
Further, by the Lipschitz continuity of $\mathcal T^{-1}$ it holds that
$$
\|\mathcal T^{-1}(\text{p})\|^2=\|\mathcal T^{-1}(\text{p})-\mathcal T^{-1}(\mathcal T(0))\|^2\leq L^2\|\text{p}-\mathcal T(0)\|^2.
$$
Putting the things together yields
\begin{align*}
p_{\mathcal T_\#P_Z}(\text{p})&=p_Z(\mathcal T^{-1}(\text{p}))|\mathrm{det}(\nabla \mathcal T^{-1}(\text{p}))| \geq \frac{1}{K^s(2\pi)^{s/2}} \exp(-\tfrac12\|\mathcal T^{-1}(\text{p})\|^2)\\
&\geq \frac{1}{K^s(2\pi)^{s/2}}\exp(-\tfrac{L^2}{2}\|\text{p}-\mathcal T(0)\|^2) =\frac{1}{L^sK^s}\,\mathcal N(\text{p}|\mathcal T(0),\tfrac{1}{L^2} I).
\end{align*}
This completes the proof.
\end{proof}

Now, we can show that the patchNR defines a probability distribution on the space of all images. This allows a MAP interpretation of the variational problem \eqref{eq:PatchNR_VariationalFormulation}.

\begin{proposition}\label{prop_induced_patchNF}
Let $P_Z=\mathcal N(0,I)$ and let $\mathcal T\colon\R^s\to\R^s$ be a bi-Lipschitz diffeomorphism, i.e., $\mathrm{Lip}(\mathcal T)<\infty$ and $\mathrm{Lip}(\mathcal T^{-1})<\infty$. Then, for any $\rho>0$, the function
$\varphi(x)=\exp(-\rho\,\mathrm{patchNR}(x))$ belongs to $L^1(\R^{d})$, where
\begin{align}
\mathrm{patchNR}(x)=\frac{1}{s}\sum_{i=1}^{N_p}-\log\big(p_{\mathcal T_\#P_Z}(E_i(x))\big).
\end{align}
\end{proposition}
\begin{proof}
Using Lemma~\ref{lem_decay}, there exists some $C>0$ such that it holds
\begin{align}
\begin{split}
\int_{\R^{d}}\varphi(x)\dx x&=\int_{\R^{d}} \Big(\prod_{i=1}^{N_p}p_{\mathcal T_\#P_Z}(E_i(x))\Big)^{\rho/s}\dx x\\
&\leq C\int_{\R^{d}} \Big(\prod_{i=1}^{N_p}\mathcal N(E_i(x)|\mathcal T(0),K^2 I)\Big)^{\rho/s}\dx x\\
&=C\int_{\R^{d}} \prod_{i=1}^{N_p}\prod_{j=1}^s\mathcal N((E_i(x))_j|(\mathcal T(0))_j,K^2)^{\rho/s}\dx x,
\end{split}
\end{align}
where $(E_i(x))_j$ is the $j$-th element from $E_i(x)$.
Since $(E_i(x))_j=x_{\sigma(i,j)}$ for some mapping $\sigma\colon\{1,...,N_p\}\times\{1,...,s\}\to\{1,...,d\}$ and using Fubini's theorem, this simplifies to
\begin{align}
\begin{split}
&\quad C\int_{\R^{d}} \prod_{i=1}^{N_p}\prod_{j=1}^s\mathcal N(x_{\sigma(i,j)}|(\mathcal T(0))_j,K^2)^{\rho/s}\dx x\\
&=C\int_{\R^{d}} \prod_{k=1}^{d} \prod_{(i,j)\in\sigma^{-1}(\{k\})}\mathcal N(x_{k}|(\mathcal T(0))_j,K^2)^{\rho/s}\dx x\\
&=C \prod_{k=1}^{d}\int_{\R} \prod_{(i,j)\in\sigma^{-1}(\{k\})}\mathcal N(x|(\mathcal T(0))_j,K^2)^{\rho/s}\dx x.
\end{split}
\end{align}
Here $\sigma^{-1}(\{k\})$ denotes the preimage of $\sigma$ which denotes the set of all index pairs $(i,j)$ such that $(E_i(x))_j=x_k$ for $x \in \mathbb{R}^d$. 
As each pixel in the images is covered by at least one patch, this set is non-empty.
Using the fact that products and powers of normal densities are integrable, we obtain that this expression is finite and the proof is complete.
\end{proof}
\section{Numerical Examples}\label{Sec_experiments}

In this section, we demonstrate the performance of our method. We focus on linear inverse problems, but the approach can also be extended to non-linear forward operators.
First, in Subsection~\ref{sec_ct}, we apply the patchNR to low-dose CT in a full angle and a limited angle setting and present an empirical convergence study for the optimization of the objective functional.  
Afterwards, in Subsection~\ref{sec_superres}, we consider superresolution on material data. This is a typical setting, where only little data is available and superresolution is needed to obtain sufficient detail for material research \cite{dahari_gansuperres,Jung2021,buzzard}. Lastly, we extend our findings
to zero-shot superresolution in Subsection \ref{app_zero_shot}.
To get a better impression on the very good performance of our method, we added
more numerical examples in  \ref{app:furhter_examples}\footnote{the code for all experiments is available at \url{https://github.com/FabianAltekrueger/patchNR}}. 

\paragraph{Comparison Methods}
We compare our method with established methods from the literature, in particular with
\begin{itemize}
    \item Wasserstein Patch Prior (WPP) \cite{AH2022,Hertrich21},
    \item Expected Patch Log Likelihood (EPLL) \cite{ZW2011}, 
    \item Local Adversarial Regularizer (localAR) \cite{prost2021learning},
\end{itemize}
 since these methods also work on patches and are model-based. 
 Note that we optimized the EPLL GMM prior using a gradient descent optimizer ourselves, 
 as the half quadratic splitting proposed originally by the authors of \cite{ZW2011} is much more expensive for the superresolution and CT forward operator. 
 Moreover, we include comparisons with 
 \begin{itemize}
     \item Plug-and-Play Forward Backward Splitting with DRUNet (DPIR) \cite{ZLZZ2021},
     \item Deep Image Prior in combination with a TV prior (DIP+TV) \cite{UVL2018,baguer2020computed},
     \item data-based methods as the post-processing UNet (FBP+UNet) \cite{JMFU17,RFB15a} for CT and an asymmetric CNN (ACNN) \cite{TXZLZ21} for superresolution.
 \end{itemize} 
 Details on the comparison methods are given in  \ref{app_quality_measures_comparisons}.
Note that post-processing approaches are no longer the state-of-the-art for CT reconstruction, but are still widely used and serve as comparison methods. Currently some learned iterative methods provide better results \cite{AdlOkt18}.

\paragraph{Architecture of PatchNR}
For the architecture we use there exists a universal approximation theorem \cite{TITOIS2020}. Therefore, a sufficiently large normalizing flow can approximate arbitrary probability distributions .
We use 5 GlowCoupling blocks and permutations in an alternating manner, where the coupling blocks are from the freely available FrEIA package\footnote{available at \url{https://github.com/VLL-HD/FrEIA}}. The 3-layer subnetworks are fully connected with ReLU activation functions and $512$ nodes resulting in 2908520 learnable parameters. The patchNR is trained on $6 \times 6$ patches, i.e., s = 36. For each image class, we trained the patchNR using Adam optimizer \cite{KB2015} with a learning rate of 0.0001, a batch size of 32 and for 750000 optimizer steps . Training took about 2.5 hours on a single NVIDIA GeForce RTX 2080 super GPU with 8 GB GPU memory.

\subsection{Computed Tomography} \label{sec_ct}

For computed tomography (CT) we use the LoDoPaB dataset \cite{LoDoPaB21}\footnote{available at \url{https://zenodo.org/record/3384092##.Ylglz3VBwgM}} for low-dose CT imaging. It is based on scans of the Lung Image Database Consortium and Image Database Resource Initiative \cite{Armato11} which serve as ground truth images, while the measurements are simulated. The dataset contains 35820 training images, 3522 validation images and 3553 test images. Here the ground truth images have a resolution of $362\times 362$px. The LoDoPab dataset uses a two-dimensional parallel beam geometry with equidistant detector bins. The forward operator is the discretized linear Radon transformation and the noise can be modelled using a Poisson distribution. Concretely, we consider the inverse problem
\begin{align*}
y = f(x) + \eta,~\text{where}~\eta = -f(x)  - \frac{1}{\mu} \log \Big(\frac{\tilde{N}_1}{N_0}\Big), ~ \tilde{N}_1 \sim \text{Pois} \big( N_0 \exp (-f(x) \mu) \big).
\end{align*}
Here $N_0 = 4096$ is the mean photon count per detector bin without attenuation and $\mu = 81.35858$ is a normalization constant. 
In order to compute the corresponding data-fidelity term, we see that the observation $y$ can be rewritten by 
$$
y = - \frac{1}{\mu} \log\Big(\frac{\tilde{N}_1}{N_0} \Big)
$$
and thus $\exp(-y \mu) N_0 = \tilde{N}_1 \sim \text{Pois}(\big( N_0 \exp (-f(x) \mu) \big)$. 
Since we assume that the pixels are corrupted independently and the negative log-likelihood of a Poisson distributed random variable is given by the Kullback-Leibler divergence, we have
\begin{align*}
- \log p(\exp(-y \mu) N_0 | \exp (-f(x)\mu) N_0) &= \sum_{i=1}^d -\log p(\exp(-y_i \mu) N_0 | \exp (-f(x)_i \mu) N_0) \\
&= \sum_{i=1}^d  e^{-f(x)_i \mu} N_0 + e^{-y_i \mu} N_0 \big(f(x)_i \mu - \log(N_0) \big)
\end{align*}
Thus, the concrete form of \eqref{eq:PatchNR_VariationalFormulation} we aim to minimize, is given by
\begin{align*}
\mathcal{J}(x;y) = \sum_{i=1}^d e^{-f(x)_i} N_0 + e^{-y_i} N_0 \big(f(x)_i - \log(N_0) \big) + \lambda \mathcal{R}(x).
\end{align*}

\begin{figure*}
\centering
\begin{subfigure}{.165\textwidth}
  \centering
  \includegraphics[width=\linewidth]{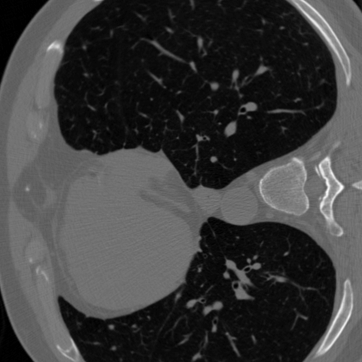}
\end{subfigure}%
\hfill
\begin{subfigure}{.165\textwidth}
  \centering
  \includegraphics[width=\linewidth]{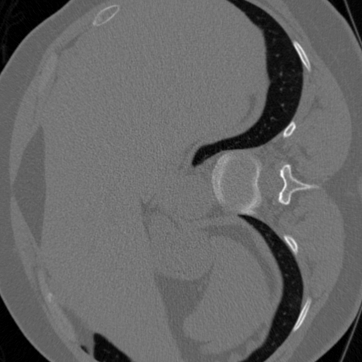}
\end{subfigure}%
\hfill
\begin{subfigure}{.165\textwidth}
  \centering
  \includegraphics[width=\linewidth]{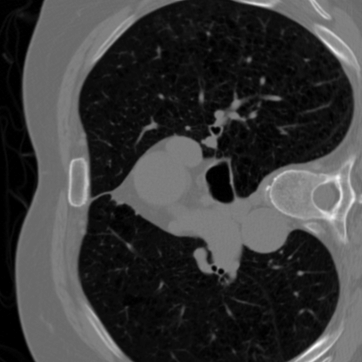}
\end{subfigure}%
\hfill
\begin{subfigure}{.165\textwidth}
  \centering
  \includegraphics[width=\linewidth]{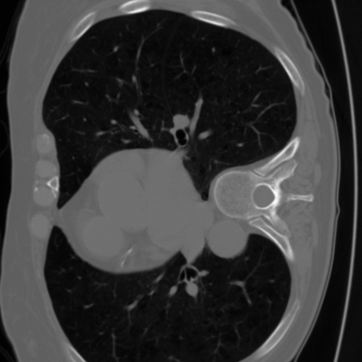}
\end{subfigure}%
\hfill
\begin{subfigure}{.165\textwidth}
  \centering
  \includegraphics[width=\linewidth]{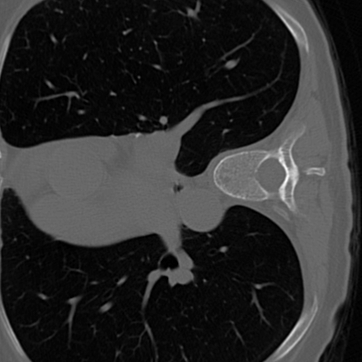}
\end{subfigure}%
\hfill
\begin{subfigure}{.165\textwidth}
  \centering
  \includegraphics[width=\linewidth]{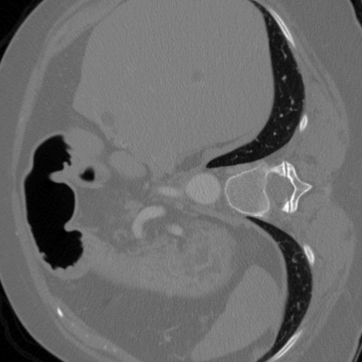}
\end{subfigure}%
\caption{Ground truth images used for training the patchNR (CT). } \label{fig_training_ground_truth}
\end{figure*}

We trained the patchNR using patches of a small set of $M=6$ handpicked CT ground truth images of size $362 \times 362$ illustrated in Figure~\ref{fig_training_ground_truth}. Once trained, the patchNR can be used both for the full angle CT and the limited angle CT setting, where we use a regularization parameter $\lambda = 700 \frac{s}{N_p}$, a random subset of $N_p = 40000$ overlapping patches in each iteration and the Adam optimizer with a learning rate of 0.005. While for full angle CT we optimized over 300 iterations, for limited angle CT 3000 iterations are used.

\paragraph{Full angle CT}

\begin{table*}[b]
\begin{center}
\scalebox{.78}{
\begin{tabular}[t]{c|ccccc|c} 
            & FBP                & DIP + TV          & EPLL        & localAR    & patchNR              & FBP+UNet \\
            &                    &                   &             &            &                      & (data-based) \\
\hline
PSNR        & 30.37 $\pm$ 2.95   & 34.45 $\pm$ 4.20  & 34.89 $\pm$ 4.41   & 33.64 $\pm$  3.74 &  \textbf{35.19} $\pm$ 4.52  & 35.48 $\pm$ 4.52 \\ 
SSIM        & 0.739 $\pm$ 0.141  & 0.821 $\pm$ 0.147 & 0.821 $\pm$ 0.154 & 0.807 $\pm$ 0.145 & \textbf{0.829} $\pm$ 0.152  & 0.837 $\pm$ 0.143  \\
\hline
Runtime     & 0.03s              & 1514.33s          & 36.65s      & 30.03s     &  48.39s              & 0.46s 
\end{tabular}}
\caption{Full angle CT. Averaged quality measures and standard deviations of the reconstructions. } 
\label{Table_averagederrorMeasures_fullangleCT}
\end{center}
\end{table*}

\begin{figure*}
\centering
\begin{subfigure}[t]{.14\textwidth}
  \includegraphics[width=\linewidth]{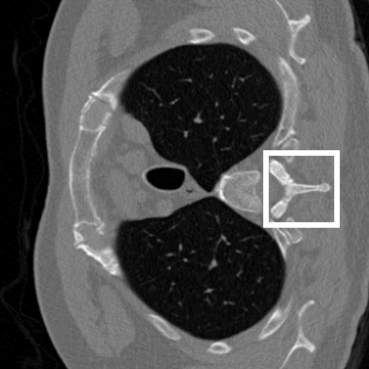}
\end{subfigure}%
\hfill
\begin{subfigure}[t]{.14\textwidth}
  \includegraphics[width=\linewidth]{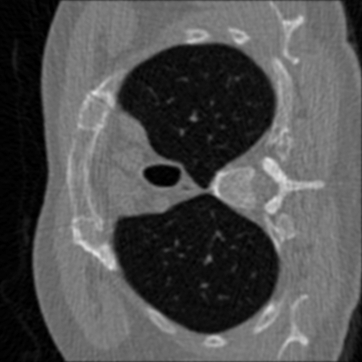}
\end{subfigure}%
\hfill
\begin{subfigure}[t]{.14\textwidth}
  \includegraphics[width=\linewidth]{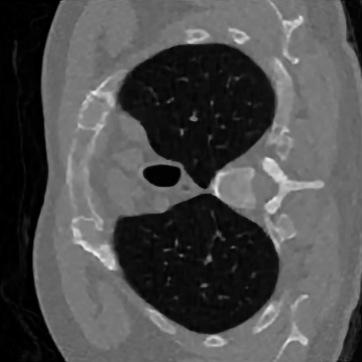}
\end{subfigure}%
\hfill
\begin{subfigure}[t]{.14\textwidth}
  \includegraphics[width=\linewidth]{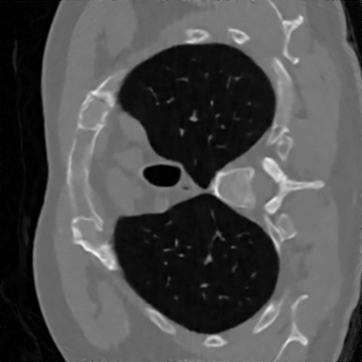}
\end{subfigure}%
\hfill
\begin{subfigure}[t]{.14\textwidth}
  \includegraphics[width=\linewidth]{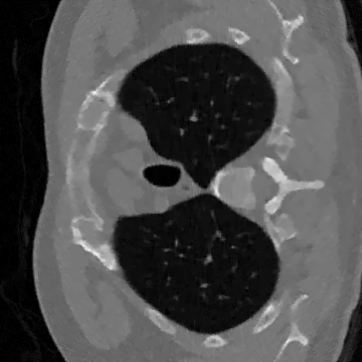}
\end{subfigure}%
\hfill
\begin{subfigure}[t]{.14\textwidth}
  \includegraphics[width=\linewidth]{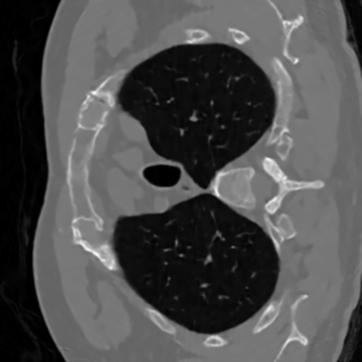}
\end{subfigure}%
\hspace{0.03cm}
\begin{subfigure}[t]{.14\textwidth}
  \includegraphics[width=\linewidth]{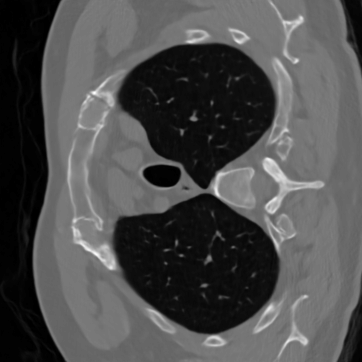}
\end{subfigure}%

\begin{subfigure}[t]{.14\textwidth}
  \includegraphics[width=\linewidth]{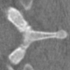}
  \caption*{Ground truth}
\end{subfigure}%
\hfill
\begin{subfigure}[t]{.14\textwidth}
  \includegraphics[width=\linewidth]{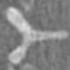}
  \caption*{FBP}  
\end{subfigure}%
\hfill
\begin{subfigure}[t]{.14\textwidth}
  \includegraphics[width=\linewidth]{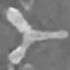}
  \caption*{DIP+TV}  
\end{subfigure}%
\hfill
\begin{subfigure}[t]{.14\textwidth}
  \includegraphics[width=\linewidth]{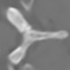}
  \caption*{EPLL}  
\end{subfigure}%
\hfill
\begin{subfigure}[t]{.14\textwidth}
  \includegraphics[width=\linewidth]{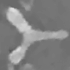}
  \caption*{localAR}  
\end{subfigure}%
\hfill
\begin{subfigure}[t]{.14\textwidth}
  \includegraphics[width=\linewidth]{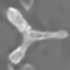}
  \caption*{patchNR}  
\end{subfigure}%
\hspace{0.03cm}
\begin{subfigure}[t]{.14\textwidth}
  \includegraphics[width=\linewidth]{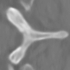}
  \captionsetup{justification=centering}
  \caption*{FBP+UNet \\(data-based)}  
\end{subfigure}%
\caption{Full angle  CT  using different methods. The zoomed-in part is marked with a white box in the ground truth image.
Our approach gives significantly better results than the model-based comparison methods.
\textit{Top}: full image. \textit{Bottom}: zoomed-in part.} \label{Fig_CT_img}
\end{figure*}

For full angle CT we consider 1000 equidistant angles between 0 and $\pi$. 
In Figure \ref{Fig_CT_img} we compare different methods for full angle low-dose CT imaging. 
Here the patchNR yields better results than DIP+TV and localAR, in particular the edges are sharper and more realistic in the reconstruction of patchNR. Visually, there are only small differences between patchNR and FBP+UNet observable, although FBP+UNet is a data-based method trained on 35820 image pairs, while we only used 6 ground truth images for training the patchNR. 
The quality measures averaged over the first 100 test images of the LoDoPaB dataset in Table \ref{Table_averagederrorMeasures_fullangleCT} confirm these observations. 
PSNR and SSIM were evaluated on an adaptive data range. Note that the diversity of the test set causes relatively high standard deviations.

\paragraph{Limited angle CT}

\begin{table*}[b]
\begin{center}
\scalebox{.78}{
\begin{tabular}[t]{c|ccccc|c} 
         & FBP     & DIP + TV & EPLL   & localAR & patchNR & FBP+UNet \\
         &         &          &        &         &         & (data-based) \\
\hline
PSNR     & 21.96 $\pm$ 2.25   & 32.57 $\pm$ 3.25  & 32.78 $\pm$ 3.46 & 31.06 $\pm$ 2.95 &  \textbf{33.20} $\pm$ 3.55  & 33.75 $\pm$ 3.58  \\ 
SSIM     & 0.531 $\pm$ 0.097 & 0.803 $\pm$ 0.146  & 0.801 $\pm$ 0.151 & 0.779 $\pm$ 0.142 & \textbf{0.811} $\pm$ 0.151 & 0.820 $\pm$ 0.140\\\hline
Runtime  & 0.02s  & 1770.89s     &  127.21s   &  53.47s    & 485.93s   & 0.53s
\end{tabular}}
\caption{Limited angle CT. Averaged quality measures and standard deviations of the reconstructions. }  
\label{Table_averagederrorMeasures_limitedangleCT}
\end{center}
\end{table*}

\begin{figure*}
\centering
\begin{subfigure}[t]{.14\textwidth}
  \includegraphics[width=\linewidth]{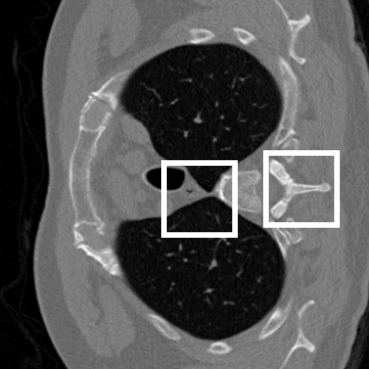}
\end{subfigure}%
\hfill
\begin{subfigure}[t]{.14\textwidth}
  \includegraphics[width=\linewidth]{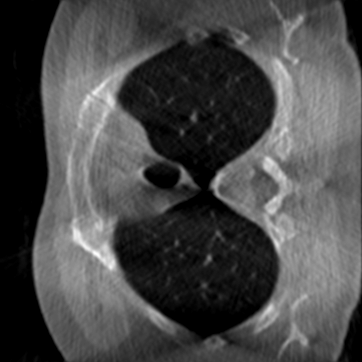}
\end{subfigure}%
\hfill
\begin{subfigure}[t]{.14\textwidth}
  \includegraphics[width=\linewidth]{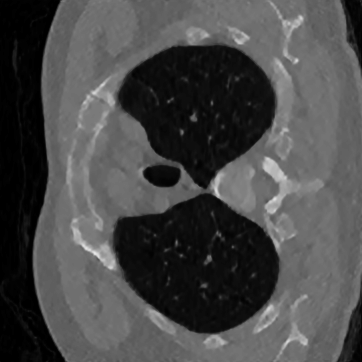}
\end{subfigure}%
\hfill
\begin{subfigure}[t]{.14\textwidth}
  \includegraphics[width=\linewidth]{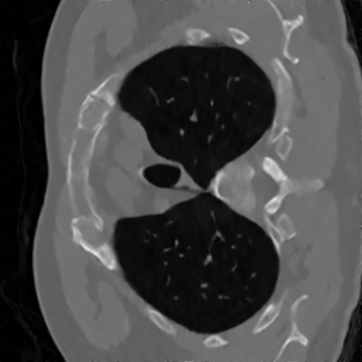}
\end{subfigure}%
\hfill
\begin{subfigure}[t]{.14\textwidth}
  \includegraphics[width=\linewidth]{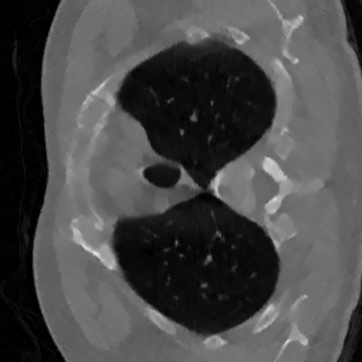}
\end{subfigure}%
\hfill
\begin{subfigure}[t]{.14\textwidth}
  \includegraphics[width=\linewidth]{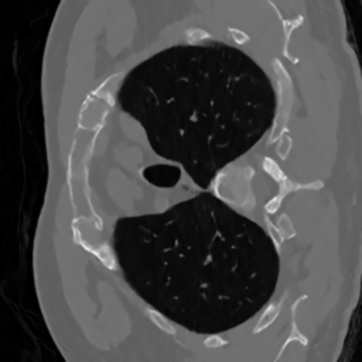}
\end{subfigure}%
\hspace{0.03cm}
\begin{subfigure}[t]{.14\textwidth}
  \includegraphics[width=\linewidth]{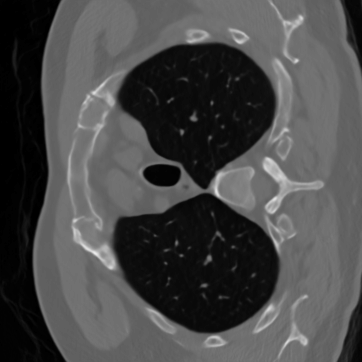}
\end{subfigure}%

\begin{subfigure}[t]{.14\textwidth}
  \includegraphics[width=\linewidth]{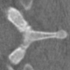}
\end{subfigure}%
\hfill
\begin{subfigure}[t]{.14\textwidth}
  \includegraphics[width=\linewidth]{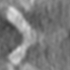}
\end{subfigure}%
\hfill
\begin{subfigure}[t]{.14\textwidth}
  \includegraphics[width=\linewidth]{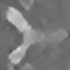}
\end{subfigure}%
\hfill
\begin{subfigure}[t]{.14\textwidth}
  \includegraphics[width=\linewidth]{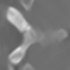}
\end{subfigure}%
\hfill
\begin{subfigure}[t]{.14\textwidth}
  \includegraphics[width=\linewidth]{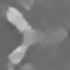}
\end{subfigure}%
\hfill
\begin{subfigure}[t]{.14\textwidth}
  \includegraphics[width=\linewidth]{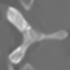}
\end{subfigure}%
\hspace{0.03cm}
\begin{subfigure}[t]{.14\textwidth}
  \includegraphics[width=\linewidth]{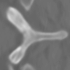}
\end{subfigure}%

\begin{subfigure}[t]{.14\textwidth}
  \includegraphics[width=\linewidth]{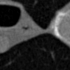}
  \caption*{Ground truth}
\end{subfigure}%
\hfill
\begin{subfigure}[t]{.14\textwidth}
  \includegraphics[width=\linewidth]{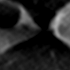}
  \caption*{FBP}
\end{subfigure}%
\hfill
\begin{subfigure}[t]{.14\textwidth}
  \includegraphics[width=\linewidth]{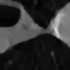}
  \caption*{DIP+TV}
\end{subfigure}%
\hfill
\begin{subfigure}[t]{.14\textwidth}
  \includegraphics[width=\linewidth]{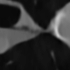}
  \caption*{EPLL}
\end{subfigure}%
\hfill
\begin{subfigure}[t]{.14\textwidth}
  \includegraphics[width=\linewidth]{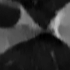}
  \caption*{localAR}
\end{subfigure}%
\hfill
\begin{subfigure}[t]{.14\textwidth}
  \includegraphics[width=\linewidth]{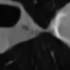}
  \caption*{patchNR}
\end{subfigure}%
\hspace{0.03cm}
\begin{subfigure}[t]{.14\textwidth}
  \includegraphics[width=\linewidth]{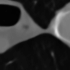}
  \captionsetup{justification=centering}
  \caption*{{FBP+UNet \\(data-based)}}
\end{subfigure}%
\caption{Limited angle reconstruction of the ground truth CT image using different methods. The zoomed-in part is marked with a white box in the ground truth image.
The improvement of the image quality by our method is even better visible than in Figure \ref{Fig_CT_img}.
\textit{Top}: full image. \textit{Middle and bottom}: zoomed-in parts.} \label{Fig_CT_limited}
\end{figure*}

Now we consider the limited angle CT reconstruction problem, i.e., instead of considering equidistant angles between $0$ and $\pi$ we only have a subset of angles. In our experiment we cut off the first and last 100 angles, i.e., we cut off $36^{\circ}$ out of $180^{\circ}$. This leads to a much worse FBP reconstruction. In Figure \ref{Fig_CT_limited}, we compare the different reconstruction methods for the limited angle problem. Although the FBP shows a very bad reconstruction in the $36^{\circ}$ part, where the angles are cut off, the patchNR can reconstruct these details well and in a realistic manner. In particular, the edges of patchNR reconstruction are preserved, while for the other methods these have a pronounced blur, see Table \ref{Table_averagederrorMeasures_limitedangleCT} for an average of the quality measures over the first 100 test images.

\paragraph{Empirical convergence analysis for full angle CT}
To reconstruct the ground truth image from given measurements $y$, we minimize the functional $\mathcal J(x;y)$ in  \eqref{eq:PatchNR_VariationalFormulation} w.r.t. $x$ using the Adam optimizer. The resulting optimization problem is non-convex and the final minimizer could depend on the initialization. In our experiments, it has proven useful to start the optimization with a rough reconstruction. For both full angle CT and limited angle CT we choose a FBP reconstruction. In order to empirically test the convergence of $\mathcal J(x;y)$ we evaluated the PSNR during the optimization process. In Figure \ref{fig:PSNRperIteration} we visualize the PSNR per iteration for the first two images of the test dataset and show reconstructions at iteration $0$, $150$ and $300$. It can be seen that the PSNR is steadily rising during the optimization. Arguably for the left image in Figure \ref{fig:PSNRperIteration} we could have chosen even more iterations.

\begin{figure*}[t]
\begin{subfigure}{0.49\textwidth}
      \centering
    \includegraphics[width=\linewidth]{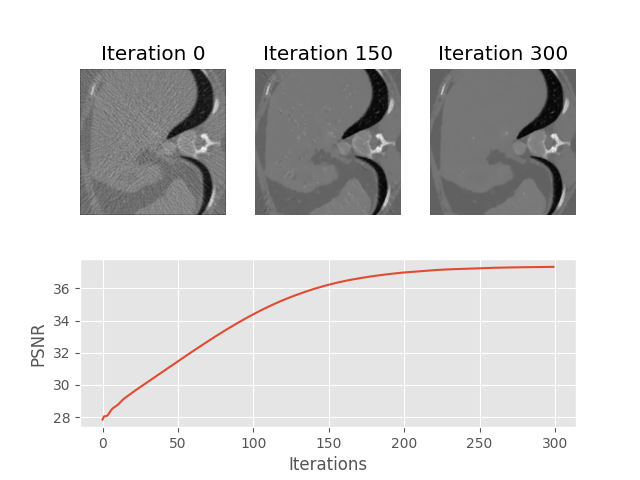}
\end{subfigure}
\begin{subfigure}{0.49\textwidth}
      \centering
    \includegraphics[width=\linewidth]{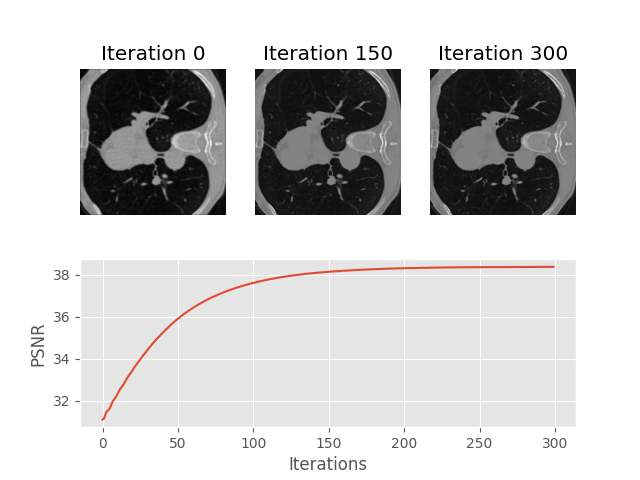}
\end{subfigure}
    \caption{The PSNR for the first two test images per iteration of the optimization process.}
    \label{fig:PSNRperIteration}
\end{figure*}

\paragraph{Ablation studies for full angle CT}
First, we trained the patchNR for patch sizes $4 \times 4, 6 \times 6, 8 \times 8$ and $10 \times 10$. In Table \ref{Table_Ablation_PatchSize} we tested the sensitivity w.r.t. the patch size. For all patch sizes we extracted 40000 patches per iteration and used the optimal regularization parameter $\lambda$ (which is set to 1600, 700, 400 and 250 for the patch sizes $4 \times 4, 6 \times 6, 8 \times 8$ and $10 \times 10$, respectively). We can observe that a larger patch size lead to slightly more blurry images.
However the PSNR seems to change very little within different patch sizes, therefore we observe that our method is quite robust against the choice of the patch size. 

\begin{table*}[b]
\begin{center}
\scalebox{.85}{
\begin{tabular}[t]{c|ccccc|c} 
Patch size  & $s = 4\times4$           & s = $6\times6$           &  $s = 8\times8$         & $s = 10\times10$ \\
\hline
PSNR        & 35.00 $\pm$ 4.45  & 35.19 $\pm$ 4.52  & 35.20 $\pm$ 4.58  &  35.17 $\pm$ 4.56 \\ 
SSIM        & 0.825 $\pm$ 0.153 & 0.829 $\pm$ 0.152 & 0.827 $\pm$ 0.154 & 0.828 $\pm$ 0.154  \\
\end{tabular}}
\caption{Comparison of full angle CT results for different patch sizes. Averaged quality measures and standard deviations of the reconstructions. }  
\label{Table_Ablation_PatchSize}
\end{center}
\end{table*}

Next, in Table \ref{Table_Ablation_NumberOfPatches} we show reconstruction results
for different numbers of patches used per iteration. Here we consider the patch size $6 \times 6$ and use the regularization parameter $\lambda = 700$. We see that the PSNR slightly increases with number of patches, while the SSIM seems to get worse at some point. 

\begin{table*}[t]
\begin{center}
\scalebox{.74}{
\begin{tabular}[t]{c|cccccc} 
Extracted patches per iteration
            & 20000           & 30000           &  40000          & 50000           & 60000 \\
\hline
PSNR        & 35.04 $\pm$ 4.39  & 35.16 $\pm$ 4.49  & 35.19 $\pm$ 4.52  & 35.21 $\pm$ 4.55  & 35.21 $\pm$ 4.56 \\ 
SSIM        & 0.829 $\pm$ 0.148 & 0.829 $\pm$ 0.151 & 0.829 $\pm$ 0.152 & 0.828 $\pm$ 0.153 & 0.828 $\pm$ 0.154 \\
\end{tabular}}
\caption{Comparison of full angle CT results for varying number of patches per iteration. Averaged quality measures and standard deviations of the reconstructions.}  
\label{Table_Ablation_NumberOfPatches}
\end{center}
\end{table*}

\begin{table*}[b]
\begin{center}
\scalebox{.85}{
\begin{tabular}[t]{c|cccccc} 
Number of training images    
            & 1               & 6            &  50 \\
\hline
PSNR        & 33.68 $\pm$ 3.57  & 35.19 $\pm$ 4.52  & 35.24 $\pm$ 4.60  \\ 
SSIM        & 0.802 $\pm$ 0.127 & 0.829 $\pm$ 0.152 & 0.827 $\pm$ 0.156 \\
\end{tabular}}
\caption{Comparison of full angle CT results for different sets of 6 ground truth images.
Averaged quality measures and standard deviations of the reconstructions. } 
\label{Table_Ablation_TrainingSetSize}
\end{center}
\end{table*}

Finally, we examine the choice of the training set. 
In Table \ref{Table_Ablation_TrainingSetSize} we evaluate the patchNR with patch size $6 \times 6$ and $\lambda = 700$, when trained on $1$, $6$ or $50$ images. Obviously, for the CT dataset one training image is not enough to learn the patch distribution. This can be explained by the diversity of the CT dataset, see e.g. Figure \ref{fig_training_ground_truth}.

In Table \ref{Table_Ablation_ChoiceOfTrainingSet} we varied the training set of $6$ images and evaluated the model on $3$ different choices. In total we trained the patchNR 15 times on 6 randomly chosen training images of the LoDoPaB dataset and then evaluated on the test set. Again we consider the patch size $6 \times 6$ and a regularization parameter $\lambda = 700$
and used 40000 extracted patches per iteration. Note that the bad case in Table \ref{Table_Ablation_ChoiceOfTrainingSet} comes from a very noisy ground truth training set of the patchNR.

Overall, we see that the method is quite robust towards certain hyperparameter changes, and it can even be a matter of taste which ones to prefer as the image metrics do not always agree.

\begin{table*}[t]
\begin{center}
\scalebox{.85}{
\begin{tabular}[t]{c|ccc|cc} 
   
            & worst run         & our run           & best run           & mean $\pm$ standard deviation \\
\hline
PSNR        & 34.90 $\pm$ 4.39  & 35.19 $\pm$ 4.52  & 35.26 $\pm$ 4.60   & 35.13 $\pm$ 0.09 \\ 

SSIM        & 0.825 $\pm$ 0.153 & 0.829 $\pm$ 0.152 & 0.828 $\pm$ 0.154  & 0.827 $\pm$ 0.001 \\
\end{tabular}}
\caption{ 40000 extracted patches per iteration. Patch size $s=6\times6$. 
Regularization parameter $\lambda = 700$. 
6 random training images. 
Averaged quality measures and standard deviations of the reconstructions. }  
\label{Table_Ablation_ChoiceOfTrainingSet}
\end{center}
\end{table*}

\subsection{Superresolution} \label{sec_superres}

We choose the forward operator $f$ as a convolution with a $16 \times 16$ 
Gaussian blur kernel with standard deviation $2$ and subsampling with stride $4$. To keep the dimensions consistent, we use zero-padding. 
For the experiments, we extract a dataset of 2D slices of size $600 \times 600$ from a 
3D material image of size $2560\times 2560\times 2120$,
which has been acquired by synchrotron micro-computed tomography. 
We consider a composite ("SiC Diamonds") obtained by microwave sintering of silicon and diamonds, see \cite{vaucher2007line}. 
We generate low resolution images by using the predefined forward operator and adding additive Gaussian noise with standard deviation $\sigma = 0.01$, i.e. we have
\begin{align} 
y = f(x) + \eta,~\text{where}~\eta \sim \mathcal{N}(0,\sigma^2 I). 
\end{align}
Consequently, from a Bayesian viewpoint the negative log likelihood $- \log(p_{Y | X=x}(y))$ can be, up to a constant, rewritten as
\begin{align*}
- \log(p_{Y | X=x}(y)) \varpropto - \log \big( \exp ( - \Vert f(x) - y \Vert^2 / (2 \sigma^2) \big) = \frac{1}{2 \sigma^2} \Vert f(x) - y \Vert^2.
\end{align*}
Thus the concrete form of \eqref{eq:PatchNR_VariationalFormulation} is given by
\begin{align} \label{equ_gaussianmodel}
\mathcal{J}(x) = \frac{1}{2 \sigma^2} \Vert f(x) - y \Vert^2 + \rho \mathcal{R}(x) = \Vert f(x) - y \Vert^2 + \lambda \mathcal{R}(x),
\end{align}
with $\lambda = 2\rho\sigma^2$.

\begin{table*}[b]
\begin{center}
\scalebox{.68}{
\begin{tabular}[t]{c|cccccc|c} 
            & bicubic            & DPIR              & DIP+TV            & EPLL              &  WPP               & patchNR                 & ACNN      \\
            & (not shown)        & (not shown)       &                   &                   &                    &                         & (data-based) \\
\hline
PSNR        &  25.63 $\pm$ 0.56  & 27.78 $\pm$ 0.53  & 27.99 $\pm$ 0.54  & 28.11 $\pm$ 0.55  & 27.80 $\pm$ 0.37   & \textbf{28.53} $\pm$ 0.49 & 28.89 $\pm$ 0.53 \\ 
SSIM        &  0.699 $\pm$ 0.012 & 0.770 $\pm$ 0.011 & 0.764 $\pm$ 0.007 & 0.779 $\pm$ 0.010 & 0.749 $\pm$ 0.011  & \textbf{0.780} $\pm$ 0.008 & 0.804 $\pm$ 0.010  \\\hline
Runtime     & 0.0002s          & 56.62s             & 234.00s            & 60.28s            & 387.28s            &  150.79s                & 0.03s  
\end{tabular}}
\caption{Superresolution. Averaged quality measures and standard deviations of the high resolution reconstructions. } 
\label{Table_averagederrorMeasures}
\end{center}
\end{table*}

\begin{figure*}
\centering
\begin{subfigure}[t]{.14\textwidth}
  \includegraphics[width=\linewidth]{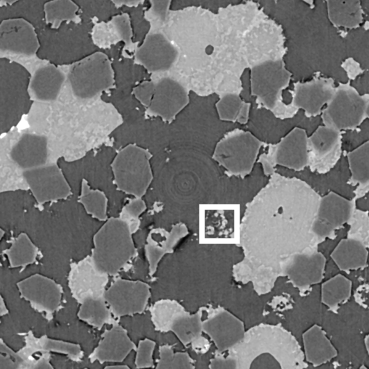}
\end{subfigure}%
\hfill
\begin{subfigure}[t]{.14\textwidth}
  \includegraphics[width=\linewidth]{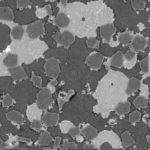}
\end{subfigure}%
\hfill
\begin{subfigure}[t]{.14\textwidth}
  \includegraphics[width=\linewidth]{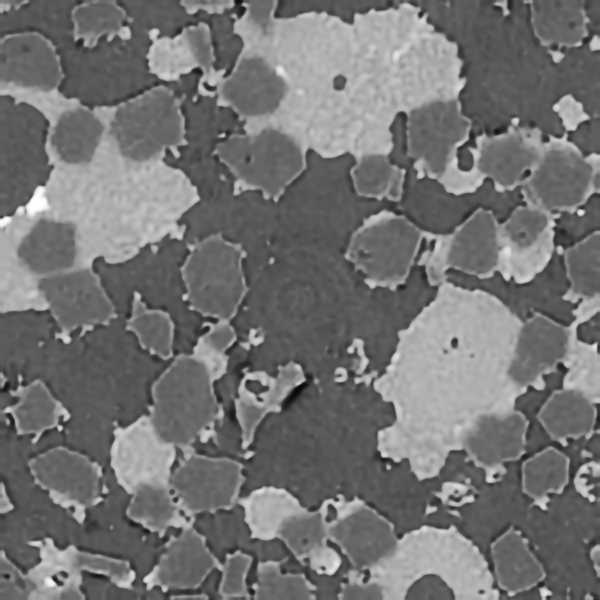}
\end{subfigure}%
\hfill
\begin{subfigure}[t]{.14\textwidth}
  \includegraphics[width=\linewidth]{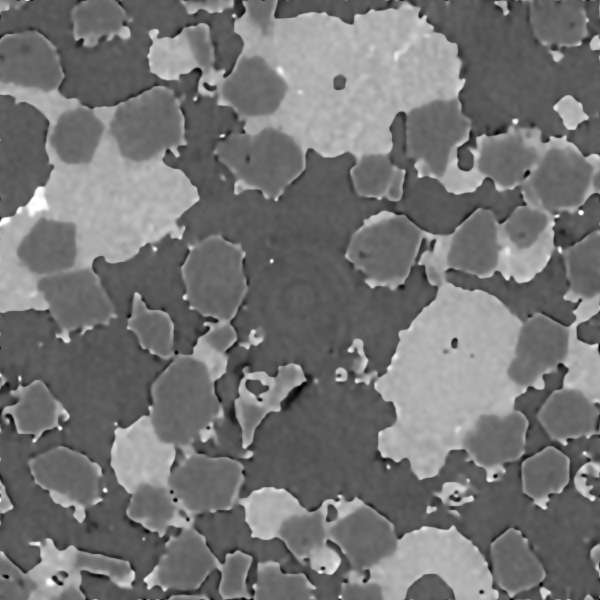}
\end{subfigure}%
\hfill
\begin{subfigure}[t]{.14\textwidth}
  \includegraphics[width=\linewidth]{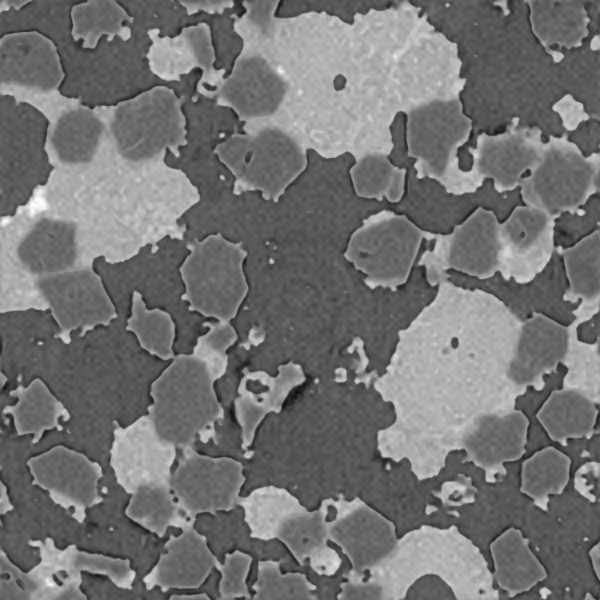}
\end{subfigure}%
\hfill
\begin{subfigure}[t]{.14\textwidth}
  \includegraphics[width=\linewidth]{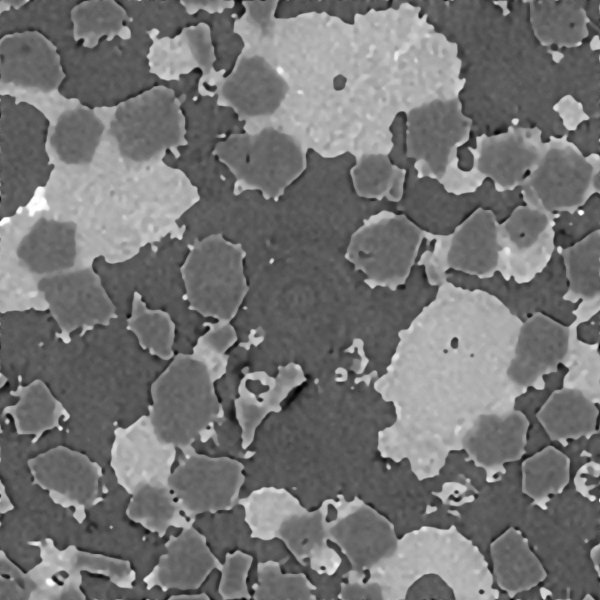}
\end{subfigure}%
\hspace{0.03cm}
\begin{subfigure}[t]{.14\textwidth}
  \includegraphics[width=\linewidth]{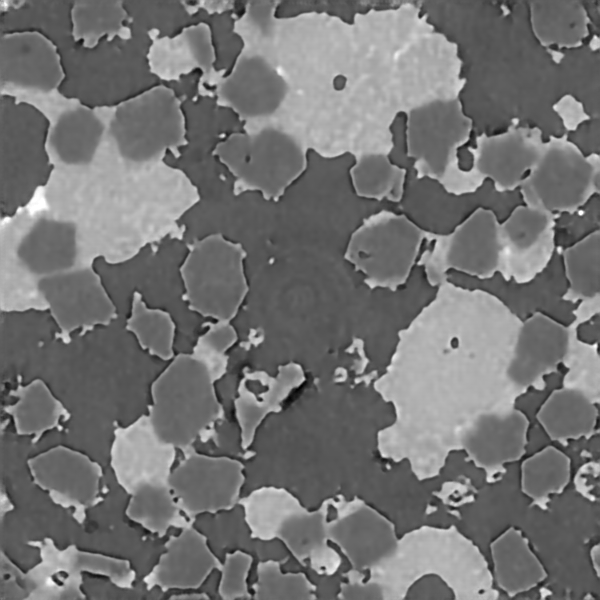}
\end{subfigure}%

\begin{subfigure}[t]{.14\textwidth}
  \includegraphics[width=\linewidth]{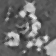}
  \caption*{HR}
\end{subfigure}%
\hfill
\begin{subfigure}[t]{.14\textwidth}
  \includegraphics[width=\linewidth]{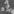}
  \caption*{LR}
\end{subfigure}%
\hfill
\begin{subfigure}[t]{.14\textwidth}
  \includegraphics[width=\linewidth]{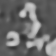}
  \caption*{DIP + TV}
\end{subfigure}%
\hfill
\begin{subfigure}[t]{.14\textwidth}
  \includegraphics[width=\linewidth]{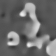}
  \caption*{EPLL}
\end{subfigure}%
\hfill
\begin{subfigure}[t]{.14\textwidth}
  \includegraphics[width=\linewidth]{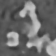}
  \caption*{WPP}
\end{subfigure}%
\hfill
\begin{subfigure}[t]{.14\textwidth}
  \includegraphics[width=\linewidth]{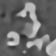}
  \caption*{patchNR}
\end{subfigure}%
\hspace{0.03cm}
\begin{subfigure}[t]{.14\textwidth}
  \includegraphics[width=\linewidth]{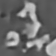}
  \captionsetup{justification=centering}
  \caption*{ACNN \\(data-based)}
\end{subfigure}%
\caption{Comparison of different methods for superresolution. 
The zoomed-in part is marked with a white box in the ground truth image.
Having in particular a look to disconnected parts in the lower right corner the
improvement of the reconstruction by our method becomes evident.
\textit{Top}: full image. \textit{Bottom}: zoomed-in part.
} \label{Fig_comp_SR}
\end{figure*}

The patchNR is trained on patches of only $M=1$ example image of size $600 \times 600$. For reconstruction, we use the regularization parameter $\lambda = 0.15 \frac{s}{N_p}$, the random subset of overlapping patches is of size $N_p = 130000$ in each iteration and we optimize over 300 iterations using the Adam optimizer with a learning rate of 0.03.
Since we do not want to consider boundary effects, we cut off a boundary of 40 pixels before evaluating the quality measures.

In Figure \ref{Fig_comp_SR} we compare different methods for reconstructing the high resolution image $x$ from the given low-resolution observation $y$. As initialization  we choose the bicubic interpolation. The  patchNR yields very clear and better images than the other model-based methods, visually and in terms of the quality metrics; see Table \ref{Table_averagederrorMeasures} for an average over 100 test images. In particular, the reconstruction of patchNR is less blurry than the DIP+TV and WPP reconstruction, specifically in regions between  edges.

\subsection{Zero-shot Superresolution with PatchNRs}\label{app_zero_shot}

In the case of superresolution, we can train the patchNR even without any training image. To this end, we combine some concepts of zero-shot superresolution by internal learning \cite{GBI2009,SCI2018} with patchNRs.
In these approaches, the main assumption is that the patch distribution of natural images is self-similar across the scales. 
Consequently, the patch distributions of the same image at different resolutions should be similar.
Motivated by this observation, we train the patchNR on the low-resolution observations such that we do not longer require any sample from the high-resolution ground truth.

In the following, we consider the case, where we have given one single low-resolution observation at training time and additionally the forward operator at test time. In particular, we do not require access to any high-resolution ground truth image and therefore the method is fully unsupervised. 
We train the patchNR on the patches from the low-resolution observation , where the training data is enriched by rotating and mirror reflecting of the patches such that we get 8 times more training patches. Note that in this setting the patchNR needs to be retrained for every new observation.

\begin{table*}[t]
\begin{center}
\scalebox{.8}{
\begin{tabular}[t]{c|cccccc} 
            & $L^2$-TV          & DIP+TV            & ZSSR               & DualSR                     &      patchNR \\
\hline 
PSNR        & 28.35 $\pm$ 3.55  & 28.44 $\pm$ 3.69  & 28.83 $\pm$ 3.57   &  28.64 $\pm$ 3.47          & \textbf{29.08} $\pm$ 3.58  \\
SSIM        & 0.820 $\pm$ 0.072 & 0.821 $\pm$ 0.087 & 0.834 $\pm$ 0.066   &  0.829 $\pm$ 0.061         & \textbf{0.846} $\pm$ 0.061 \\
\hline
Runtime     &  13.12s           & 171.51s           & 56.64s             &      53.47s                & 132.36s
\end{tabular}}
\caption{Zero-shot superresolution. Averaged quality measures and standard deviations of the reconstructions of BSD68 dataset.}  
\label{Table_averagederrorMeasures_zero_shot}
\end{center}
\end{table*}

\begin{figure*}[b!]
\centering
\begin{subfigure}[t]{.14\textwidth}
  \includegraphics[width=\linewidth]{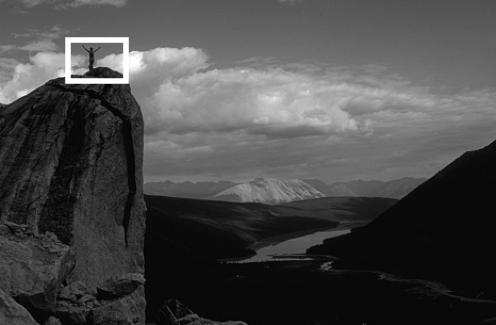}
\end{subfigure}%
\hfill
\begin{subfigure}[t]{.14\textwidth}
  \includegraphics[width=\linewidth]{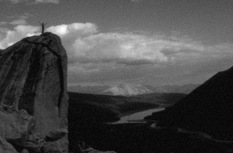}
\end{subfigure}%
\hfill
\begin{subfigure}[t]{.14\textwidth}
  \includegraphics[width=\linewidth]{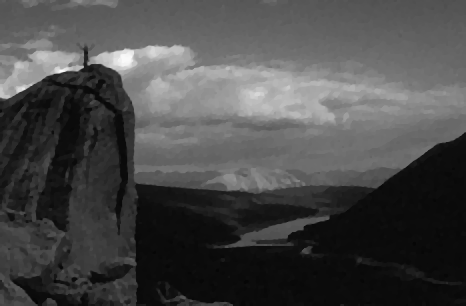}
\end{subfigure}%
\hfill
\begin{subfigure}[t]{.14\textwidth}
  \includegraphics[width=\linewidth]{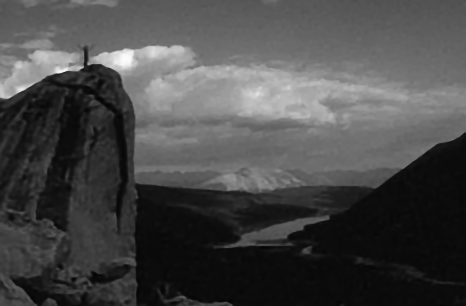}
\end{subfigure}%
\hfill
\begin{subfigure}[t]{.14\textwidth}
  \includegraphics[width=\linewidth]{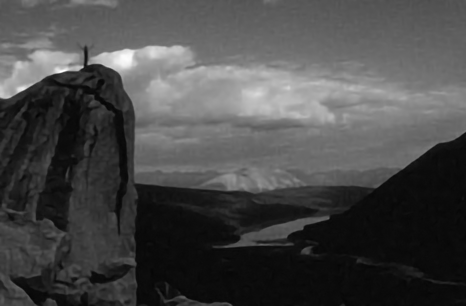}
\end{subfigure}%
\hfill
\begin{subfigure}[t]{.14\textwidth}
  \includegraphics[width=\linewidth]{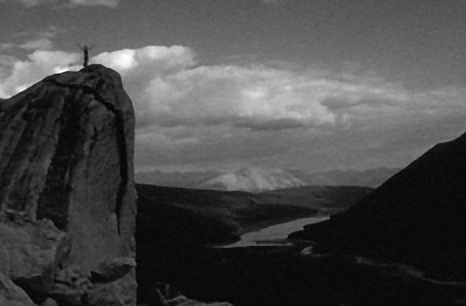}
\end{subfigure}%
\hfill
\begin{subfigure}[t]{.14\textwidth}
  \includegraphics[width=\linewidth]{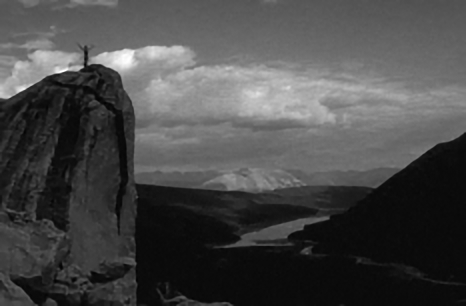}
\end{subfigure}%

\begin{subfigure}[t]{.14\textwidth}
  \includegraphics[width=\linewidth]{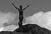}
\end{subfigure}%
\hfill
\begin{subfigure}[t]{.14\textwidth}
  \includegraphics[width=\linewidth]{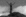}
\end{subfigure}%
\hfill
\begin{subfigure}[t]{.14\textwidth}
  \includegraphics[width=\linewidth]{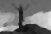}
\end{subfigure}%
\hfill
\begin{subfigure}[t]{.14\textwidth}
  \includegraphics[width=\linewidth]{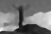}
\end{subfigure}%
\hfill
\begin{subfigure}[t]{.14\textwidth}
  \includegraphics[width=\linewidth]{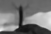}
\end{subfigure}%
\hfill
\begin{subfigure}[t]{.14\textwidth}
  \includegraphics[width=\linewidth]{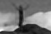} 
\end{subfigure}%
\hfill
\begin{subfigure}[t]{.14\textwidth}
  \includegraphics[width=\linewidth]{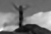}
\end{subfigure}%

\centering
\begin{subfigure}[t]{.14\textwidth}
  \includegraphics[width=\linewidth]{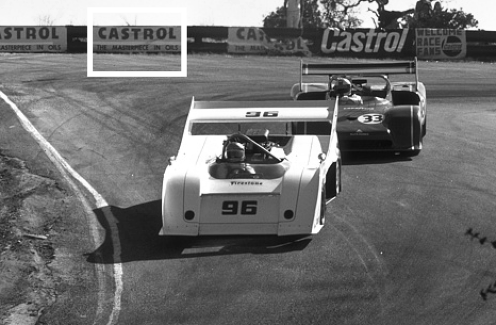}
\end{subfigure}%
\hfill
\begin{subfigure}[t]{.14\textwidth}
  \includegraphics[width=\linewidth]{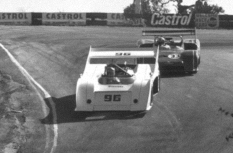}
\end{subfigure}%
\hfill
\begin{subfigure}[t]{.14\textwidth}
  \includegraphics[width=\linewidth]{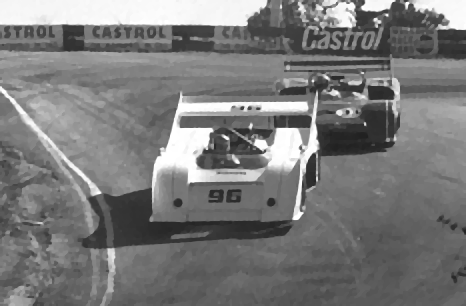}
\end{subfigure}%
\hfill
\begin{subfigure}[t]{.14\textwidth}
  \includegraphics[width=\linewidth]{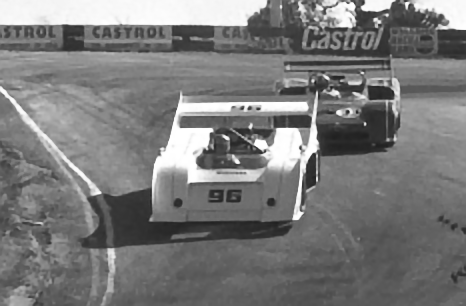}
\end{subfigure}%
\hfill
\begin{subfigure}[t]{.14\textwidth}
  \includegraphics[width=\linewidth]{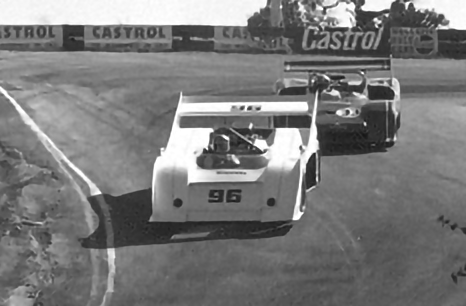}
\end{subfigure}%
\hfill
\begin{subfigure}[t]{.14\textwidth}
  \includegraphics[width=\linewidth]{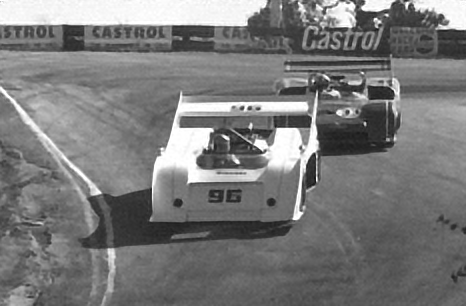}
\end{subfigure}%
\hfill
\begin{subfigure}[t]{.14\textwidth}
  \includegraphics[width=\linewidth]{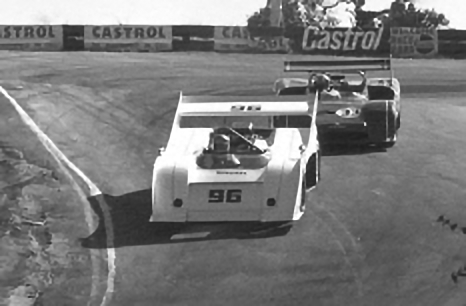}
\end{subfigure}%

\begin{subfigure}[t]{.14\textwidth}
  \includegraphics[width=\linewidth]{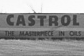}
\end{subfigure}%
\hfill
\begin{subfigure}[t]{.14\textwidth}
  \includegraphics[width=\linewidth]{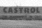}
\end{subfigure}%
\hfill
\begin{subfigure}[t]{.14\textwidth}
  \includegraphics[width=\linewidth]{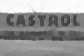}
\end{subfigure}%
\hfill
\begin{subfigure}[t]{.14\textwidth}
  \includegraphics[width=\linewidth]{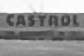}
\end{subfigure}%
\hfill
\begin{subfigure}[t]{.14\textwidth}
  \includegraphics[width=\linewidth]{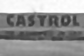}
\end{subfigure}%
\hfill
\begin{subfigure}[t]{.14\textwidth}
  \includegraphics[width=\linewidth]{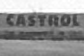}
\end{subfigure}%
\hfill
\begin{subfigure}[t]{.14\textwidth}
  \includegraphics[width=\linewidth]{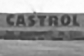}
\end{subfigure}%

\centering
\begin{subfigure}[t]{.14\textwidth}
  \includegraphics[width=\linewidth]{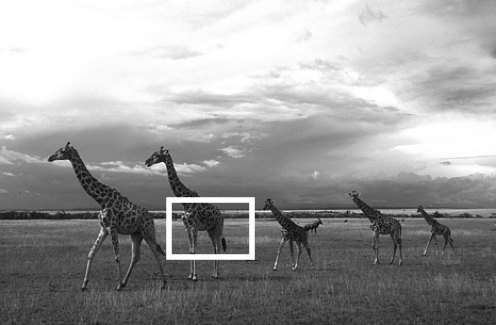}
\end{subfigure}%
\hfill
\begin{subfigure}[t]{.14\textwidth}
  \includegraphics[width=\linewidth]{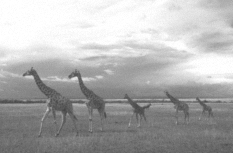}
\end{subfigure}%
\hfill
\begin{subfigure}[t]{.14\textwidth}
  \includegraphics[width=\linewidth]{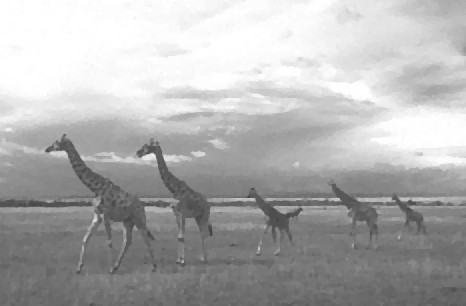}
\end{subfigure}%
\hfill
\begin{subfigure}[t]{.14\textwidth}
  \includegraphics[width=\linewidth]{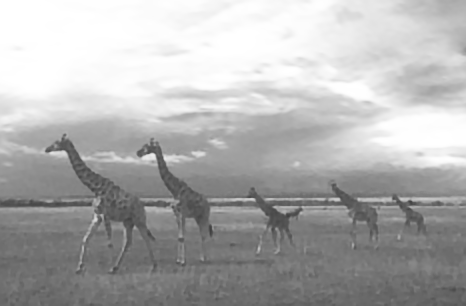}
\end{subfigure}%
\hfill
\begin{subfigure}[t]{.14\textwidth}
  \includegraphics[width=\linewidth]{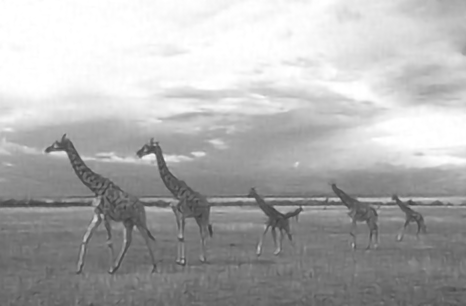}
\end{subfigure}%
\hfill
\begin{subfigure}[t]{.14\textwidth}
  \includegraphics[width=\linewidth]{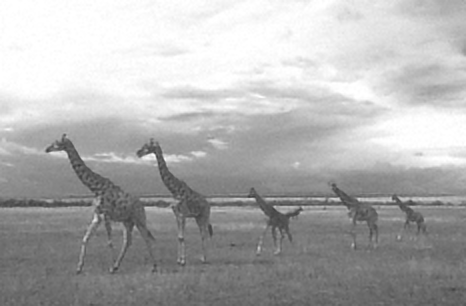}
\end{subfigure}%
\hfill
\begin{subfigure}[t]{.14\textwidth}
  \includegraphics[width=\linewidth]{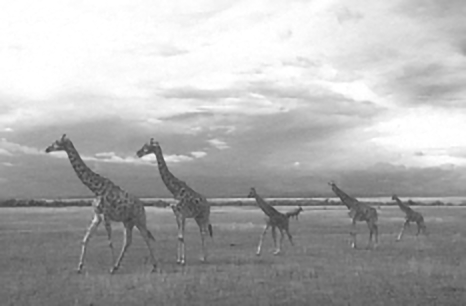}
\end{subfigure}%

\begin{subfigure}[t]{.14\textwidth}
  \includegraphics[width=\linewidth]{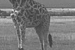}
  \caption*{HR}
\end{subfigure}%
\hfill
\begin{subfigure}[t]{.14\textwidth}
  \includegraphics[width=\linewidth]{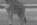}
  \caption*{LR}
\end{subfigure}%
\hfill
\begin{subfigure}[t]{.14\textwidth}
  \includegraphics[width=\linewidth]{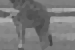}
    \caption*{$L^2$-TV}
\end{subfigure}%
\hfill
\begin{subfigure}[t]{.14\textwidth}
  \includegraphics[width=\linewidth]{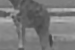}
  \caption*{DIP+TV}
\end{subfigure}%
\hfill
\begin{subfigure}[t]{.14\textwidth}
  \includegraphics[width=\linewidth]{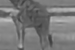}
  \caption*{ZSSR}  
\end{subfigure}%
\hfill
\begin{subfigure}[t]{.14\textwidth}
  \includegraphics[width=\linewidth]{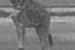}
  \caption*{DualSR}    
\end{subfigure}%
\hfill
\begin{subfigure}[t]{.14\textwidth}
  \includegraphics[width=\linewidth]{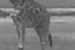}
  \caption*{patchNR}  
\end{subfigure}%
\caption{Zero-shot superresolution for three images from the BSD68 dataset. The zoomed-in part is marked with a white box in the ground truth image. \textit{Top}: full image. \textit{Bottom}: zoomed-in part.} \label{fig:zero_shot}
\end{figure*}

First we use a convolution with a $16 \times 16$ Gaussian blur kernel with standard deviation 1 and stride 2 as forward operator and add Gaussian noise with standard deviation 0.01 on the low-resolution observation. The patchNR is trained for 10000 optimizer steps with a learning rate of $0.0001$, a batch size of 128 and a patch size of $6 \times 6$.
Then, for reconstruction, we use a random subset of $N_p =  80000$ patches per iteration, a regularization parameter $\lambda = 0.25 \frac{s}{N_p}$ and optimize over 60 iterations using the Adam optimizer with a learning rate of 0.01. As comparison baselines we use $L^2$-TV \cite{ROF1992}, DIP+TV, ZSSR \cite{SCI2018} and DualSR \cite{EPC2021}. We test the methods on the BSD68 dataset \cite{martin2001database}and the resulting quality measures are given in Table \ref{Table_averagederrorMeasures_zero_shot}. In Figure~\ref{fig:zero_shot} we present three reconstruction examples of the test set. Reconstructions of the patchNR lead to less blurry images and in particular structures and edges are preserved. In contrast,  in $L^2$-TV and DIP+TV some parts of the images are smoothed out.

In a second experiment we consider the same forward operator as in Section~\ref{sec_superres}, that is a convolution with a $16 \times 16$ Gaussian blur kernel with standard deviation 2 and stride 4, and add Gaussian noise with standard deviation 0.01 on the low-resolution observation. The patchNR is trained for 10000 optimizer steps with a learning rate of $0.0001$, a batch size of 128 and a patch size of $6 \times 6$.
Then, for reconstruction, we use a random subset of $N_p =  50000$ patches per iteration, a regularization parameter $\lambda = 0.25 \frac{s}{N_p}$ and optimize over 60 iterations using the Adam optimizer with a learning rate of 0.01 

\begin{table*}[t]
\begin{center}
\scalebox{.8}{
\begin{tabular}[t]{c|cccccc} 
            & $L^2$-TV          & DIP+TV            & ZSSR               & DualSR                     &      patchNR \\
\hline 
PSNR        & 27.85 $\pm$ 0.55  & \textbf{27.99} $\pm$ 0.54  & 27.44 $\pm$ 0.55   &  27.64 $\pm$ 0.57          & 27.94 $\pm$ 0.55  \\
SSIM        & 0.768 $\pm$ 0.008 & 0.764 $\pm$ 0.007 & 0.758 $\pm$ 0.008   &  0.764 $\pm$ 0.008         & \textbf{0.776} $\pm$ 0.009 \\
\hline
Runtime     &  38.11s           & 234.00s           & 42.43s             &      46.55s                & 120.47s
\end{tabular}}
\caption{Zero-shot superresolution. Averaged quality measures and standard deviations of the high-resolution reconstructions.}  
\label{Table_averagederrorMeasures_zero_shot_SiC}
\end{center}
\end{table*}
\begin{figure*}[b!]
\centering
\begin{subfigure}[t]{.14\textwidth}
  \includegraphics[width=\linewidth]{images/SiC_x4/img_hr_rectangle.png}
\end{subfigure}%
\hfill
\begin{subfigure}[t]{.14\textwidth}
  \includegraphics[width=\linewidth]{images/SiC_x4/img_lr.png}
\end{subfigure}%
\hfill
\begin{subfigure}[t]{.14\textwidth}
\includegraphics[width=\linewidth]{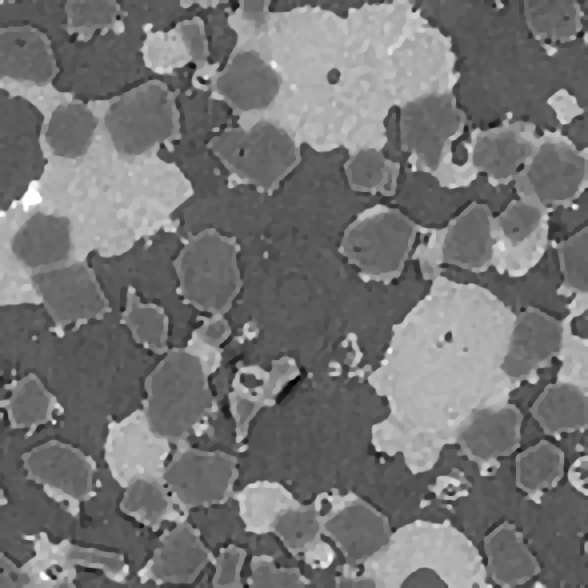}
\end{subfigure}%
\hfill
\begin{subfigure}[t]{.14\textwidth}
  \includegraphics[width=\linewidth]{images/SiC_x4/DIP_tv.png}
\end{subfigure}%
\hfill
\begin{subfigure}[t]{.14\textwidth}
  \includegraphics[width=\linewidth]{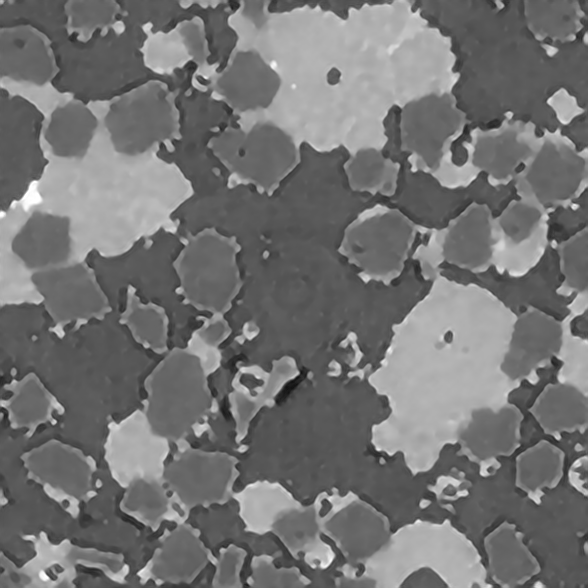}
\end{subfigure}%
\hfill
\begin{subfigure}[t]{.14\textwidth}
  \includegraphics[width=\linewidth]{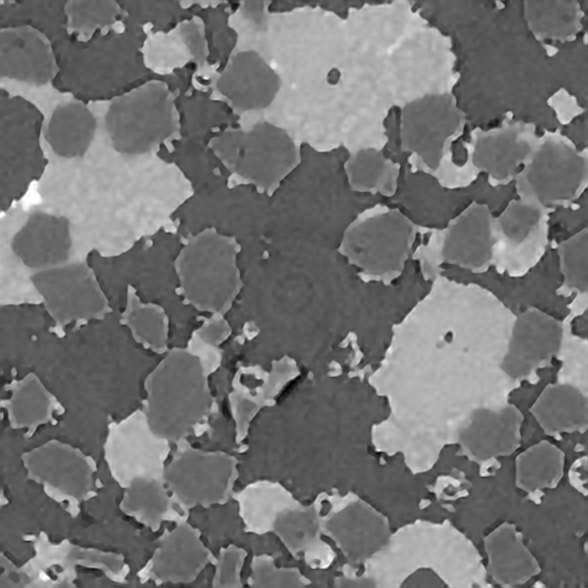}
\end{subfigure}%
\hfill
\begin{subfigure}[t]{.14\textwidth}
  \includegraphics[width=\linewidth]{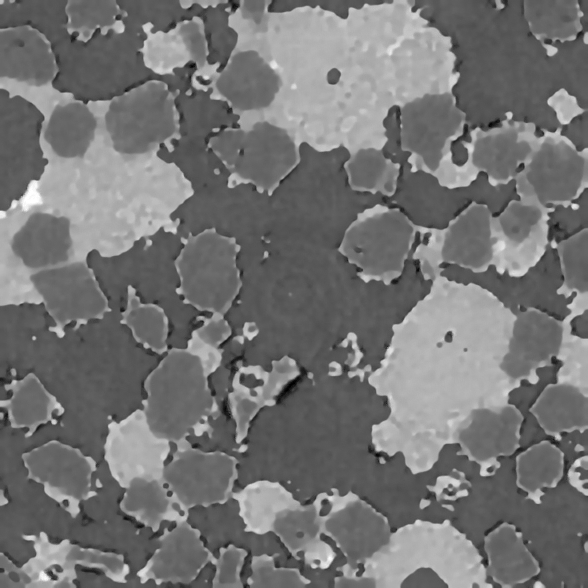}
\end{subfigure}%

\begin{subfigure}[t]{.14\textwidth}
  \includegraphics[width=\linewidth]{images/SiC_x4/hr_zoom2.png}
  \caption*{HR}
\end{subfigure}%
\hfill
\begin{subfigure}[t]{.14\textwidth}
  \includegraphics[width=\linewidth]{images/SiC_x4/lr_zoom2.png}
  \caption*{LR}
\end{subfigure}%
\hfill
\begin{subfigure}[t]{.14\textwidth}
  \includegraphics[width=\linewidth]{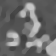}
  \caption*{$L^2$-TV}
\end{subfigure}%
\hfill
\begin{subfigure}[t]{.14\textwidth}
  \includegraphics[width=\linewidth]{images/SiC_x4/diptv_zoom2.png}
  \caption*{DIP + TV}
\end{subfigure}%
\hfill
\begin{subfigure}[t]{.14\textwidth}
  \includegraphics[width=\linewidth]{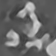}
  \caption*{ZSSR}
\end{subfigure}%
\hfill
\begin{subfigure}[t]{.14\textwidth}
  \includegraphics[width=\linewidth]{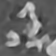}
  \caption*{DualSR}
\end{subfigure}%
\hfill
\begin{subfigure}[t]{.14\textwidth}
  \includegraphics[width=\linewidth]{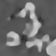}
  \captionsetup{justification=centering}
  \caption*{patchNR}
\end{subfigure}%
\caption{Zero-shot superresolution. The zoomed-in part is marked with a white box in the ground truth image. \textit{Top}: full image. \textit{Bottom}: zoomed-in part.
\textit{Top}: full image. \textit{Bottom}: zoomed-in part.
} \label{Fig_comp_zeroshot_SiC}
\end{figure*}

We test the methods on the same test images as in Section~\ref{sec_superres}. In Figure~\ref{Fig_comp_zeroshot_SiC} we compare the reconstructions of the different methods. We can observe a visually similar quality of the reconstructions for the DIP+TV and patchNR, while the other methods lead to a significant blur in the reconstructions. This can be also seen in the resulting quality measures given in Table~\ref{Table_averagederrorMeasures_zero_shot_SiC}. Note that here we do not consider natural images but material data and the assumption of self-similarity between different scales is not fulfilled. Consequently, we cannot expect a good reconstruction of the both methods ZSSR and DualSR. The methods $L^2$-TV and DIP+TV do not need these assumptions and thus perform well on the data, while there is a slight worsening in the quality of the patchNR in contrast to Section~\ref{sec_superres}. This nicely demonstrates the gain of using a small amount of ground truth data for training the patchNR.

\section{Discussion and Conclusions} \label{sec:conclusions}

In this paper, we introduced patchNRs, which are patch-based NFs used for regularizing the variational modeling of inverse problems. Learning patchNRs requires only few ground truth images. In particular no paired data is necessary.
We demonstrated the performance of our method by numerical examples and showed that it leads to better results than comparable, established methods, visually and in terms of the quality measures. Note that it is not clear how patch based learning influences biases in datasets. Using a small number of images can be very dangerous as these datasets are easily imbalanced. Further research needs to be invested into understanding how many images are sufficient for an adequate patch representation of a dataset and when the patch representation can be used as a prior for inverse problems. Moreover, quality measures for images are not sufficient for judging the quality of an image, in particular in medical applications. Further evaluation with medical expertise needs to be done before drawing any conclusions.
Moreover, all patch-based methods are essentially limited in the sense that they can not capture global image correlations. 
Furthermore, normalizing flows are often not able to capture out of distribution data \cite{outofdist}, so if a patch is far away from the patch "manifold", the likelihoods might be meaningless. However, it is an open question whether patch-based learning mitigates this effect.

The patchNR can be extended into several directions. First, we want to use the regularizer for training NNs in an model-based way for a fast reconstruction of several observations. Then the patchNR can be applied for uncertainty quantification by using, e.g., invertible architectures \cite{AKRK2019,dinhrnvp,HHS2021} or Langevin sampling methods \cite{song2022solving}.

\section*{Acknowledgements}
F.A. acknowledge support from the German Research Foundation (DFG) under Germany`s Excellence Strategy – The Berlin Mathematics
Research Center MATH+ within the project TP: EF3-7,
A.D. from (DFG; GRK 2224/1) and the Klaus Tschira Stiftung via the project MALDISTAR (project number 00.010.2019), 
P.H. from the DFG within the project SPP 2298 "Theoretical Foundations of Deep Learning"
and J.H. by the DFG within the project STE 571/16-1.
The data from Section~\ref{sec_superres} has been acquired in the frame of the EU Horizon 2020 Marie Sklodowska-Curie Actions Innovative Training 
Network MUMMERING (MUltiscale, Multimodal and Multidimensional imaging for EngineeRING, Grant Number 765604) at the beamline 
TOMCAT of the SLS by A. Saadaldin, D. Bernard, and F. Marone Welford. We acknowledge the Paul Scherrer Institut, Villigen, 
Switzerland for provision of synchrotron radiation beamtime at the TOMCAT beamline X02DA of the SLS.    
P.H. thanks Cosmas Heiß and Shayan Hundrieser for fruitful discussions.


\bibliography{literatur_patchINN}

\begin{thebibliography}{10}

\bibitem{Adler2018}
J.~Adler, H.~Kohr, A.~Ringh, J.~Moosmann, S.~Banert, M.~J. Ehrhardt, G.~R. Lee,
  K.~Niinimaki, B.~Gris, O.~Verdier, J.~Karlsson, W.~J. Palenstijn, O.~Öktem,
  C.~Chen, H.~A. Loarca, and M.~Lohmann.
\newblock Operator discretization library {(ODL)}, 2018.

\bibitem{AdlOkt18}
J.~Adler and O.~\"Oktem.
\newblock Learned primal-dual reconstruction.
\newblock {\em IEEE Transactions on Medical Imaging}, 37(6):1322–1332, Jun
  2018.

\bibitem{AH2022}
F.~Altekr{\"u}ger and J.~Hertrich.
\newblock {WPPNets and WPPFlows}: The power of {W}asserstein patch priors for
  superresolution.
\newblock {\em arXiv preprint arXiv:2201.08157}, 2022.

\bibitem{ATKP2018}
R.~Anirudh, J.~J. Thiagarajan, B.~Kailkhura, and T.~Bremer.
\newblock An unsupervised approach to solving inverse problems using generative
  adversarial networks.
\newblock {\em arXiv preprint arXiv:1805.07281}, 2018.

\bibitem{AKRK2019}
L.~Ardizzone, J.~Kruse, C.~Rother, and U.~K{\"{o}}the.
\newblock Analyzing inverse problems with invertible neural networks.
\newblock In {\em 7th International Conference on Learning Representations,
  {ICLR} 2019, New Orleans, LA, USA, May 6-9, 2019}, 2019.

\bibitem{ALKRK2019}
L.~Ardizzone, C.~L{\"u}th, J.~Kruse, C.~Rother, and U.~K{\"o}the.
\newblock Guided image generation with conditional invertible neural networks.
\newblock {\em arXiv preprint arXiv:1907.02392}, 2019.

\bibitem{Armato11}
S.~G. Armato~et al.
\newblock The {Lung Image Database Consortium (LIDC)} and {Image Database
  Resource Initiative (IDRI)}: a completed reference database of lung nodules
  on {CT} scans.
\newblock {\em Medical physics}, 38(2):915--931, 2011.

\bibitem{arridge2019solving}
S.~Arridge, P.~Maass, O.~{\"O}ktem, and C.-B. Sch{\"o}nlieb.
\newblock Solving inverse problems using data-driven models.
\newblock {\em Acta Numerica}, 28:1--174, 2019.

\bibitem{ADLAH2020}
M.~Asim, M.~Daniels, O.~Leong, A.~Ahmed, and P.~Hand.
\newblock Invertible generative models for inverse problems: mitigating
  representation error and dataset bias.
\newblock In H.~D. III and A.~Singh, editors, {\em Proceedings of the 37th
  International Conference on Machine Learning}, volume 119 of {\em Proceedings
  of Machine Learning Research}, pages 399--409. PMLR, 13--18 Jul 2020.

\bibitem{baguer2020computed}
D.~O. Baguer, J.~Leuschner, and M.~Schmidt.
\newblock Computed tomography reconstruction using deep image prior and learned
  reconstruction methods.
\newblock {\em Inverse Problems}, 36(9):094004, 2020.

\bibitem{barbano2021deep}
R.~Barbano, J.~Leuschner, M.~Schmidt, A.~Denker, A.~Hauptmann, P.~Maa{\ss}, and
  B.~Jin.
\newblock Is deep image prior in need of a good education?
\newblock {\em arXiv preprint arXiv:2111.11926}, 2021.

\bibitem{BCM2005}
A.~Buades, B.~Coll, and J.-M. Morel.
\newblock A non-local algorithm for image denoising.
\newblock In {\em 2005 IEEE Computer Society Conference on Computer Vision and
  Pattern Recognition}, volume~2, pages 60--65. IEEE, 2005.

\bibitem{CBDJ2019}
R.~Chen, J.~Behrmann, D.~K. Duvenaud, and J.-H. Jacobsen.
\newblock Residual flows for invertible generative modeling.
\newblock In {\em Advances in Neural Information Processing Systems},
  volume~32. Curran Associates, Inc., 2019.

\bibitem{Cheng_2021_CVPR}
Z.~Cheng, Z.~Xiong, C.~Chen, D.~Liu, and Z.-J. Zha.
\newblock Light field super-resolution with zero-shot learning.
\newblock In {\em Proceedings of the IEEE/CVF Conference on Computer Vision and
  Pattern Recognition (CVPR)}, pages 10010--10019, June 2021.

\bibitem{DFKE09}
K.~Dabov, A.~Foi, V.~Katkovnik, and K.~Egiazarian.
\newblock {BM3D} image denoising with shape-adaptive principal component
  analysis.
\newblock In {\em SPARS'09-Signal Processing with Adaptive Sparse Structured
  Representations}, 2009.

\bibitem{dahari_gansuperres}
A.~Dahari, S.~Kench, I.~Squires, and S.~J. Cooper.
\newblock Super-resolution of multiphase materials by combining complementary
  2d and 3d image data using generative adversarial networks.
\newblock {\em arXiv preprint arXiv:2110.11281}, 2021.

\bibitem{delon2018gaussian}
J.~Delon and A.~Houdard.
\newblock Gaussian priors for image denoising.
\newblock In {\em Denoising of Photographic Images and Video}, pages 125--149.
  Springer, 2018.

\bibitem{dinhrnvp}
L.~Dinh, J.~Sohl{-}Dickstein, and S.~Bengio.
\newblock Density estimation using real {NVP}.
\newblock In {\em 5th International Conference on Learning Representations,
  {ICLR} 2017, Toulon, France, April 24-26, 2017, Conference Track
  Proceedings}. OpenReview.net, 2017.

\bibitem{Duff21}
M.~Duff, N.~D.~F. Campbell, and M.~J. Ehrhardt.
\newblock Regularising inverse problems with generative machine learning
  models.
\newblock {\em arXiv preprint arXiv:2107.11191}, 2021.

\bibitem{EPC2021}
M.~Emad, M.~Peemen, and H.~Corporaal.
\newblock {DualSR}: Zero-shot dual learning for real-world super-resolution.
\newblock In {\em Proceedings of the IEEE/CVF Winter Conference on Applications
  of Computer Vision}, pages 1629--1638, 2021.

\bibitem{FW2021}
R.~Friedman and Y.~Weiss.
\newblock Posterior sampling for image restoration using explicit patch priors.
\newblock {\em arXiv preprint arXiv:2104.09895}, 2021.

\bibitem{GW2000}
C.~R. Genovese and L.~Wasserman.
\newblock Rates of convergence for the {G}aussian mixture sieve.
\newblock {\em The Annals of Statistics}, 28(4):1105--1127, 2000.

\bibitem{gilton2019}
D.~Gilton, G.~Ongie, and R.~Willett.
\newblock Learned patch-based regularization for inverse problems in imaging.
\newblock In {\em 2019 IEEE 8th International Workshop on Computational
  Advances in Multi-Sensor Adaptive Processing (CAMSAP)}, pages 211--215. IEEE,
  2019.

\bibitem{GBI2009}
D.~Glasner, S.~Bagon, and M.~Irani.
\newblock Super-resolution from a single image.
\newblock In {\em 2009 IEEE 12th international conference on computer vision},
  pages 349--356. IEEE, 2009.

\bibitem{gangoodfellow}
I.~Goodfellow, J.~Pouget-Abadie, M.~Mirza, B.~Xu, D.~Warde-Farley, S.~Ozair,
  A.~Courville, and Y.~Bengio.
\newblock Generative adversarial nets.
\newblock In Z.~Ghahramani, M.~Welling, C.~Cortes, N.~Lawrence, and
  K.~Weinberger, editors, {\em Advances in Neural Information Processing
  Systems}, volume~27. Curran Associates, Inc., 2014.

\bibitem{Granot_2022_CVPR}
N.~Granot, B.~Feinstein, A.~Shocher, S.~Bagon, and M.~Irani.
\newblock Drop the gan: In defense of patches nearest neighbors as single image
  generative models.
\newblock In {\em Proceedings of the IEEE/CVF Conference on Computer Vision and
  Pattern Recognition (CVPR)}, pages 13460--13469, June 2022.

\bibitem{HHS2021}
P.~Hagemann, J.~Hertrich, and G.~Steidl.
\newblock Stochastic normalizing flows for inverse problems: a {M}arkov
  {C}hains viewpoint.
\newblock {\em arXiv preprint arXiV:2109.11375}, 2021.

\bibitem{HN2021}
P.~Hagemann and S.~Neumayer.
\newblock Stabilizing invertible neural networks using mixture models.
\newblock {\em Inverse Problems}, 37(8):085002, 2021.

\bibitem{Helminger2021GenericIR}
L.~Helminger, M.~Bernasconi, A.~Djelouah, M.~H. Gross, and C.~Schroers.
\newblock Generic image restoration with flow based priors.
\newblock {\em 2021 IEEE/CVF Conference on Computer Vision and Pattern
  Recognition Workshops (CVPRW)}, pages 334--343, 2021.

\bibitem{Hertrich21}
J.~Hertrich, A.~Houdard, and C.~Redenbach.
\newblock Wasserstein patch prior for image superresolution.
\newblock {\em arXiv preprint arXiv:2109.12880}, 2021.

\bibitem{HNS2021}
J.~Hertrich, S.~Neumayer, and G.~Steidl.
\newblock Convolutional proximal neural networks and plug-and-play algorithms.
\newblock {\em Linear Algebra and its Applications}, 631:203--234, 2021.

\bibitem{HNABBSS2020}
J.~Hertrich, D.~P.~L. Nguyen, J.-F. Aujol, D.~Bernard, Y.~Berthoumieu,
  A.~Saadaldin, and G.~Steidl.
\newblock {PCA} reduced {G}aussian mixture models with applications in
  superresolution.
\newblock {\em Inverse Problems \& Imaging}, 2021.

\bibitem{HBD2018}
A.~Houdard, C.~Bouveyron, and J.~Delon.
\newblock High-dimensional mixture models for unsupervised image denoising
  ({HDMI}).
\newblock {\em SIAM Journal on Imaging Sciences}, 11(4):2815--2846, 2018.

\bibitem{HLP22}
S.~Hurault, A.~Leclaire, and N.~Papadakis.
\newblock Gradient step denoiser for convergent plug-and-play.
\newblock In {\em International Conference on Learning Representations}, 2022.

\bibitem{JKYB2020}
P.~Jaini, I.~Kobyzev, Y.~Yu, and M.~Brubaker.
\newblock Tails of lipschitz triangular flows.
\newblock In {\em International Conference on Machine Learning}, pages
  4673--4681. PMLR, 2020.

\bibitem{JMFU17}
K.~H. Jin, M.~T. McCann, E.~Froustey, and M.~Unser.
\newblock Deep convolutional neural network for inverse problems in imaging.
\newblock {\em IEEE Transactions on Image Processing}, 26(9):4509--4522, 2017.

\bibitem{Jung2021}
J.~Jung, J.~Na, H.~K. Park, J.~M. Park, G.~Kim, S.~Lee, and H.~S. Kim.
\newblock Super-resolving material microstructure image via deep learning for
  microstructure characterization and mechanical behavior analysis.
\newblock {\em npj Computational Materials}, 7(1):96, 2021.

\bibitem{kawar2022denoising}
B.~Kawar, M.~Elad, S.~Ermon, and J.~Song.
\newblock Denoising diffusion restoration models.
\newblock In {\em ICLR Workshop on Deep Generative Models for Highly Structured
  Data}, 2022.

\bibitem{kawar2021snips}
B.~Kawar, G.~Vaksman, and M.~Elad.
\newblock {SNIPS}: Solving noisy inverse problems stochastically.
\newblock In A.~Beygelzimer, Y.~Dauphin, P.~Liang, and J.~W. Vaughan, editors,
  {\em Advances in Neural Information Processing Systems}, 2021.

\bibitem{K1981}
R.~Keys.
\newblock Cubic convolution interpolation for digital image processing.
\newblock {\em IEEE Transactions on Acoustics, Speech, and Signal Processing},
  29(6):1153--1160, 1981.

\bibitem{KB2015}
D.~P. Kingma and J.~Ba.
\newblock Adam: {A} method for stochastic optimization.
\newblock In {\em International Conference on Learning Representations}, 2015.

\bibitem{KD2018}
D.~P. Kingma and P.~Dhariwal.
\newblock Glow: Generative flow with invertible 1x1 convolutions.
\newblock {\em Advances in neural information processing systems}, 31, 2018.

\bibitem{kingmaautoenc}
D.~P. Kingma and M.~Welling.
\newblock Auto-encoding variational bayes.
\newblock In Y.~Bengio and Y.~LeCun, editors, {\em 2nd International Conference
  on Learning Representations, {ICLR} 2014, Banff, AB, Canada, April 14-16,
  2014, Conference Track Proceedings}, 2014.

\bibitem{outofdist}
P.~Kirichenko, P.~Izmailov, and A.~G. Wilson.
\newblock Why normalizing flows fail to detect out-of-distribution data.
\newblock In H.~Larochelle, M.~Ranzato, R.~Hadsell, M.~Balcan, and H.~Lin,
  editors, {\em Advances in Neural Information Processing Systems}, volume~33,
  page 20578–20589. Curran Associates, Inc., 2020.

\bibitem{KEKP2020}
E.~Kobler, A.~Effland, K.~Kunisch, and T.~Pock.
\newblock Total deep variation for linear inverse problems.
\newblock In {\em Proceedings of the IEEE/CVF Conference on Computer Vision and
  Pattern Recognition}, pages 7549--7558, 2020.

\bibitem{KEKP2021}
E.~Kobler, A.~Effland, K.~Kunisch, and T.~Pock.
\newblock Total deep variation: A stable regularization method for inverse
  problems.
\newblock {\em IEEE Transactions on Pattern Analysis and Machine Intelligence},
  2021.

\bibitem{Kohli2017}
M.~D. Kohli, R.~M. Summers, and J.~R. Geis.
\newblock Medical image data and datasets in the era of machine
  learning—whitepaper from the 2016 {C-MIMI} meeting dataset session.
\newblock {\em Journal of Digital Imaging}, 30(4):392–399, 2017.

\bibitem{LNPS2017}
F.~Laus, M.~Nikolova, J.~Persch, and G.~Steidl.
\newblock A nonlocal denoising algorithm for manifold-valued images using
  second order statistics.
\newblock {\em SIAM Journal on Imaging Sciences}, 10(1):416–448, 2017.

\bibitem{LBM2013}
M.~Lebrun, A.~Buades, and J.-M. Morel.
\newblock A nonlocal {B}ayesian image denoising algorithm.
\newblock {\em SIAM Journal on Imaging Sciences}, 6(3):1665--1688, 2013.

\bibitem{LoDoPaB21}
J.~Leuschner, M.~Schmidt, D.~O. Baguer, and P.~Maass.
\newblock {LoDoPaB-CT}, a benchmark dataset for low-dose computed tomography
  reconstruction.
\newblock {\em Scientific Data}, 8(109), 2021.

\bibitem{Leuschner21}
J.~Leuschner, M.~Schmidt, P.~S. Ganguly, V.~Andriiashen, S.~B. Coban,
  A.~Denker, D.~Bauer, A.~Hadjifaradji, K.~J. Batenburg, P.~Maass, and M.~van
  Eijnatten.
\newblock Quantitative comparison of deep learning-based image reconstruction
  methods for low-dose and sparse-angle {CT} applications.
\newblock {\em Journal of Imaging}, 7(3), 2021.

\bibitem{Lugmayr20}
A.~Lugmayr, M.~Danelljan, L.~Van~Gool, and R.~Timofte.
\newblock {SRFlow}: Learning the super-resolution space with normalizing flow.
\newblock In {\em ECCV}, 2020.

\bibitem{lunz2018adversarial}
S.~Lunz, O.~{\"O}ktem, and C.-B. Sch{\"o}nlieb.
\newblock Adversarial regularizers in inverse problems.
\newblock {\em Advances in neural information processing systems}, 31, 2018.

\bibitem{martin2001database}
D.~Martin, C.~Fowlkes, D.~Tal, and J.~Malik.
\newblock A database of human segmented natural images and its application to
  evaluating segmentation algorithms and measuring ecological statistics.
\newblock In {\em Proceedings Eighth IEEE International Conference on Computer
  Vision. ICCV 2001}, volume~2, pages 416--423. IEEE, 2001.

\bibitem{MO2014}
M.~Mirza and S.~Osindero.
\newblock Conditional generative adversarial nets.
\newblock {\em arXiv preprint arXiv:1411.1784}, 2014.

\bibitem{MCOS2021}
S.~Mukherjee, M.~Carioni, O.~{\"O}ktem, and C.-B. Sch{\"o}nlieb.
\newblock End-to-end reconstruction meets data-driven regularization for
  inverse problems.
\newblock {\em Advances in Neural Information Processing Systems},
  34:21413--21425, 2021.

\bibitem{OJMBDW2020}
G.~Ongie, A.~Jalal, C.~A. Metzler, R.~G. Baraniuk, A.~G. Dimakis, and
  R.~Willet.
\newblock Deep learning techniques for inverse problems in imaging.
\newblock {\em arXiv preprint arXiv:2005.06001}, 2020.

\bibitem{pan2020dgp}
X.~Pan, X.~Zhan, B.~Dai, D.~Lin, C.~C. Loy, and P.~Luo.
\newblock Exploiting deep generative prior for versatile image restoration and
  manipulation.
\newblock In {\em European Conference on Computer Vision (ECCV)}, 2020.

\bibitem{PDDN2019}
S.~Parameswaran, C.~Deledalle, L.~Denis, and T.~Q.~Nguyen.
\newblock Accelerating {GMM}-based patch priors for image restoration: Three
  ingredients for a 100x speed-up.
\newblock {\em {IEEE} Transactions on Image Processing}, 28(2):687--698, 2019.

\bibitem{prost2021learning}
J.~Prost, A.~Houdard, A.~Almansa, and N.~Papadakis.
\newblock Learning local regularization for variational image restoration.
\newblock In {\em International Conference on Scale Space and Variational
  Methods in Computer Vision}, pages 358--370. Springer, 2021.

\bibitem{Radon86}
J.~Radon.
\newblock On the determination of functions from their integral values along
  certain manifolds.
\newblock {\em IEEE Transactions on Medical Imaging}, 5(4):170--176, 1986.

\bibitem{buzzard}
E.~J. Reid, L.~F. Drummy, C.~A. Bouman, and G.~T. Buzzard.
\newblock Multi-resolution data fusion for super resolution imaging.
\newblock {\em IEEE Transactions on Computational Imaging}, 8:81--95, 2022.

\bibitem{varinfNF}
D.~Rezende and S.~Mohamed.
\newblock Variational inference with normalizing flows.
\newblock In {\em International conference on machine learning}, pages
  1530--1538. PMLR, 2015.

\bibitem{REM2017}
Y.~Romano, M.~Elad, and P.~Milanfar.
\newblock The little engine that could: Regularization by denoising ({RED}).
\newblock {\em SIAM Journal on Imaging Sciences}, 10(4):1804--1844, 2017.

\bibitem{RFB15a}
O.~Ronneberger, P.~Fischer, and T.~Brox.
\newblock {U-Net}: Convolutional networks for biomedical image segmentation.
\newblock In {\em Medical Image Computing and Computer-Assisted Intervention
  (MICCAI)}, volume 9351 of {\em LNCS}, pages 234--241. Springer, 2015.

\bibitem{ROF1992}
L.~I. Rudin, S.~Osher, and E.~Fatemi.
\newblock Nonlinear total variation based noise removal algorithms.
\newblock {\em Physica D: Nonlinear Phenomena}, 60(1):259--268, 1992.

\bibitem{RH2021}
L.~Ruthotto and E.~Haber.
\newblock An introduction to deep generative modeling.
\newblock {\em DMV Mitteilungen}, 44(3):1--24, 2021.

\bibitem{SJ2016}
P.~Sandeep and T.~Jacob.
\newblock Single image super-resolution using a joint {GMM} method.
\newblock {\em IEEE Transactions on Image Processing}, 25(9):4233--4244, 2016.

\bibitem{STA2021}
H.~Shi, Y.~Traonmilin, and J.-F. Aujol.
\newblock Compressive learning for patch-based image denoising.
\newblock {\em HAL preprint hal-03429102}, 2021.

\bibitem{SCI2018}
A.~Shocher, N.~Cohen, and M.~Irani.
\newblock “{Z}ero-shot” super-resolution using deep internal learning.
\newblock In {\em Proceedings of the IEEE conference on computer vision and
  pattern recognition}, pages 3118--3126, 2018.

\bibitem{Soh_2020_CVPR}
J.~W. Soh, S.~Cho, and N.~I. Cho.
\newblock Meta-transfer learning for zero-shot super-resolution.
\newblock In {\em Proceedings of the IEEE/CVF Conference on Computer Vision and
  Pattern Recognition (CVPR)}, June 2020.

\bibitem{SLY2015}
K.~Sohn, H.~Lee, and X.~Yan.
\newblock Learning structured output representation using deep conditional
  generative models.
\newblock {\em Advances in Neural Information Processing Systems}, 28, 2015.

\bibitem{songermongrad}
Y.~Song and S.~Ermon.
\newblock Generative modeling by estimating gradients of the data distribution.
\newblock In H.~Wallach, H.~Larochelle, A.~Beygelzimer, F.~d\textquotesingle
  Alché-Buc, E.~Fox, and R.~Garnett, editors, {\em Advances in Neural
  Information Processing Systems}, volume~32. Curran Associates, Inc., 2019.

\bibitem{song2022solving}
Y.~Song, L.~Shen, L.~Xing, and S.~Ermon.
\newblock Solving inverse problems in medical imaging with score-based
  generative models.
\newblock In {\em The 10th International Conference on Learning
  Representations}, 2022.

\bibitem{SVWB2016}
S.~Sreehari, S.~V. Venkatakrishnan, B.~Wohlberg, G.~T. Buzzard, L.~F. Drummy,
  J.~P. Simmons, and C.~A. Bouman.
\newblock Plug-and-play priors for bright field electron tomography and sparse
  interpolation.
\newblock {\em IEEE Transactions on Computational Imaging}, 2(4):408--423,
  2016.

\bibitem{TITOIS2020}
T.~Teshima, I.~Ishikawa, K.~Tojo, K.~Oono, M.~Ikeda, and M.~Sugiyama.
\newblock Coupling-based invertible neural networks are universal
  diffeomorphism approximators.
\newblock In {\em Proceedings of the 34th International Conference on Neural
  Information Processing Systems}, Vancouver, BC, Canada, 2020. NIPS'20.

\bibitem{TXZLZ21}
C.~Tian, Y.~Xu, W.~Zuo, C.-W. Lin, and D.~Zhang.
\newblock Asymmetric {CNN} for image superresolution.
\newblock {\em IEEE Transactions on Systems, Man, and Cybernetics: Systems},
  2021.

\bibitem{UVL2018}
D.~Ulyanov, A.~Vedaldi, and V.~Lempitsky.
\newblock Deep image prior.
\newblock In {\em Proceedings of the IEEE Conference on Computer Vision and
  Pattern Recognition}, pages 9446--9454, 2018.

\bibitem{vaucher2007line}
S.~Vaucher, P.~Unifantowicz, C.~Ricard, L.~Dubois, M.~Kuball, J.-M.
  Catala-Civera, D.~Bernard, M.~Stampanoni, and R.~Nicula.
\newblock On-line tools for microscopic and macroscopic monitoring of microwave
  processing.
\newblock {\em Physica B: Condensed Matter}, 398(2):191--195, 2007.

\bibitem{VBW2013}
S.~V. Venkatakrishnan, C.~A. Bouman, and B.~Wohlberg.
\newblock Plug-and-play priors for model based reconstruction.
\newblock In {\em 2013 IEEE Global Conference on Signal and Information
  Processing}, pages 945--948. IEEE, 2013.

\bibitem{Wei2022DeepUW}
X.~Wei, H.~V. Gorp, L.~G. Carabarin, D.~Freedman, Y.~C. Eldar, and R.~van
  Sloun.
\newblock Deep unfolding with normalizing flow priors for inverse problems.
\newblock {\em IEEE Transactions on Signal Processing}, 70:2962--2971, 2022.

\bibitem{WLD2021}
J.~Whang, Q.~Lei, and A.~Dimakis.
\newblock Solving inverse problems with a flow-based noise model.
\newblock In M.~Meila and T.~Zhang, editors, {\em Proceedings of the 38th
  International Conference on Machine Learning}, volume 139 of {\em Proceedings
  of Machine Learning Research}, pages 11146--11157. PMLR, 18--24 Jul 2021.

\bibitem{xu2017enhanced}
C.~S. Xu, K.~J. Hayworth, Z.~Lu, P.~Grob, A.~M. Hassan, J.~G.
  Garc{\'\i}a-Cerd{\'a}n, K.~K. Niyogi, E.~Nogales, R.~J. Weinberg, and H.~F.
  Hess.
\newblock Enhanced {FIB-SEM} systems for large-volume {3D} imaging.
\newblock {\em eLife}, 6, 2017.

\bibitem{ZLZZ2021}
K.~Zhang, Y.~Li, W.~Zuo, L.~Zhang, L.~Van~Gool, and R.~Timofte.
\newblock Plug-and-play image restoration with deep denoiser prior.
\newblock {\em IEEE Transactions on Pattern Analysis and Machine Intelligence},
  2021.

\bibitem{ct_zero}
Z.~Zhang, S.~Yu, W.~Qin, X.~Liang, Y.~Xie, and G.~Cao.
\newblock Ct super resolution via zero shot learning.
\newblock {\em arXiv preprint arXiv:2012.08943}, 2020.

\bibitem{ZW2011}
D.~Zoran and Y.~Weiss.
\newblock From learning models of natural image patches to whole image
  restoration.
\newblock In {\em IEEE International Conference on Computer Vision}, pages
  479--486. IEEE, 2011.

\end{thebibliography}

\clearpage

\appendix

\section{Comparison Methods} 
\label{app_quality_measures_comparisons}

\begin{itemize}
\item \textbf{Bicubic interpolation.}
For superresolution, the simplest comparison is the bicubic interpolation \cite{K1981}, which is based on the local approximation of the image by polynomials of degree 3.

\item \textbf{Filtered Backprojection and UNet.}
For CT a classical method is the Filtered Backprojection (FBP), described by the adjoint Radon transform  \cite{Radon86}. We used the ODL implementation \cite{Adler2018} for our experiments. There we choose the filter type Hann and a frequency scaling of $0.641$. In order to improve the image quality of the FBP, a post-processing network can be learned. Here we consider the popular choice of a UNet (FBP+UNet) \cite{RFB15a}, which was used in \cite{JMFU17} for CT imaging. We use the implementation from \cite{Leuschner21}\footnote{available at \url{https://jleuschn.github.io/docs.dival/dival.reconstructors.fbpunet_reconstructor.html}}, which is trained on the 35820 training images of LoDoPaB dataset.

\item \textbf{Wasserstein Patch Prior.}
The idea of the Wasserstein Patch Prior (WPP) \cite{AH2022,Hertrich21}\footnote{We use the original implementation from \cite{Hertrich21} available at \\ \url{https://github.com/johertrich/Wasserstein_Patch_Prior}} is to use the Wasserstein-2 distance between the patch distribution of the reconstruction and the patch distribution of a given reference image. Here a high resolution reference image $\tilde{x}$ (or a high-resolution cutout) with a similar patch distribution as the unknown high resolution image $x$ is needed.
For a representation of structures of different sizes, $x$ and $\tilde{x}$ are considered at different scales $x_1 = x, \tilde{x}_1 = \tilde{x}, x_l = Ax_{l-1}, \tilde{x}_l = A\tilde{x}_{l-1}$ for a downsampling operator $A$.

Then the aim is to minimize the functional 
\begin{align}
\mathcal{J}(x) = \mathcal{D}(f(x),y) + \lambda \sum_{l=1}^L W_2^2 (\mu_{x_l}, \mu_{\tilde{x}_l}),
\end{align}
where the patch distributions of $x$ and $\tilde{x}$ are defined by
\begin{align}
\mu_{x_l} = \frac{1}{N_l} \sum_{i=1}^{N_l} \delta_{P_i x_l}, \mu_{\tilde{x}_l} = \frac{1}{\tilde{N}_l} \sum_{i=1}^{\tilde{N}_l} \delta_{P_i \tilde{x}_l}.
\end{align}

\item \textbf{Deep Image Prior with TV regularization.}
The idea of the Deep Image Prior (DIP) \cite{UVL2018} is to solve the optimization problem
$$
\hat \theta\in\argmin_{\theta} \mathcal{D}(f(G_\theta(z)),y),
$$
where $G_\theta$ is a convolutional neural network with parameters $\theta$ and $z$ is a randomly chosen input. 
Then, the reconstruction $\hat x$ is given by $\hat x=G_\theta(z)$.
It was shown in \cite{UVL2018} that DIP admits competitive results for many inverse problems. A combination of DIP with the TV (DIP+TV) regualizer was successfully used in \cite{baguer2020computed}\footnote{For superresolution, we use the original implementation from \cite{UVL2018} available at\\ \url{https://github.com/DmitryUlyanov/deep-image-prior} in combination with a TV regulariser; for CT, we use the original implementation from \cite{baguer2020computed} available at\\ \url{https://github.com/jleuschn/dival/blob/master/dival/reconstructors}} for CT reconstruction. 
Here the optimization problem is extended to 
$$
\hat \theta\in\argmin_{\theta} \mathcal{D}(f(G_\theta(z)),y) + \text{TV}((G_\theta(z)),
$$
Note that each reconstruction with the DIP+TV requires the training of a neural network. In contrast to WPP and patchNR, the DIP+TV is a data-free method, i.e., it does not require any clean image for training. We tested a pre-trained variant of the DIP+TV on the same training images. However, this did not improve the results significantly. Note that warm-start intialization techniques were proposed in \cite{barbano2021deep}. Here, the authors observed faster reconstruction times for a pre-trained DIP but not significantly better results. Therefore we stick to the random initialization.

\item \textbf{Plug-and-Play Forward Backward Splitting with DRUNet.} In Plug-and-Play methods, first introduced by \cite{VBW2013}, the main idea is to consider an optimization algorithm from convex analysis for solving \eqref{eq_variational} and to replace the proximal operator with respect to the regularizer by a more general denoiser. Here, modify the forward backward splitting algorithm
\begin{align*}
x_{n+1} = \text{prox}_{\eta R}(x_n - \eta \nabla_x \mathcal{D}(f(x_n),y))
\end{align*}
for minimizing the functional \eqref{eq_variational} by the iteration
\begin{align} \label{equ_pnp}
x_{n+1} = \mathcal{G}(x_n - \eta \nabla_x \mathcal{D}(f(x_n),y)),
\end{align}
where $\mathcal{G}$ is a neural network trained for denoising natural images. We use the DRUNet (DPIR) from \cite{ZLZZ2021} as denoiser and run \eqref{equ_pnp} for 100 iterations. Note that the denoiser is trained on natural images and not on images from the specific image domain.
However, as we have given only very few clean images from the considered image domain it is impossible to train a denoiser with comparable quality on them.

\item \textbf{Local Adversarial Regularizer.}
The adversarial regularizer was introduced in \cite{lunz2018adversarial} and this framework was recently applied for learning patch-based regularizers (localAR) \cite{prost2021learning}. The idea is to train a network $r_\theta$ as a critic between patch distributions in order to distinguish between clean and degraded patches. The network is trained by minimizing
\begin{align}
D(\theta) = \mathbb{E}_{z \sim \mathbb{P}_c}\big[r_\theta (z)\big] - \mathbb{E}_{z \sim \mathbb{P}_n}\big[r_\theta (z)\big] + \mu \mathbb{E}_{z \sim \mathbb{P}_i} \big[(\Vert \nabla_z r_\theta (z) \Vert_2 -1)^2 \big],
\end{align}
where $\mathbb{P}_c$ and $\mathbb{P}_n$ are the distributions of clean and noisy patches, respectively, and $\mathbb{P}_i$ is the distribution of all lines connecting samples in $\mathbb{P}_c$ and $\mathbb{P}_n$. Then the aim is to minimize the functional 
\begin{align}
\mathcal{J}(x) = \mathcal{D}(f(x),y) + \lambda \frac{1}{\vert I \vert} \sum_{i \in \mathcal{I}} r_\theta(P_i(x))
\end{align}
For our experiments we used the code of \cite{prost2021learning}, but instead of patch size 15 we used the patch size 6 and replaced the fully convolutional discriminator by a discriminator with 2 convolutional layers, followed by 4 fully connected layers.

\item \textbf{Expected Patch Log Likelihood.}
The Expected Patch Log Likelihood (EPLL) prior \cite{ZW2011} assumes that the patch distribution of the ground truth can be approximated by a GMM $p$ fitted to the patch distribution of the reference images. Reconstruction is done by minimizing the functional
\begin{align*}
\mathcal{J}(x) = \mathcal{D}(f(x),y) - \lambda \sum_{i=1}^N p(P_i (x)).
\end{align*}
In \cite{ZW2011} the authors used half quadratic splitting to optimize this objective function. For our experiments we implemented the GMM in PyTorch and used the Adam \cite{KB2015} optimizer. This is because we are not aware how to efficiently implement the half quadratic splitting for the CT forward operator.

\item \textbf{Asymmetric CNN.} The asymmetric CNN (ACNN) \cite{TXZLZ21} is a 23-layer CNN trained in a data-based way on 28 paired images pairs of the composite "SiC Diamonds" using the $L^2$ loss function.

\item \textbf{Zero Shot Super-Resolution.} For Zero Shot Super-Resolution (ZSSR) \cite{SCI2018} the main assumption is that the patch distribution of natural images is self-similar across the scales. Exploiting this fact, a lightweight CNN is trained in a supervised fashion on a paired dataset generated by the low-resolution image itself. This dataset is created by downsampling the low-resolution image to obtain a lower-resolution image and is enlarged by data augmentation like random rotations, random crops or mirror reflections. A high-resolution prediction is then created by applying the trained model to the low-resolution observation. For our experiments we reimplemented the ZSSR.

\item \textbf{DualSR.} The idea of DualSR \cite{EPC2021}\footnote{We use the original implementation available at \url{https://github.com/memad73/DualSR}} is a dual-path pipeline, where an upsampling GAN learns the upsampling process and a downsampling GAN learns the degradation model, trained on cropped parts of the low-resolution image. This method can be used for blind superresolution, but since we know the forward operator in our case, we replace the downsampling GAN by the given degradation process.

\end{itemize}

\section{Further examples}\label{app:furhter_examples}
Here we give some more examples of our experiments from Section \ref{Sec_experiments}; see Figure \ref{fig:SIC_further}, Figure \ref{app:Fig_CT_img} and Figure \ref{app:Fig_CT_img_further}. 

\begin{figure*}[t]
\centering
\begin{subfigure}[t]{.15\textwidth}
  \includegraphics[width=\linewidth]{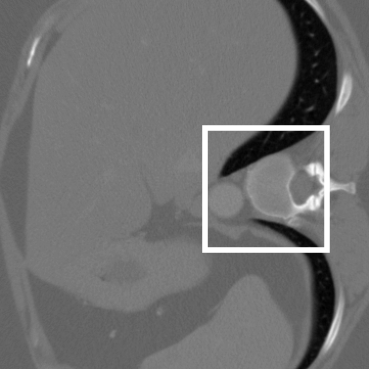}
\end{subfigure}%
\hfill
\begin{subfigure}[t]{.15\textwidth}
  \includegraphics[width=\linewidth]{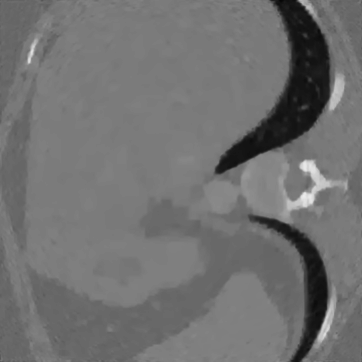}
\end{subfigure}%
\hfill
\begin{subfigure}[t]{.15\textwidth}
  \includegraphics[width=\linewidth]{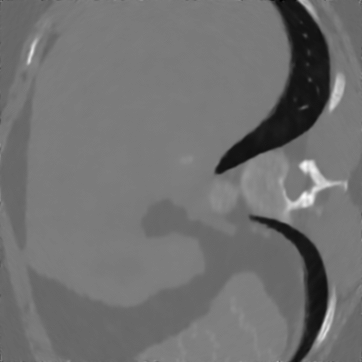}
\end{subfigure}%
\hfill
\begin{subfigure}[t]{.15\textwidth}
  \includegraphics[width=\linewidth]{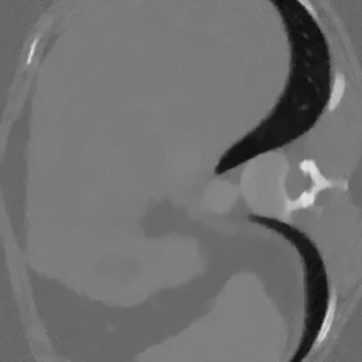}
\end{subfigure}%
\hfill
\begin{subfigure}[t]{.15\textwidth}
  \includegraphics[width=\linewidth]{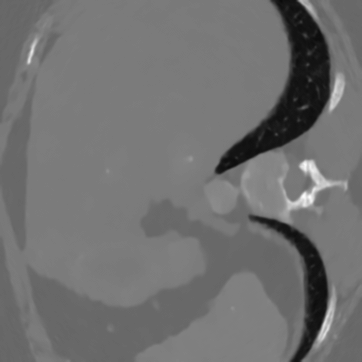}
\end{subfigure}%
\hfill
\begin{subfigure}[t]{.15\textwidth}
  \includegraphics[width=\linewidth]{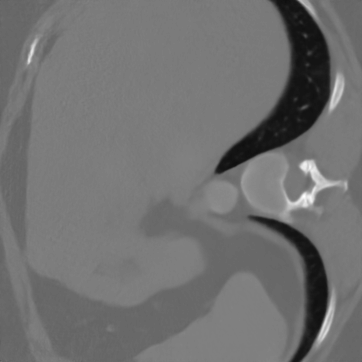}
\end{subfigure}%

\begin{subfigure}[t]{.15\textwidth}
  \includegraphics[width=\linewidth]{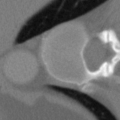}
\end{subfigure}%
\hfill
\begin{subfigure}[t]{.15\textwidth}
  \includegraphics[width=\linewidth]{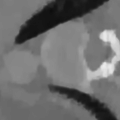}
\end{subfigure}%
\hfill
\begin{subfigure}[t]{.15\textwidth}
  \includegraphics[width=\linewidth]{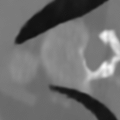}
\end{subfigure}%
\hfill
\begin{subfigure}[t]{.15\textwidth}
  \includegraphics[width=\linewidth]{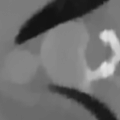}
\end{subfigure}%
\hfill
\begin{subfigure}[t]{.15\textwidth}
  \includegraphics[width=\linewidth]{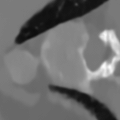}
\end{subfigure}%
\hfill
\begin{subfigure}[t]{.15\textwidth}
  \includegraphics[width=\linewidth]{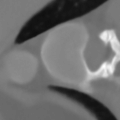}
\end{subfigure}%
\vspace{.5cm}

\begin{subfigure}[t]{.15\textwidth}
  \includegraphics[width=\linewidth]{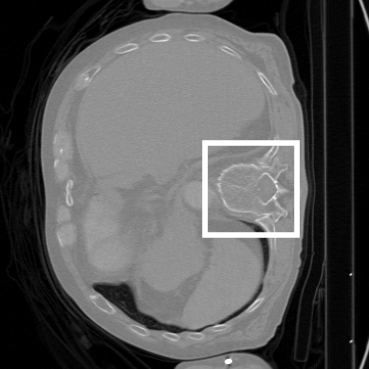}
\end{subfigure}%
\hfill
\begin{subfigure}[t]{.15\textwidth}
  \includegraphics[width=\linewidth]{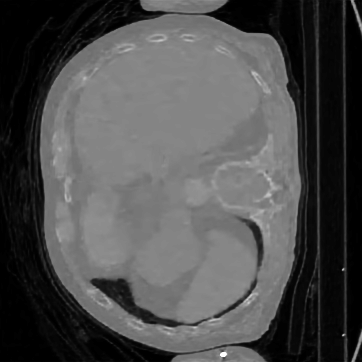}
\end{subfigure}%
\hfill
\begin{subfigure}[t]{.15\textwidth}
  \includegraphics[width=\linewidth]{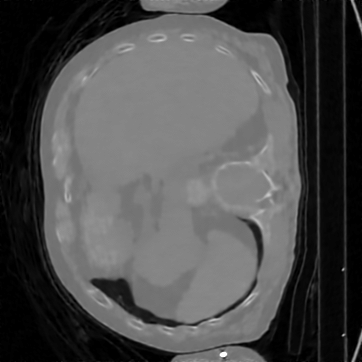}
\end{subfigure}%
\hfill
\begin{subfigure}[t]{.15\textwidth}
  \includegraphics[width=\linewidth]{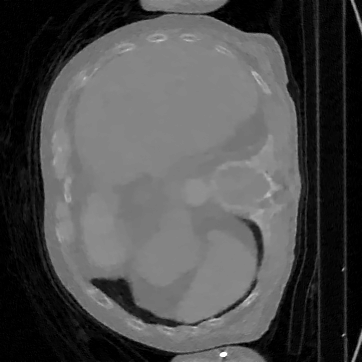}
\end{subfigure}%
\hfill
\begin{subfigure}[t]{.15\textwidth}
  \includegraphics[width=\linewidth]{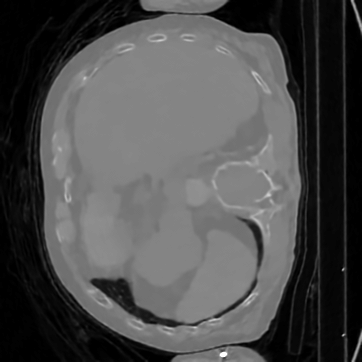}
\end{subfigure}%
\hfill
\begin{subfigure}[t]{.15\textwidth}
  \includegraphics[width=\linewidth]{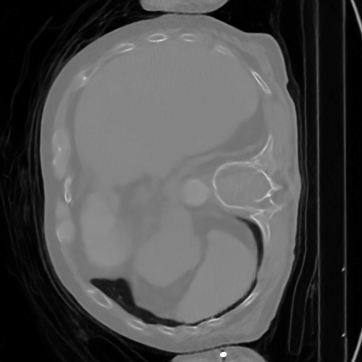}
\end{subfigure}%

\begin{subfigure}[t]{.15\textwidth}
  \includegraphics[width=\linewidth]{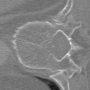}
\end{subfigure}%
\hfill
\begin{subfigure}[t]{.15\textwidth}
  \includegraphics[width=\linewidth]{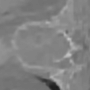}
\end{subfigure}%
\hfill
\begin{subfigure}[t]{.15\textwidth}
  \includegraphics[width=\linewidth]{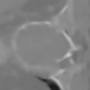}
\end{subfigure}%
\hfill
\begin{subfigure}[t]{.15\textwidth}
  \includegraphics[width=\linewidth]{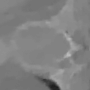}
\end{subfigure}%
\hfill
\begin{subfigure}[t]{.15\textwidth}
  \includegraphics[width=\linewidth]{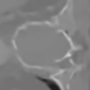}
\end{subfigure}%
\hfill
\begin{subfigure}[t]{.15\textwidth}
  \includegraphics[width=\linewidth]{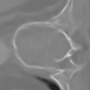}
\end{subfigure}%
\vspace{.5cm}

\begin{subfigure}[t]{.15\textwidth}
  \includegraphics[width=\linewidth]{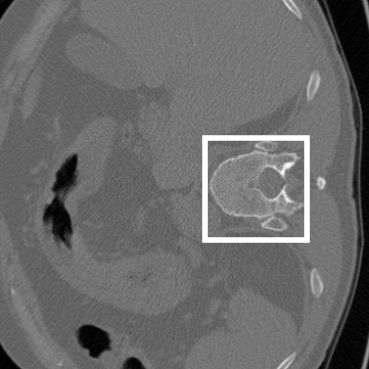}
\end{subfigure}%
\hfill
\begin{subfigure}[t]{.15\textwidth}
  \includegraphics[width=\linewidth]{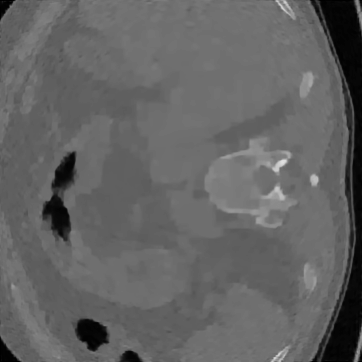}
\end{subfigure}%
\hfill
\begin{subfigure}[t]{.15\textwidth}
  \includegraphics[width=\linewidth]{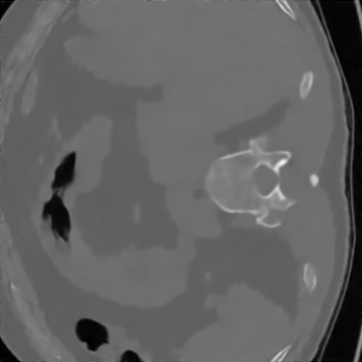}
\end{subfigure}%
\hfill
\begin{subfigure}[t]{.15\textwidth}
  \includegraphics[width=\linewidth]{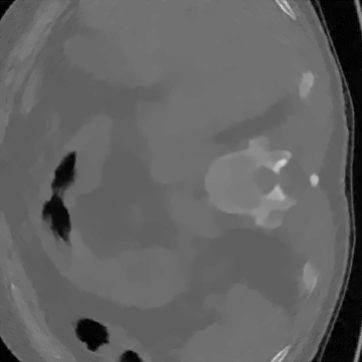}
\end{subfigure}%
\hfill
\begin{subfigure}[t]{.15\textwidth}
  \includegraphics[width=\linewidth]{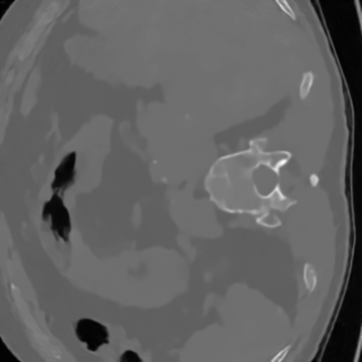}
\end{subfigure}%
\hfill
\begin{subfigure}[t]{.15\textwidth}
  \includegraphics[width=\linewidth]{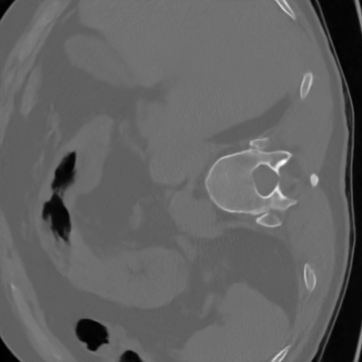}
\end{subfigure}%

\begin{subfigure}[t]{.15\textwidth}
  \includegraphics[width=\linewidth]{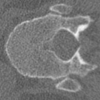}
\end{subfigure}%
\hfill
\begin{subfigure}[t]{.15\textwidth}
  \includegraphics[width=\linewidth]{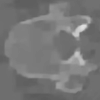}
\end{subfigure}%
\hfill
\begin{subfigure}[t]{.15\textwidth}
  \includegraphics[width=\linewidth]{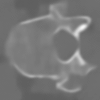}
\end{subfigure}%
\hfill
\begin{subfigure}[t]{.15\textwidth}
  \includegraphics[width=\linewidth]{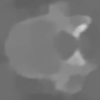}
\end{subfigure}%
\hfill
\begin{subfigure}[t]{.15\textwidth}
  \includegraphics[width=\linewidth]{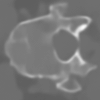}
\end{subfigure}%
\hfill
\begin{subfigure}[t]{.15\textwidth}
  \includegraphics[width=\linewidth]{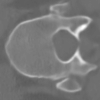}
\end{subfigure}%
\vspace{.5cm}

\begin{subfigure}[t]{.15\textwidth}
  \includegraphics[width=\linewidth]{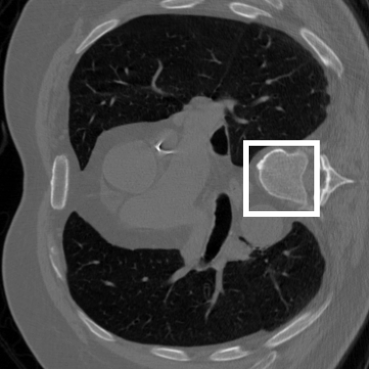}
\end{subfigure}%
\hfill
\begin{subfigure}[t]{.15\textwidth}
  \includegraphics[width=\linewidth]{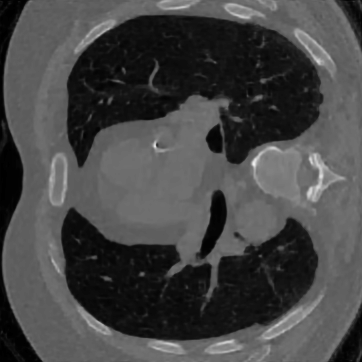}
\end{subfigure}%
\hfill
\begin{subfigure}[t]{.15\textwidth}
  \includegraphics[width=\linewidth]{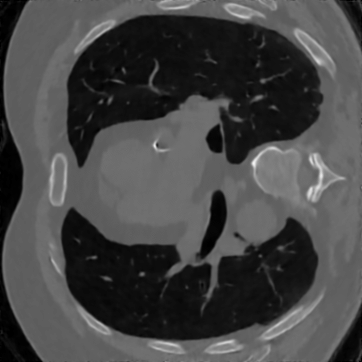}
\end{subfigure}%
\hfill
\begin{subfigure}[t]{.15\textwidth}
  \includegraphics[width=\linewidth]{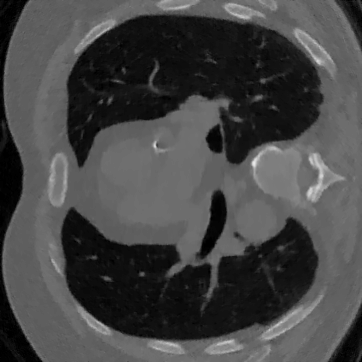}
\end{subfigure}%
\hfill
\begin{subfigure}[t]{.15\textwidth}
  \includegraphics[width=\linewidth]{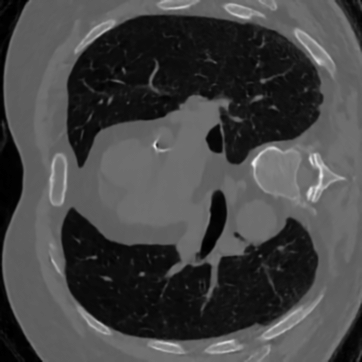}
\end{subfigure}%
\hfill
\begin{subfigure}[t]{.15\textwidth}
  \includegraphics[width=\linewidth]{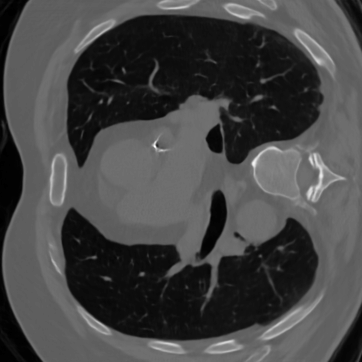}
\end{subfigure}%

\begin{subfigure}[t]{.15\textwidth}
  \includegraphics[width=\linewidth]{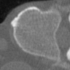}
  \caption*{Ground truth}
\end{subfigure}%
\hfill
\begin{subfigure}[t]{.15\textwidth}
  \includegraphics[width=\linewidth]{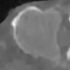}
  \caption*{DIP+TV}  
\end{subfigure}%
\hfill
\begin{subfigure}[t]{.15\textwidth}
  \includegraphics[width=\linewidth]{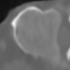}
  \caption*{EPLL}  
\end{subfigure}%
\hfill
\begin{subfigure}[t]{.15\textwidth}
  \includegraphics[width=\linewidth]{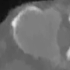}
  \caption*{localAR}  
\end{subfigure}%
\hfill
\begin{subfigure}[t]{.15\textwidth}
  \includegraphics[width=\linewidth]{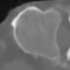}
  \caption*{patchNR}  
\end{subfigure}%
\hfill
\begin{subfigure}[t]{.15\textwidth}
  \includegraphics[width=\linewidth]{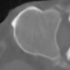}
  \captionsetup{justification=centering}
  \caption*{{FBP+UNet \\(data-based)}}  
\end{subfigure}%
\caption{Full angle reconstruction of the ground truth CT image using different methods. The zoomed-in part is marked with a white box in the ground truth image.
\textit{Top}: full image. \textit{Bottom}: zoomed-in part.} \label{app:Fig_CT_img}
\end{figure*}

\begin{figure*}
\centering
\begin{subfigure}[t]{.15\textwidth}
  \includegraphics[width=\linewidth]{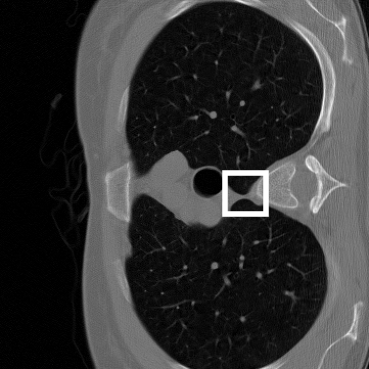}
\end{subfigure}%
\hfill
\begin{subfigure}[t]{.15\textwidth}
  \includegraphics[width=\linewidth]{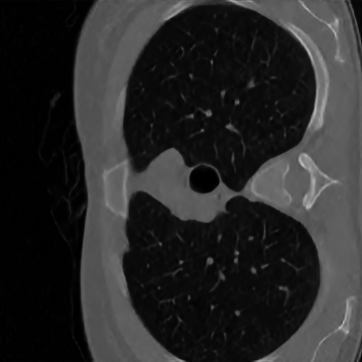}
\end{subfigure}%
\hfill
\begin{subfigure}[t]{.15\textwidth}
  \includegraphics[width=\linewidth]{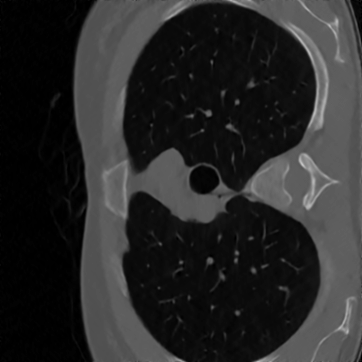}
\end{subfigure}%
\hfill
\begin{subfigure}[t]{.15\textwidth}
  \includegraphics[width=\linewidth]{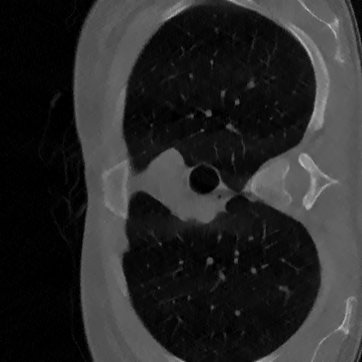}
\end{subfigure}%
\hfill
\begin{subfigure}[t]{.15\textwidth}
  \includegraphics[width=\linewidth]{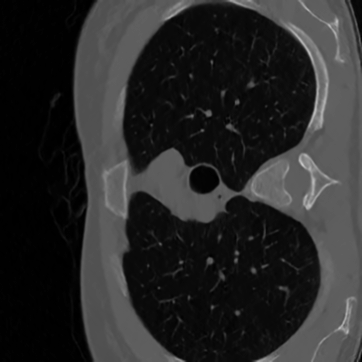}
\end{subfigure}%
\hfill
\begin{subfigure}[t]{.15\textwidth}
  \includegraphics[width=\linewidth]{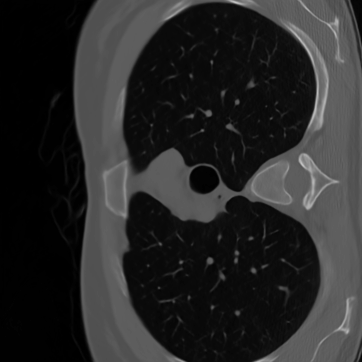}
\end{subfigure}%

\begin{subfigure}[t]{.15\textwidth}
  \includegraphics[width=\linewidth]{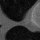}
\end{subfigure}%
\hfill
\begin{subfigure}[t]{.15\textwidth}
  \includegraphics[width=\linewidth]{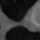}
\end{subfigure}%
\hfill
\begin{subfigure}[t]{.15\textwidth}
  \includegraphics[width=\linewidth]{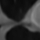}
\end{subfigure}%
\hfill
\begin{subfigure}[t]{.15\textwidth}
  \includegraphics[width=\linewidth]{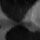}
\end{subfigure}%
\hfill
\begin{subfigure}[t]{.15\textwidth}
  \includegraphics[width=\linewidth]{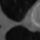}
\end{subfigure}%
\hfill
\begin{subfigure}[t]{.15\textwidth}
  \includegraphics[width=\linewidth]{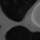}
\end{subfigure}%
\vspace{.5cm}

\begin{subfigure}[t]{.15\textwidth}
  \includegraphics[width=\linewidth]{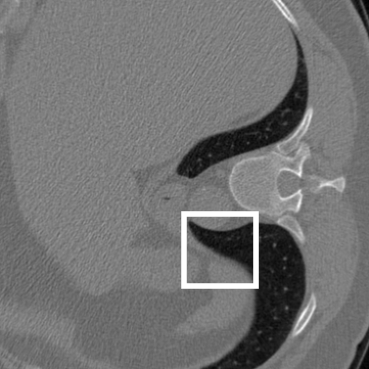}
\end{subfigure}%
\hfill
\begin{subfigure}[t]{.15\textwidth}
  \includegraphics[width=\linewidth]{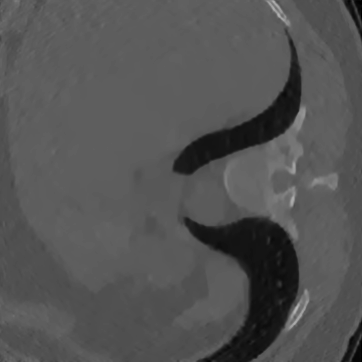}
\end{subfigure}%
\hfill
\begin{subfigure}[t]{.15\textwidth}
  \includegraphics[width=\linewidth]{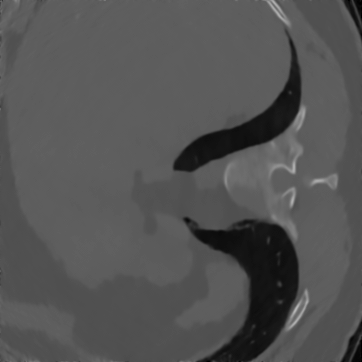}
\end{subfigure}%
\hfill
\begin{subfigure}[t]{.15\textwidth}
  \includegraphics[width=\linewidth]{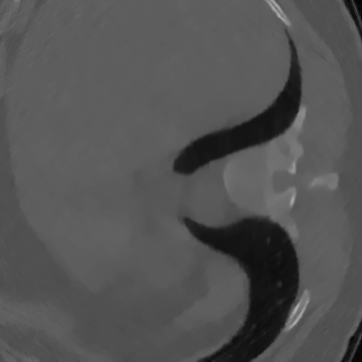}
\end{subfigure}%
\hfill
\begin{subfigure}[t]{.15\textwidth}
  \includegraphics[width=\linewidth]{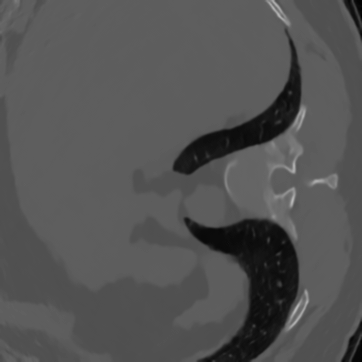}
\end{subfigure}%
\hfill
\begin{subfigure}[t]{.15\textwidth}
  \includegraphics[width=\linewidth]{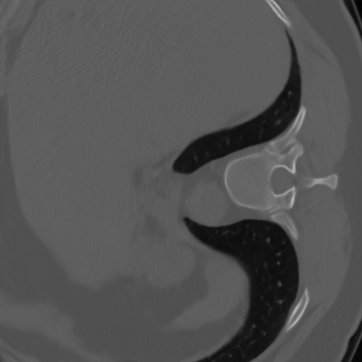}
\end{subfigure}%

\begin{subfigure}[t]{.15\textwidth}
  \includegraphics[width=\linewidth]{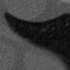}
\end{subfigure}%
\hfill
\begin{subfigure}[t]{.15\textwidth}
  \includegraphics[width=\linewidth]{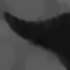} 
\end{subfigure}%
\hfill
\begin{subfigure}[t]{.15\textwidth}
  \includegraphics[width=\linewidth]{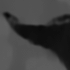}
\end{subfigure}%
\hfill
\begin{subfigure}[t]{.15\textwidth}
  \includegraphics[width=\linewidth]{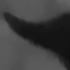}
\end{subfigure}%
\hfill
\begin{subfigure}[t]{.15\textwidth}
  \includegraphics[width=\linewidth]{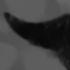}
\end{subfigure}%
\hfill
\begin{subfigure}[t]{.15\textwidth}
  \includegraphics[width=\linewidth]{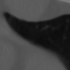}
\end{subfigure}%
\vspace{.5cm}

\begin{subfigure}[t]{.15\textwidth}
  \includegraphics[width=\linewidth]{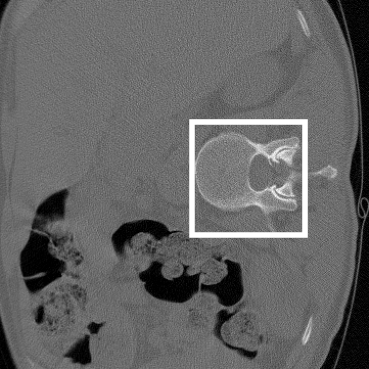}
\end{subfigure}%
\hfill
\begin{subfigure}[t]{.15\textwidth}
  \includegraphics[width=\linewidth]{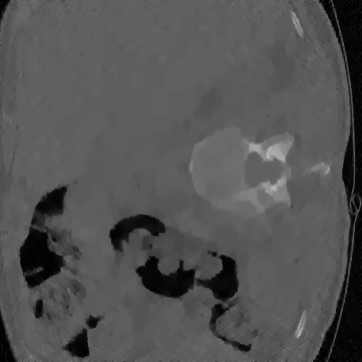}
\end{subfigure}%
\hfill
\begin{subfigure}[t]{.15\textwidth}
  \includegraphics[width=\linewidth]{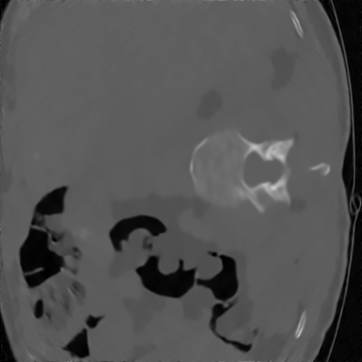}
\end{subfigure}%
\hfill
\begin{subfigure}[t]{.15\textwidth}
  \includegraphics[width=\linewidth]{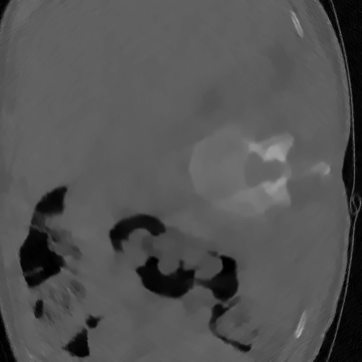}
\end{subfigure}%
\hfill
\begin{subfigure}[t]{.15\textwidth}
  \includegraphics[width=\linewidth]{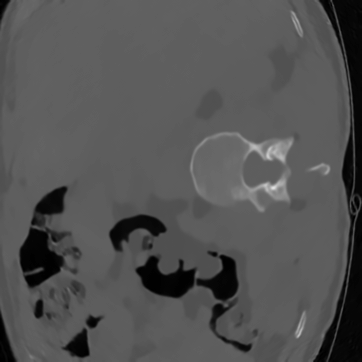}
\end{subfigure}%
\hfill
\begin{subfigure}[t]{.15\textwidth}
  \includegraphics[width=\linewidth]{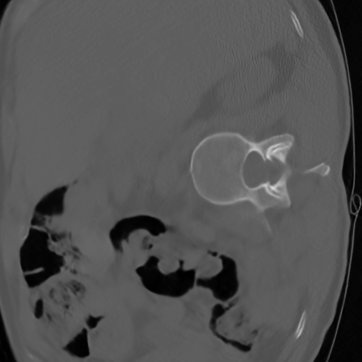}
\end{subfigure}%

\begin{subfigure}[t]{.15\textwidth}
  \includegraphics[width=\linewidth]{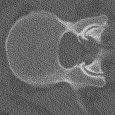}
\end{subfigure}%
\hfill
\begin{subfigure}[t]{.15\textwidth}
  \includegraphics[width=\linewidth]{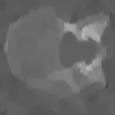}
\end{subfigure}%
\hfill
\begin{subfigure}[t]{.15\textwidth}
  \includegraphics[width=\linewidth]{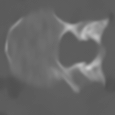}
\end{subfigure}%
\hfill
\begin{subfigure}[t]{.15\textwidth}
  \includegraphics[width=\linewidth]{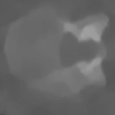}
\end{subfigure}%
\hfill
\begin{subfigure}[t]{.15\textwidth}
  \includegraphics[width=\linewidth]{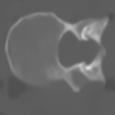}
\end{subfigure}%
\hfill
\begin{subfigure}[t]{.15\textwidth}
  \includegraphics[width=\linewidth]{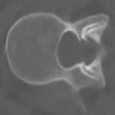}
\end{subfigure}%
\vspace{.5cm}

\begin{subfigure}[t]{.15\textwidth}
  \includegraphics[width=\linewidth]{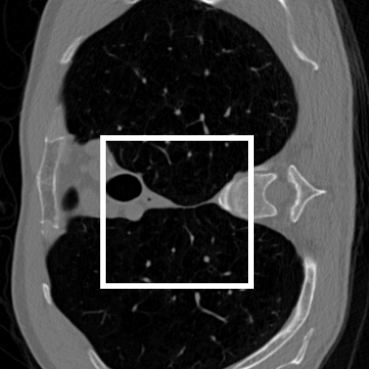}
\end{subfigure}%
\hfill
\begin{subfigure}[t]{.15\textwidth}
  \includegraphics[width=\linewidth]{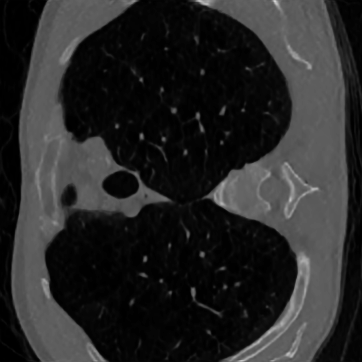}
\end{subfigure}%
\hfill
\begin{subfigure}[t]{.15\textwidth}
  \includegraphics[width=\linewidth]{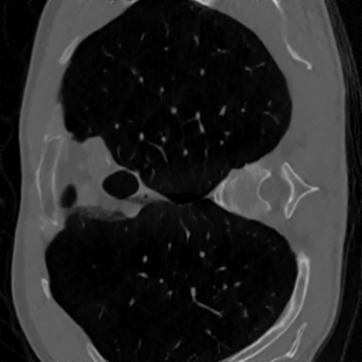}
\end{subfigure}%
\hfill
\begin{subfigure}[t]{.15\textwidth}
  \includegraphics[width=\linewidth]{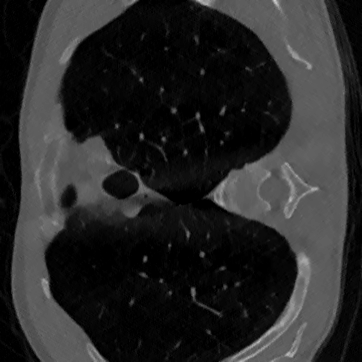}
\end{subfigure}%
\hfill
\begin{subfigure}[t]{.15\textwidth}
  \includegraphics[width=\linewidth]{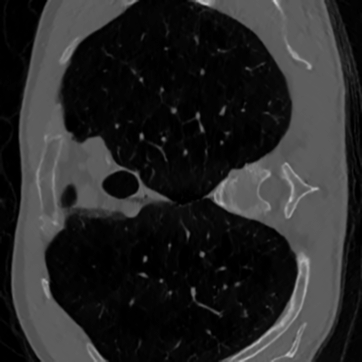}
\end{subfigure}%
\hfill
\begin{subfigure}[t]{.15\textwidth}
  \includegraphics[width=\linewidth]{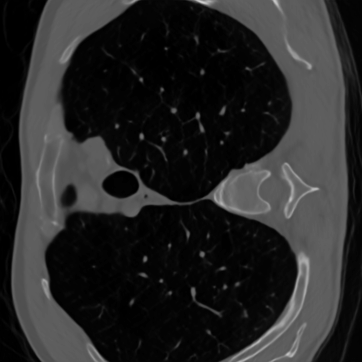}
\end{subfigure}%

\begin{subfigure}[t]{.15\textwidth}
  \includegraphics[width=\linewidth]{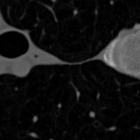}
  \caption*{Ground truth}
\end{subfigure}%
\hfill
\begin{subfigure}[t]{.15\textwidth}
  \includegraphics[width=\linewidth]{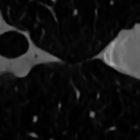}
  \caption*{DIP+TV}  
\end{subfigure}%
\hfill
\begin{subfigure}[t]{.15\textwidth}
  \includegraphics[width=\linewidth]{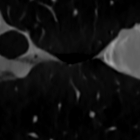}
  \caption*{EPLL}  
\end{subfigure}%
\hfill
\begin{subfigure}[t]{.15\textwidth}
  \includegraphics[width=\linewidth]{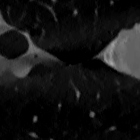}
  \caption*{localAR}  
\end{subfigure}%
\hfill
\begin{subfigure}[t]{.15\textwidth}
  \includegraphics[width=\linewidth]{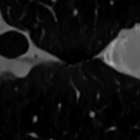}
  \caption*{patchNR}  
\end{subfigure}%
\hfill
\begin{subfigure}[t]{.15\textwidth}
  \includegraphics[width=\linewidth]{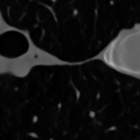}
  \captionsetup{justification=centering}
  \caption*{{FBP+UNet \\(data-based)}}  
\end{subfigure}%
\caption{Limited angle reconstruction of the ground truth CT image using different methods. The zoomed-in part is marked with a white box in the ground truth image.
\textit{Top}: full image. \textit{Bottom}: zoomed-in part.} \label{app:Fig_CT_img_further}
\end{figure*}

\begin{figure*}
\centering
\begin{subfigure}[t]{.14\textwidth}
  \includegraphics[width=\linewidth]{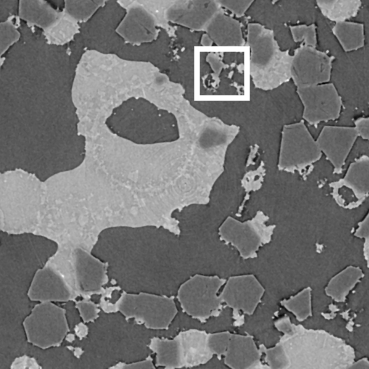}
\end{subfigure}%
\hfill
\begin{subfigure}[t]{.14\textwidth}
  \includegraphics[width=\linewidth]{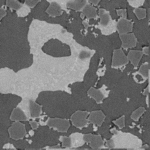}
\end{subfigure}%
\hfill
\begin{subfigure}[t]{.14\textwidth}
  \includegraphics[width=\linewidth]{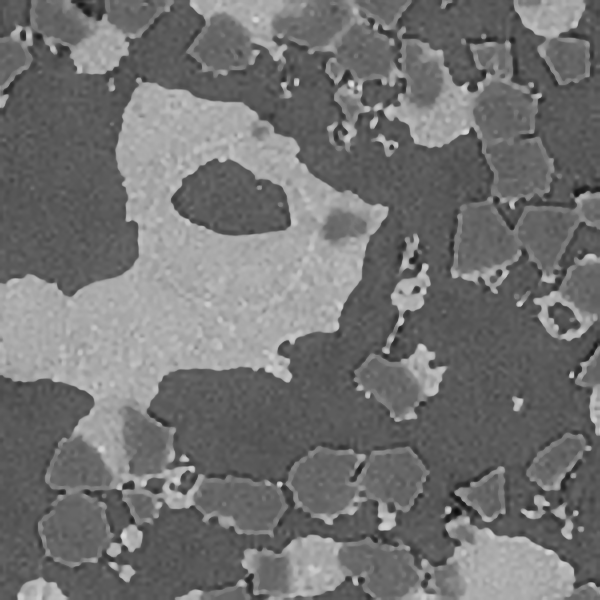}
\end{subfigure}%
\hfill
\begin{subfigure}[t]{.14\textwidth}
  \includegraphics[width=\linewidth]{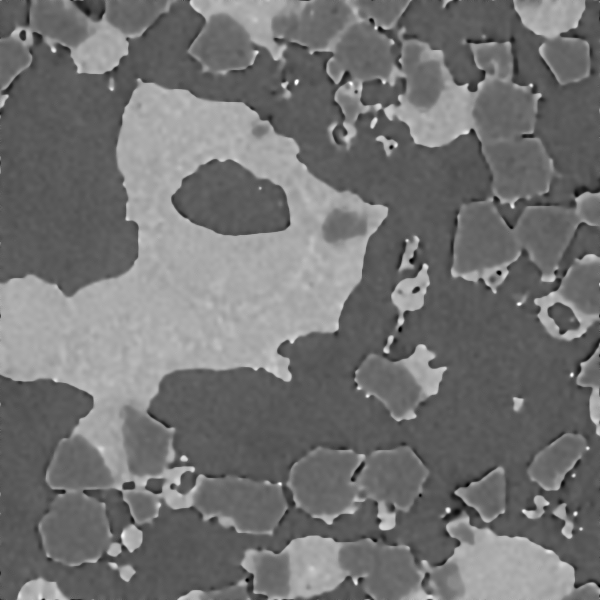}
\end{subfigure}%
\hfill
\begin{subfigure}[t]{.14\textwidth}
  \includegraphics[width=\linewidth]{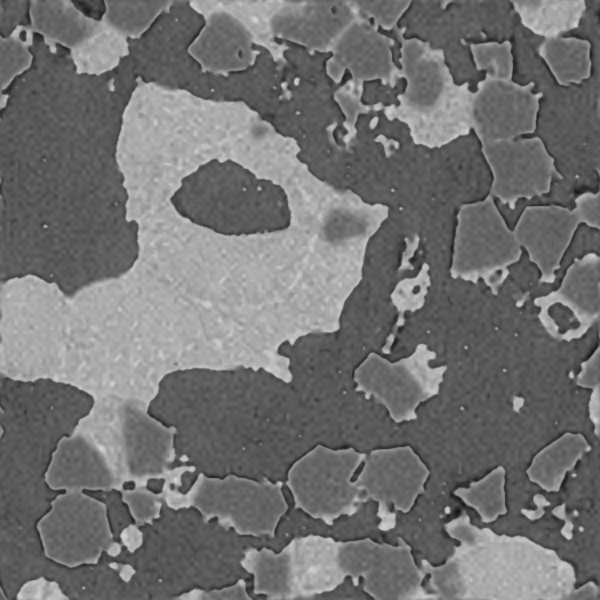}
\end{subfigure}%
\hfill
\begin{subfigure}[t]{.14\textwidth}
  \includegraphics[width=\linewidth]{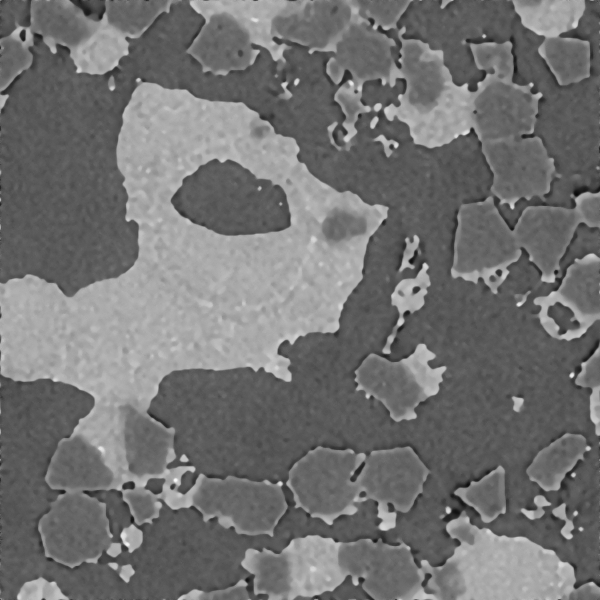}
\end{subfigure}%
\hspace{0.03cm}
\begin{subfigure}[t]{.14\textwidth}
  \includegraphics[width=\linewidth]{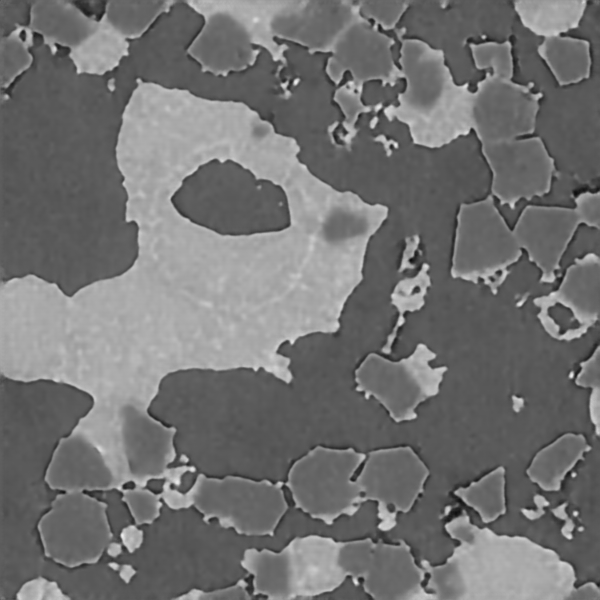}
\end{subfigure}%

\begin{subfigure}[t]{.14\textwidth}
  \includegraphics[width=\linewidth]{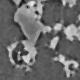}
\end{subfigure}%
\hfill
\begin{subfigure}[t]{.14\textwidth}
  \includegraphics[width=\linewidth]{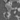}
\end{subfigure}%
\hfill
\begin{subfigure}[t]{.14\textwidth}
  \includegraphics[width=\linewidth]{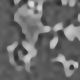}
\end{subfigure}%
\hfill
\begin{subfigure}[t]{.14\textwidth}
  \includegraphics[width=\linewidth]{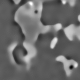}
\end{subfigure}%
\hfill
\begin{subfigure}[t]{.14\textwidth}
  \includegraphics[width=\linewidth]{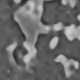}
\end{subfigure}%
\hfill
\begin{subfigure}[t]{.14\textwidth}
  \includegraphics[width=\linewidth]{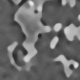}
\end{subfigure}%
\hspace{0.03cm}
\begin{subfigure}[t]{.14\textwidth}
  \includegraphics[width=\linewidth]{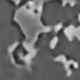}
\end{subfigure}%
\vspace{.5cm}

\begin{subfigure}[t]{.14\textwidth}
  \includegraphics[width=\linewidth]{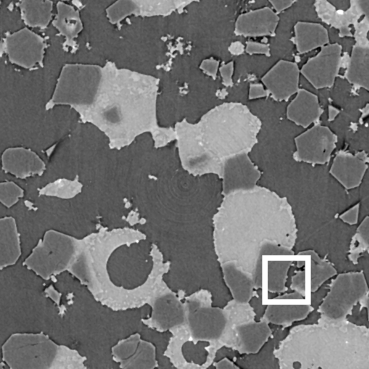}
\end{subfigure}%
\hfill
\begin{subfigure}[t]{.14\textwidth}
  \includegraphics[width=\linewidth]{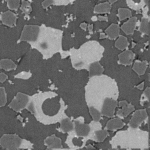}
\end{subfigure}%
\hfill
\begin{subfigure}[t]{.14\textwidth}
  \includegraphics[width=\linewidth]{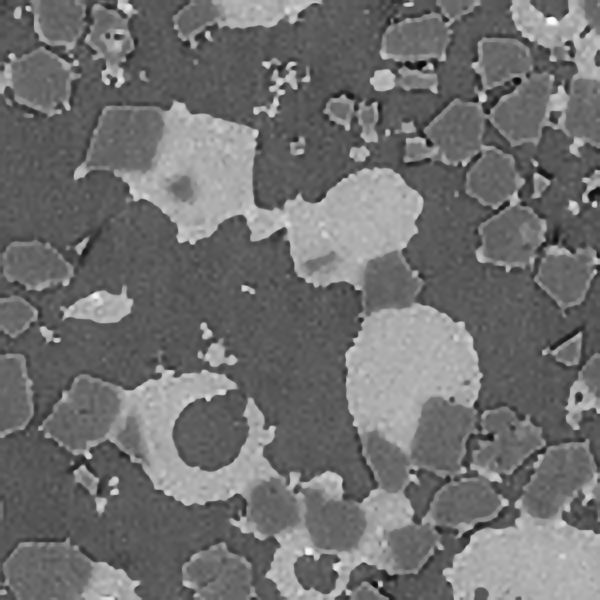}
\end{subfigure}%
\hfill
\begin{subfigure}[t]{.14\textwidth}
  \includegraphics[width=\linewidth]{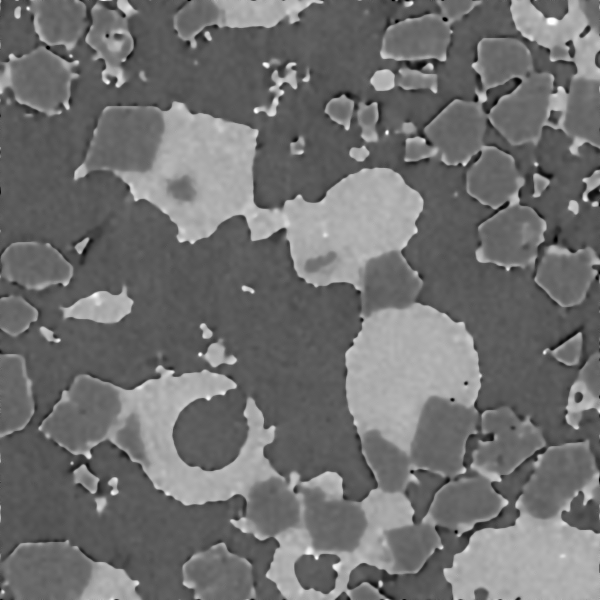}
\end{subfigure}%
\hfill
\begin{subfigure}[t]{.14\textwidth}
  \includegraphics[width=\linewidth]{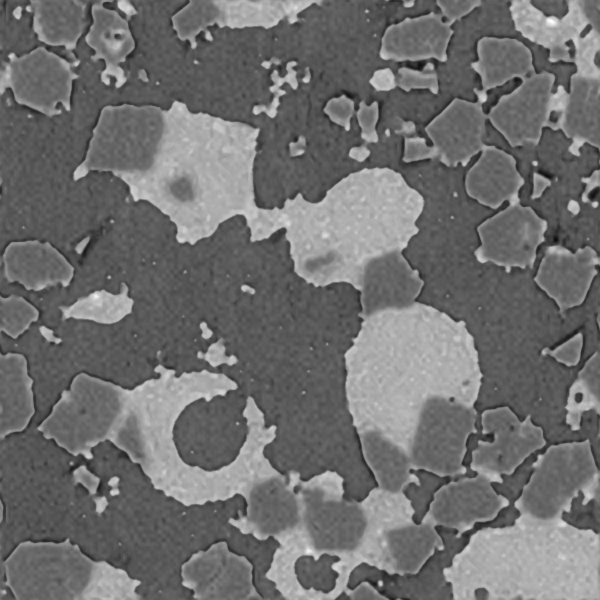}
\end{subfigure}%
\hfill
\begin{subfigure}[t]{.14\textwidth}
  \includegraphics[width=\linewidth]{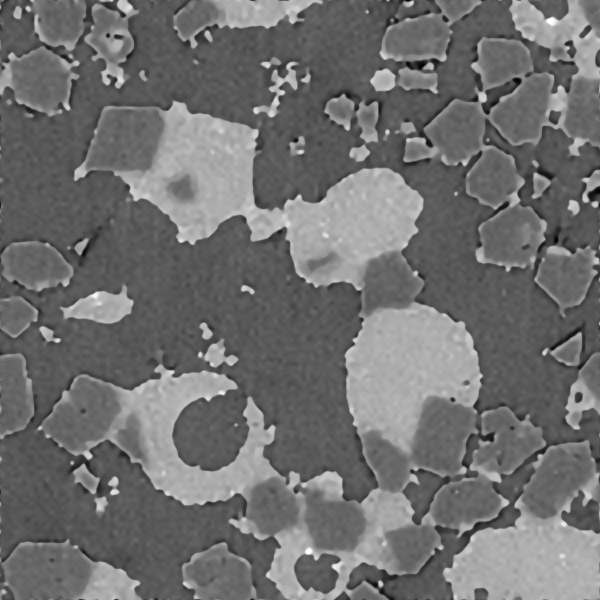}
\end{subfigure}%
\hspace{0.03cm}
\begin{subfigure}[t]{.14\textwidth}
  \includegraphics[width=\linewidth]{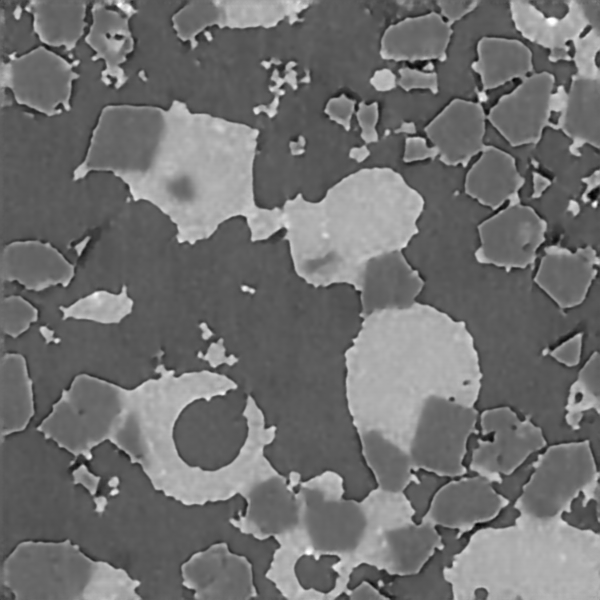}
\end{subfigure}%

\begin{subfigure}[t]{.14\textwidth}
  \includegraphics[width=\linewidth]{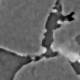}
\end{subfigure}%
\hfill
\begin{subfigure}[t]{.14\textwidth}
  \includegraphics[width=\linewidth]{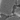}
\end{subfigure}%
\hfill
\begin{subfigure}[t]{.14\textwidth}
  \includegraphics[width=\linewidth]{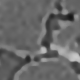}
\end{subfigure}%
\hfill
\begin{subfigure}[t]{.14\textwidth}
  \includegraphics[width=\linewidth]{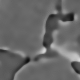}
\end{subfigure}%
\hfill
\begin{subfigure}[t]{.14\textwidth}
  \includegraphics[width=\linewidth]{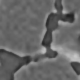}
\end{subfigure}%
\hfill
\begin{subfigure}[t]{.14\textwidth}
  \includegraphics[width=\linewidth]{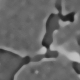}
\end{subfigure}%
\hspace{0.03cm}
\begin{subfigure}[t]{.14\textwidth}
  \includegraphics[width=\linewidth]{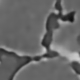}
\end{subfigure}%

\vspace{.5cm}

\begin{subfigure}[t]{.14\textwidth}
  \includegraphics[width=\linewidth]{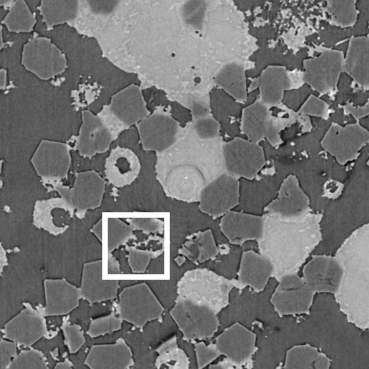}
\end{subfigure}%
\hfill
\begin{subfigure}[t]{.14\textwidth}
  \includegraphics[width=\linewidth]{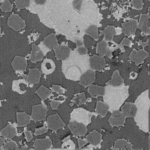}
\end{subfigure}%
\hfill
\begin{subfigure}[t]{.14\textwidth}
  \includegraphics[width=\linewidth]{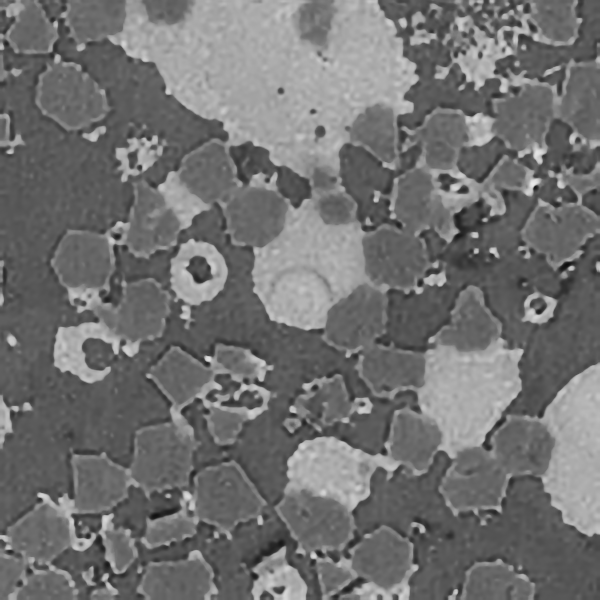}
\end{subfigure}%
\hfill
\begin{subfigure}[t]{.14\textwidth}
  \includegraphics[width=\linewidth]{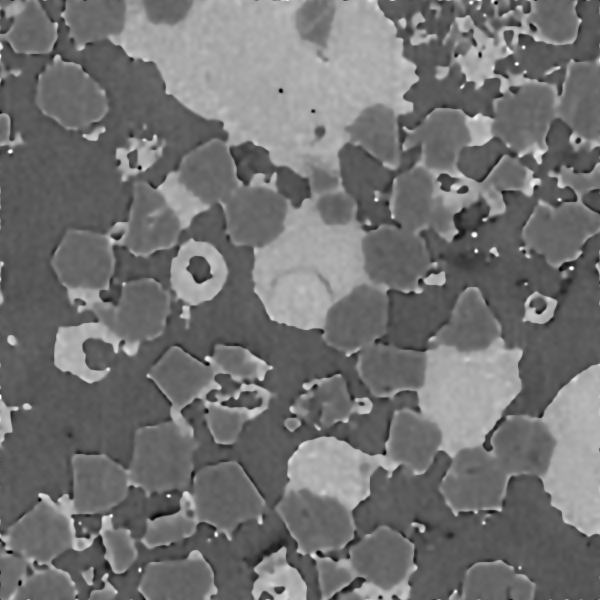}
\end{subfigure}%
\hfill
\begin{subfigure}[t]{.14\textwidth}
  \includegraphics[width=\linewidth]{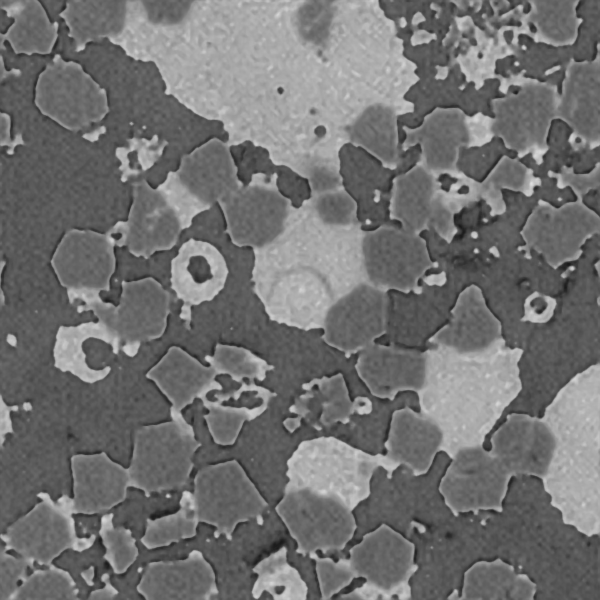}
\end{subfigure}%
\hfill
\begin{subfigure}[t]{.14\textwidth}
  \includegraphics[width=\linewidth]{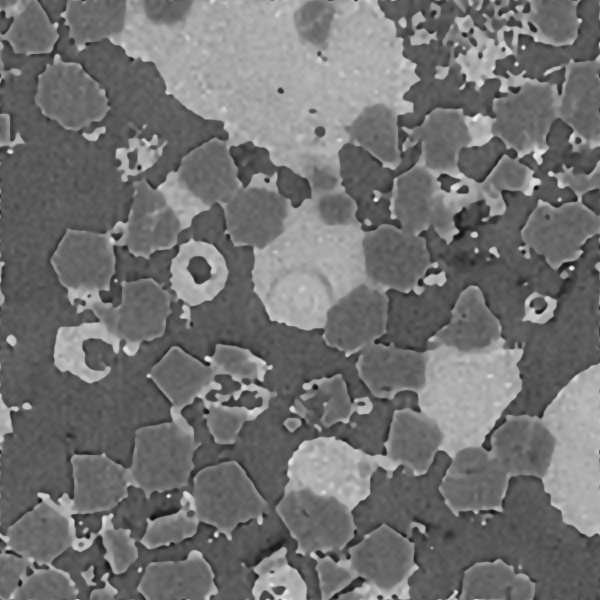}
\end{subfigure}%
\hspace{0.03cm}
\begin{subfigure}[t]{.14\textwidth}
  \includegraphics[width=\linewidth]{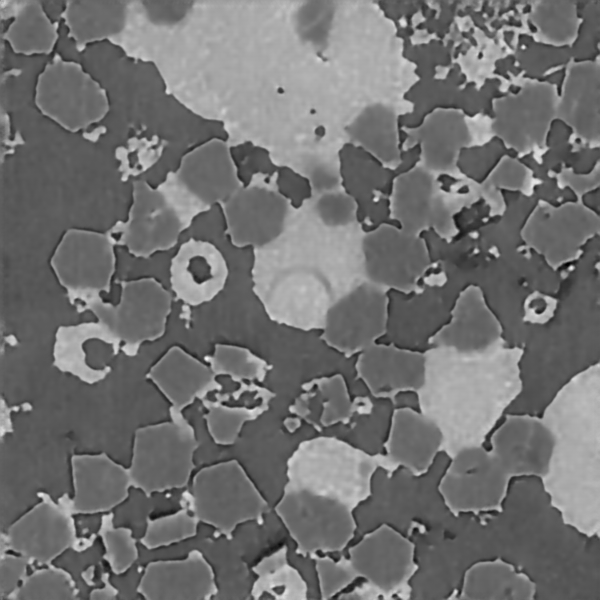}
\end{subfigure}%

\begin{subfigure}[t]{.14\textwidth}
  \includegraphics[width=\linewidth]{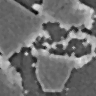}
\end{subfigure}%
\hfill
\begin{subfigure}[t]{.14\textwidth}
  \includegraphics[width=\linewidth]{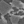}
\end{subfigure}%
\hfill
\begin{subfigure}[t]{.14\textwidth}
  \includegraphics[width=\linewidth]{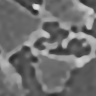}
\end{subfigure}%
\hfill
\begin{subfigure}[t]{.14\textwidth}
  \includegraphics[width=\linewidth]{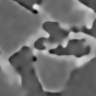}
\end{subfigure}%
\hfill
\begin{subfigure}[t]{.14\textwidth}
  \includegraphics[width=\linewidth]{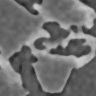}
\end{subfigure}%
\hfill
\begin{subfigure}[t]{.14\textwidth}
  \includegraphics[width=\linewidth]{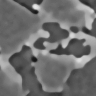}
\end{subfigure}%
\hspace{0.03cm}
\begin{subfigure}[t]{.14\textwidth}
  \includegraphics[width=\linewidth]{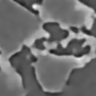}
\end{subfigure}%

\vspace{.5cm}
\begin{subfigure}[t]{.14\textwidth}
  \includegraphics[width=\linewidth]{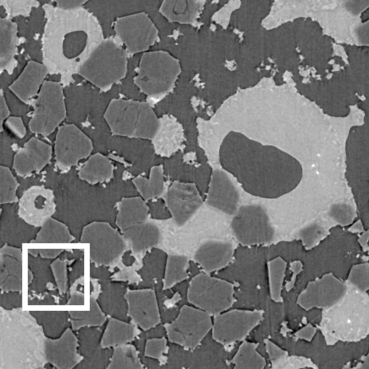}
\end{subfigure}%
\hfill
\begin{subfigure}[t]{.14\textwidth}
  \includegraphics[width=\linewidth]{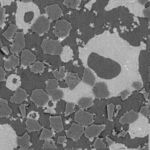}
\end{subfigure}%
\hfill
\begin{subfigure}[t]{.14\textwidth}
  \includegraphics[width=\linewidth]{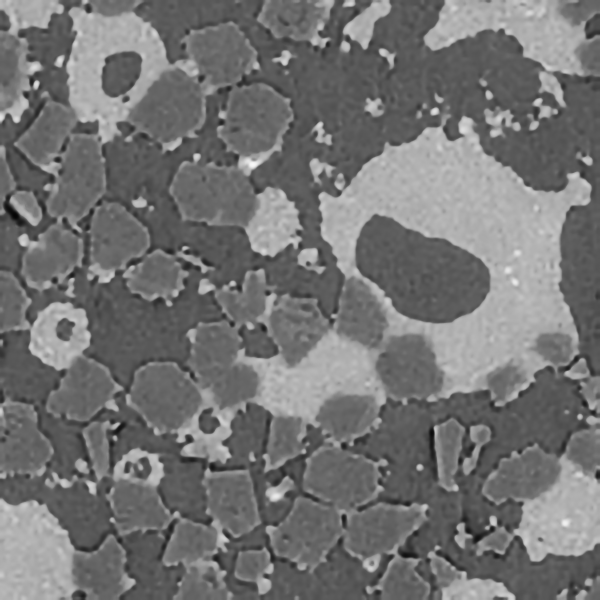}
\end{subfigure}%
\hfill
\begin{subfigure}[t]{.14\textwidth}
  \includegraphics[width=\linewidth]{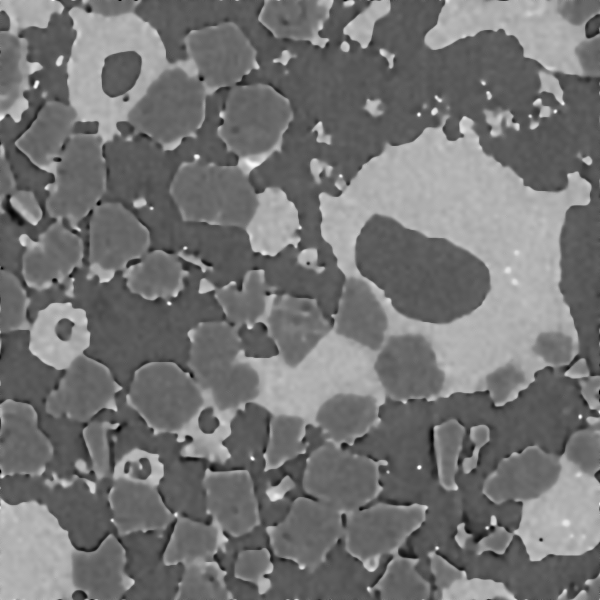}
\end{subfigure}%
\hfill
\begin{subfigure}[t]{.14\textwidth}
  \includegraphics[width=\linewidth]{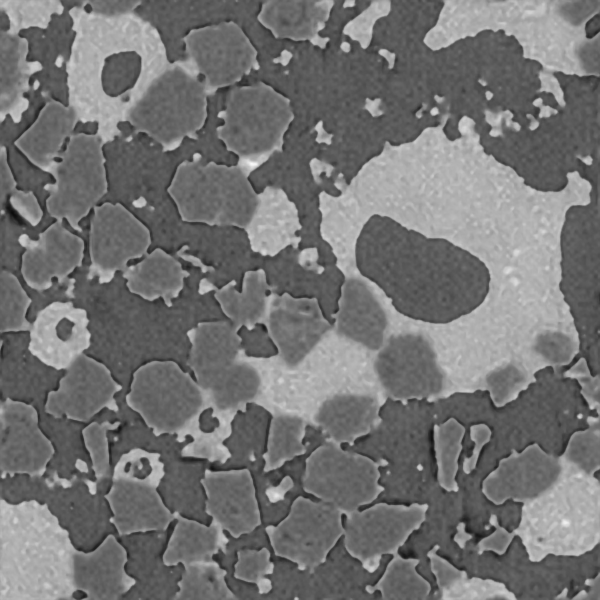}
\end{subfigure}%
\hfill
\begin{subfigure}[t]{.14\textwidth}
  \includegraphics[width=\linewidth]{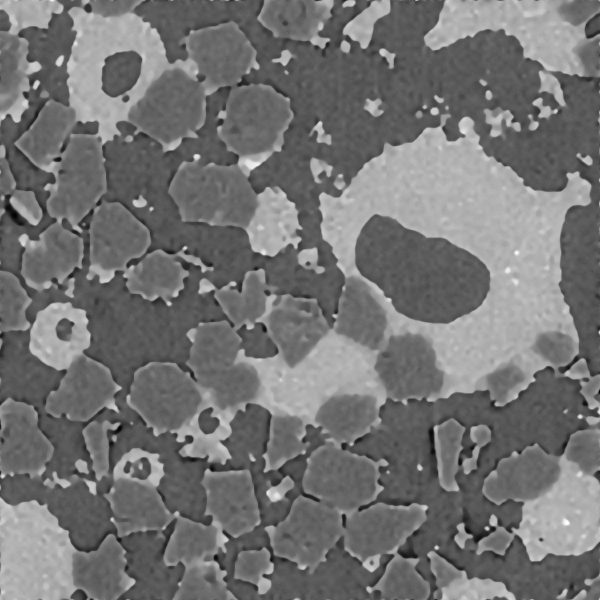}
\end{subfigure}%
\hspace{0.03cm}
\begin{subfigure}[t]{.14\textwidth}
  \includegraphics[width=\linewidth]{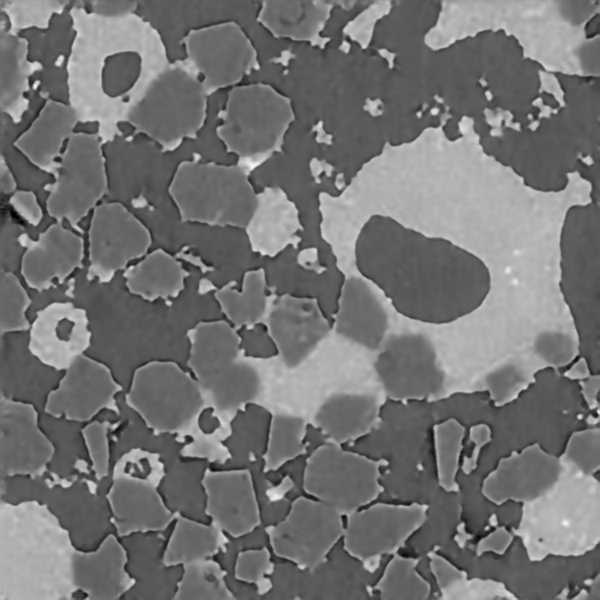}
\end{subfigure}%

\begin{subfigure}[t]{.14\textwidth}
  \includegraphics[width=\linewidth]{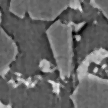}
  \caption*{HR}
\end{subfigure}%
\hfill
\begin{subfigure}[t]{.14\textwidth}
  \includegraphics[width=\linewidth]{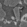}
  \caption*{LR}
\end{subfigure}%
\hfill
\begin{subfigure}[t]{.14\textwidth}
  \includegraphics[width=\linewidth]{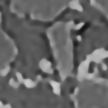}
  \caption*{DIP+TV}
\end{subfigure}%
\hfill
\begin{subfigure}[t]{.14\textwidth}
  \includegraphics[width=\linewidth]{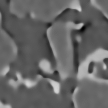}
  \caption*{EPLL}
\end{subfigure}%
\hfill
\begin{subfigure}[t]{.14\textwidth}
  \includegraphics[width=\linewidth]{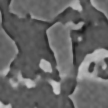}
  \caption*{WPP}
\end{subfigure}%
\hfill
\begin{subfigure}[t]{.14\textwidth}
  \includegraphics[width=\linewidth]{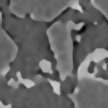}
  \caption*{patchNR}
\end{subfigure}%
\hspace{0.03cm}
\begin{subfigure}[t]{.14\textwidth}
  \includegraphics[width=\linewidth]{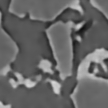}
  \captionsetup{justification=centering}
  \caption*{ACNN \\(data-based)}
\end{subfigure}%
\caption{Comparison of different methods for superresolution. The zoomed-in part is marked with a white box in the ground truth image.
\textit{Top}: full image. \textit{Bottom}: zoomed-in part.
} \label{fig:SIC_further}
\end{figure*}

\end{document}